\documentclass[letterpaper]{article} 

\usepackage{amsmath,amsthm,amssymb,graphicx,mathtools,setspace,mdframed,verbatim, dsfont,pifont,soul,wrapfig}
\usepackage{enumerate}

\usepackage[shortlabels]{enumitem}

\usepackage{tikz}
\usepackage{caption}
\usepackage{subcaption}

\usepackage[square,numbers]{natbib}
\usepackage{geometry}
\usepackage{bold-extra}
\usepackage{nicefrac}
\usepackage{microtype}
\usepackage{graphicx}

\usepackage{booktabs} 
\usepackage{hyperref}

\usepackage{etaremune}
\usepackage{tikz}
\usepackage{dsfont}
\usepackage{mathtools}
\usepackage{xcolor}  

\usepackage[capitalize,noabbrev]{cleveref}

\usepackage{diy}

\newcommand{\1}{\mathds{1}}
\definecolor{myorange}{RGB}{245,156,74}
\hypersetup{
	colorlinks=true,
	linkcolor=blue,
	filecolor=blue,      
	urlcolor=black,
	citecolor=brown,
}

\newenvironment{Msg}[1]
  {\mdfsetup{
    frametitle={\colorbox{white}{\space \large #1\space}},
    shadow=true,shadowsize=1pt,
    skipabove=2pt,
    innertopmargin=-3pt,
    innerbottommargin=7pt,
    innerrightmargin=7pt,
    innerleftmargin=7pt,
    frametitleaboveskip=-\ht\strutbox,
    frametitlealignment=\center,
    linewidth=0.5pt
    }
  \begin{mdframed}%
  }
{\end{mdframed}}

\allowdisplaybreaks

\newcounter{main}
\numberwithin{main}{section}
\newtheorem{theorem}[main]{Theorem}
\newtheorem{proposition}[main]{Proposition}
\newtheorem{lemma}[main]{Lemma}

\theoremstyle{definition}
\newtheorem{assumption}[main]{Assumption}
\newtheorem{definition}[main]{Definition}
\newtheorem{hypothesis}{Induction Hypothesis}[section]

\theoremstyle{remark}

\numberwithin{equation}{section}
\usepackage{titlesec}

\title{A Theoretical Analysis of Self-Supervised Learning \\
for Vision Transformers}
\author{
    Yu Huang\footnote{The first two authors contributed equally.} \thanks{Department of Statistics and Data Science, Wharton School, University
of Pennsylvania. \texttt{\href{mailto:yuh42@wharton.upenn.edu}{\color{black}yuh42@wharton.upenn.edu}}}
     \\
    UPenn 
    \and
    Zixin Wen\footnotemark[1] \thanks{Machine Learning Department, Carnegie Mellon University. \texttt{\href{mailto:zixinw@andrew.cmu.edu}{\color{black}zixinw@andrew.cmu.edu}}}
     \\
    CMU 
    \and 
    Yuejie Chi\thanks{Department of Electrical and Computer Engineering, Carnegie Mellon University.  \texttt{\href{mailto:yuejiec@andrew.cmu.edu}{\color{black}yuejiec@andrew.cmu.edu}}}
     \\
    CMU 
    \and 
    Yingbin Liang\thanks{Department of Electrical and Computer Engineering, The Ohio State University. \texttt{\href{mailto:liang.889@osu.com}{\color{black}liang.889@osu.edu}}}\\
   OSU 
}

\date{March 2024; Revised February 2025}
\makeatother

\begin{document}

\maketitle

\begin{abstract}
    Self-supervised learning (SSL) has become a foundational approach in computer vision, which is broadly categorized into reconstruction-based methods like masked autoencoders (MAE) and discriminative methods such as contrastive learning (CL).  Recent empirical observations reveal that MAE and CL capture distinct representations. CL tends to focus on global patterns, while MAE adeptly captures {\bf both global and subtle local} information simultaneously. 
    Despite a flurry of recent empirical investigations to shed light on this difference, theoretical understanding remains limited, especially on the dominant architecture: {\bf vision transformers} (ViTs). In this paper, to provide rigorous insights, we model the distribution of visual data by considering two types of spatial features: dominant global features and comparatively minuscule local features and study the impact of imbalance among these features. 
    We analyze the training dynamics of one-layer softmax-based ViTs on both MAE and CL objectives using gradient descent. Our analysis shows that as the degree of feature imbalance varies, ViTs trained with the MAE objective effectively learn both global and local features to achieve near-optimal reconstruction, while the CL-trained ViTs favor predominantly global features, even under mild imbalance. These results provide a theoretical explanation for the different behaviors of MAE and CL observed in prior studies.
\end{abstract}

\setcounter{tocdepth}{2}


 \tableofcontents



\section{Introduction}

Self-supervised learning (SSL) has been a leading approach to pretrain neural networks for downstream applications since the introduction of BERT~\citep{devlin2018bert} and GPT~\citep{radford2018improving} in natural language processing (NLP). On the other hand, in vision, self-supervised learning focused more on {\em discriminative} methods, which include contrastive learning (CL)~\citep{he2020momentum,chen2020simple} and non-contrastive learning methods~\citep{grill2020bootstrap,chen2020simple,caron2021emerging,zbontar2021barlow}. Inspired by masked language models in NLP and the seminal work of vision transformers (ViTs) \citep{dosovitskiy2020image}, {\em generative} approaches, such as masked reconstruction-based methods, have gained prominence in self-supervised vision pretraining. The masked autoencoders (MAE)~\citep{he2022masked} and SimMIM~\citep{xie2022simmim} have demonstrated the effectiveness of visual representation learning via reconstruction-based objectives. 

Contrastive learning-like objectives promote instance discrimination among samples in the same batch of training. With suitable data augmentation, CL returns well-trained vision encoders like CLIP~\citep{radford2021learning} and DINO~\citep{caron2021emerging} that can serve as backbones for state-of-the-art multimodal large language models (MLLMs)~\citep{tong2024cambrian}. Masked reconstruction objectives (e.g., MAE), on the other hand, enforce neural networks to reconstruct some or all patches of an image given masked inputs.
In practice, the MAE-like approach proves to have intriguing generalization properties that differ significantly from the behaviors of CL. The seminal work~\citep{he2022masked} showed that MAE can visibly conduct visual reasoning to fill missing patches even under very high masking rates. 
Some critical observations from recent research~\citep{wei2022contrastive, park2023what,  xie2023revealing} provide comparative studies of 
these SSL approaches. They concluded that the ViTs trained via generative objectives display {\bf diverse attention patterns}: different query patches pay attention to distinct local areas.  
This is in sharp contrast to the discriminative approaches, whose attention heads focus primarily on the most significant global pattern regardless of where the query patches are, as shown in \Cref{fig:diverse-pattern}.  
These empirical observations motivate the question: from a {\it theoretical}
standpoint, how do ViTs pick up these observed attention patterns during the training process, respectively for different SSL methods?

Despite extensive empirical efforts of studying SSL in vision pretraining, its theoretical understanding is still nascent. Most existing theories of SSL focused on the discriminative approach~\citep{arora2019theoretical,chen2021intriguing,robinson2021can,haochen2021provable,tian2021understanding,wang2021towards,wen2021toward,wen2022mechanism}, especially (non-)contrastive learning. There are also a few attempts towards understanding methods using the generative approach like masked reconstructions~\citep{cao2022understand, zhang2022mask, haochen2021provable, pan2022towards}, which mainly adapt the theories developed for CL to their context. 
In fact, there are two major limitations of these prior works: {\em i)} {\bf Transformers}, as the dominant architecture in practice, were not studied in the aforementioned works of 
self-supervised learning and {\em ii)} there still lacks a suitable theoretical framework that can provide convincing explanations for the empirical findings in \cite{park2023what,  xie2023revealing}, especially on the difference of the attention patterns learned by different approaches of SSL. The above limitations highlight a significant gap in the literature on SSL for vision pretraining.\footnote{More detailed discussions for related work can be found in
\Cref{sec-related}.}

\begin{figure}[tb]
   \centering
   \includegraphics[width=0.9\linewidth]{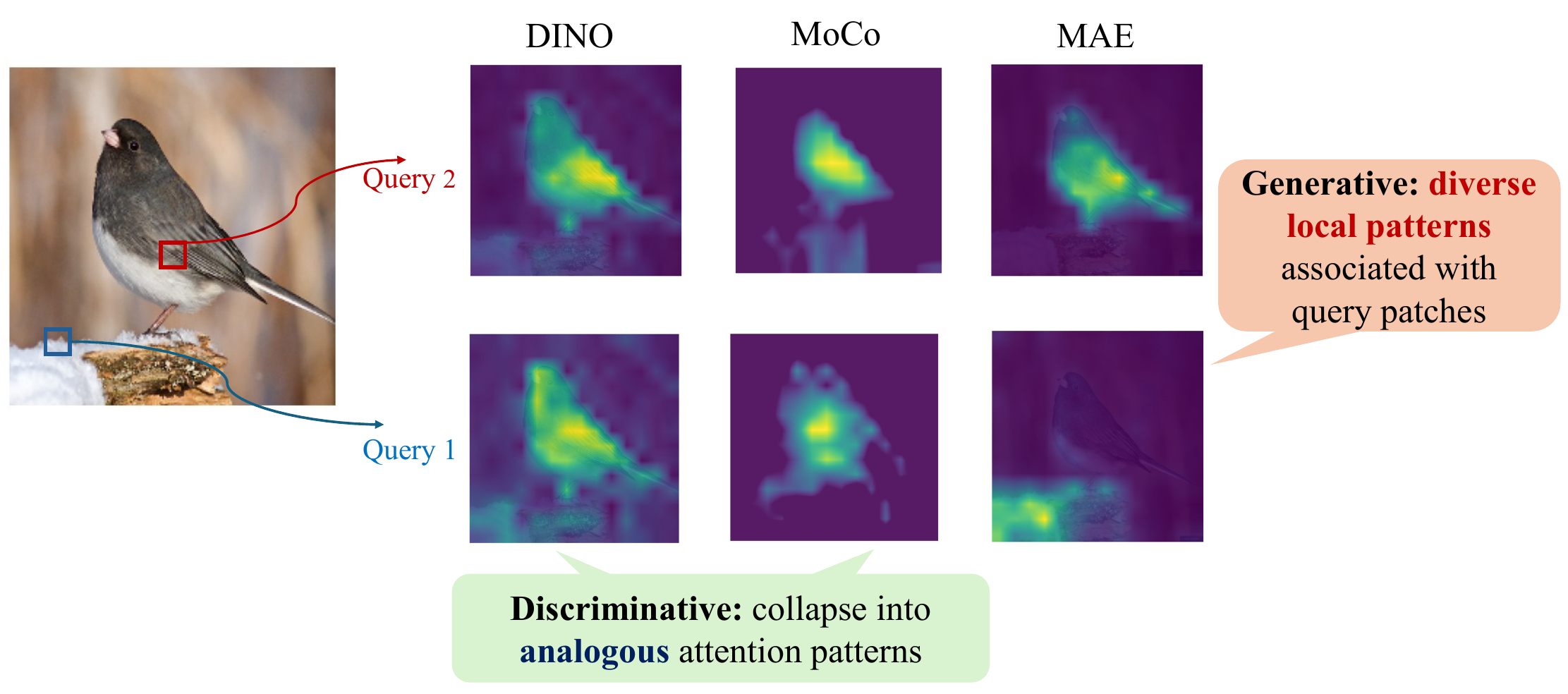}
   \caption{
   Visualization of attention maps in the last layer of the ViTs for query patches from two different spatial locations, similar to those presented in \cite{park2023what}. The ViTs were trained by the generative self-supervised learning approach of masked reconstruction (MAE) and discriminative methods:  DINO~\citep{caron2021emerging} and MoCo~\citep{Chen2021AnES}. 
   } 
   \label{fig:diverse-pattern}
\end{figure}

Motivated by the limited theoretical characterization of SSL for vision with transformers, especially in comparing CL and masked reconstruction objectives, we aim to address the following research questions:
\begin{Msg}{Our Research Questions} 
Can we {\it theoretically} characterize 
the solutions that ViTs converge to in these two mainstream self-supervised learning approaches? How do differences in attention patterns emerge during their respective training processes? 
\end{Msg}

\paragraph{Contributions.}  
In this paper, we take a step toward answering the above questions. We study the training process of one-layer softmax-based ViTs via gradient descent (GD) for both masked reconstruction and contrastive learning, focusing on \textbf{spatially structured data distributions} generalized from supervised learning settings~\citep{jelassi2022vision}. In our setting, each image is sampled from distinct clusters characterized by unique patch-wise feature associations. Each cluster contains two types of features: a large portion of 
patches reside in a global area with shared global features, while relatively few patches occupy the remaining local areas with their own local features. 
We measure the imbalance of feature distribution by a condition called the {\em information gap} $\Delta$, which is formally defined in \cref{eq:gap}.  
Under such a setting: 


\begin{enumerate}
   \item We provide {global convergence} guarantees for training ViTs on both the MAE  and the CL loss fucntions. To the best of our knowledge, this is the first end-to-end guarantee for learning ViTs with self-supervised learning objectives; 
   
   \item We provide a comprehensive characterization of the training dynamics of {\it attention correlations} (see \Cref{def:attn-dynamics}) to illustrate the attention patterns to which ViTs converge:
   \begin{itemize}
       \item MAE provably learns \textbf{diverse} attention patterns,  with each patch concentrating its attention on its designated area based on its position, even under a substantial information gap $\Delta$;
       \item CL primarily learns a \textbf{global} attention pattern, causing all patches to focus on the global area regardless of their locations, even with a minor information gap $\Delta$.
   \end{itemize}
   
\end{enumerate}
These qualitative differences in the solutions learned by the two SSL methods provide strong theoretical support for the empirical behavior gaps observed in \cite{park2023what, xie2023revealing}, and highlight the theoretical advantage of MAE in handling highly imbalanced data structures.

\paragraph{Notation.}  We introduce a few notation to be used throughout the paper. For any two functions $h(x)$ and $g(x)$, we use $h(x)=\Omega(g(x))$ $\big($resp. $h(x)=O(g(x))\big)$ to denote that there exist some universal constants $C_1>0$ and $a_1$, s.t. $|h(x)|\geq C_1|g(x)|$ $\big($resp. $|h(x)|\leq C_1|g(x)|\big)$ for all $x\geq a_1$;  Furthermore, $h(x)=\Theta(g(x))$ indicates $h(x)=\Omega(g(x))$ and $h(x)=O(g(x))$ hold simultaneously. 
We use $\ind\{\cdot\}$ to denote the indicator function, and let $[N] \coloneqq \{1,2, \ldots,N\}$.
We use $\widetilde{O}$, $\widetilde{\Omega}$, and $\widetilde{\Theta}$ to further hide logarithmic factors in the respective notation.
We use $\poly(P)$ and $\polylog(P)$ to represent  large constant-degree polynomials of $P$ and $\log(P)$, respectively.  

\section{Problem Setup}
In this section, we present our problem formulations for studying the training process of ViTs in self-supervised pretraining. We begin with some background information, followed by a description of our data distribution. We then detail the  pretraining strategies using MAE and CL, respectively, with the specific transformer architecture considered in this paper.  

\subsection{Background on self-supervised learning}

\paragraph{Masked reconstruction-based learning.} We follow the masked reconstruction frameworks in \cite{he2022masked,xie2022simmim}. Each  data sample $X\in \mathbb{R}^{d \times P}$ has the form $X= (X_{\pb})_{ \pb \in \mathcal{P} }$, which has $|\mathcal{P}| = P$ patches, and each patch $X_{\pb}\in\mathbb{R}^d$. Given a collection of images $\{X_i\}_{i\in[n]}$, we select a masking set {$\cM_i\subset \cP$} for each image $X_i$, and mask these patches to a uniform value $\mathsf{M} \in \R^d$. The resulting masked images $\{\mask(X_i)\}_{i\in[n]}$ are given by 
\begin{align} \label{eq:masked_image}
\mask(X_i)_\pb = \left\{ \begin{array}{cc} 
[X_i]_\pb  & \pb \in \cU_i \\
\mathsf{M} & \pb \in \cM_i
\end{array}
\right. , \qquad i\in[n],
\end{align}
where $\cU_i=\cP\setminus\cM_i$ is the index set of unmasked patches.
Let $F:X \mapsto \widehat{X}$ be an architecture that outputs a reconstructed image $\widehat{X} \in\R^{d\times P}$ for any given input $X\in\R^{d\times P}$. 
The pretraining objective is then defined as the mean-squared reconstruction loss over a series of subsets \(\cP'_i\subset \cP\) of the image as follows:
\begin{equation}
\textstyle   \cL_{\texttt{masked}} (F) = \frac{1}{n}\sum_{i=1}^n \sum_{\pb \in \cP'_i} \Big\|[X_i]_{\pb} - [F(\mask(X_i))]_\pb\Big \|_2^2.\label{eq:gen-loss}
\end{equation}
MAE~\citep{he2022masked} chose the subset $\cP_i'$ as the set of masked patches $\cM_i$, whereas SimMIM~\citep{xie2022simmim} aimed to reconstruct the full image $\cP_i'=\cP$. We do not explore the trade-offs between these two approaches in our study. 

\paragraph{Contrastive learning.} Contrastive learning~\citep{chen2020simple} aims to learn meaningful representations \( F \) by distinguishing between similar and dissimilar data points.  For a given batch \( \{X_i\}_{i \in [n]} \), we generate a positive pair \( (X_i^{(1)}, X_i^{(2)}) \) for each \( i  \) by applying random augmentations to \( X_i \). Negative pairs \( (X_i^{(1)}, X_j^{(2)}) \) for \( j \neq i \) are formed from different data points. The model \( F \) is trained to minimize the following contrastive loss:
\begin{equation}
\textstyle   \cL_{\texttt{contrastive}} (F) = \frac{1}{n}\sum_{i=1}^n    \left[-\tau\log \left(\frac{e^{\operatorname{\mathsf{Sim}}_F\left(X_i^{(1)}, X_i^{(2)}\right)/\tau }}{\sum_{j \in [n]} e^{\operatorname{\mathsf{Sim}}_F\left(X_i^{(1)}, X_j^{(2)}\right)/\tau }}\right)\right], \label{eq:gen-loss-cl}
\end{equation}
where \( \operatorname{\mathsf{Sim}}_F \) measures the similarity between two representations, and \( \tau \) is a temperature parameter controlling the sharpness of the distribution.
\subsection{Data distribution}\label{sec:data}
We assume the data samples $X\in \mathbb{R}^{d \times P}$ are drawn independently based on some data distribution $\mathcal{D}$. To capture the \emph{feature-position (FP) correlation} in the learning problem, we consider the following setup for vision data. We assume that the data distribution consists of many different clusters, where each cluster captures a distinct spatial pattern, and hence is defined by a different partition of patches with a different set of visual features. We define the data distribution $\mathcal{D}$ formally as follows. An intuitive illustration of data generation is given in \Cref{fig:data-distribution}.
\begin{figure}[tb]
   \centering
\includegraphics[width=.9\linewidth]{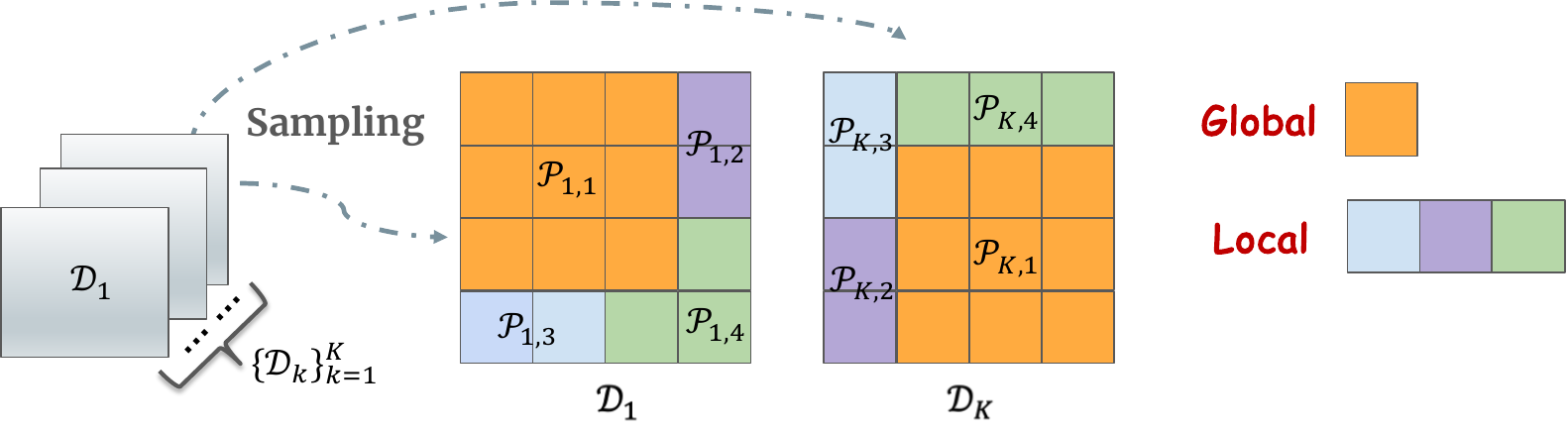}
   \caption{Illustration of our data distribution (see \Cref{def:data}). Each cluster $\cD_k$ is segmented into distinct areas $\cP_{k,j}$ 
   , with squares in the same color representing the same area $\cP_{k,j}$. The global area $\cP_{k,1}$ (depicted in orange) contains a larger count of patches compared to any other local areas. It is important to note that while we use spatially contiguous partitions for clarity in this illustration, our data model is also applicable to non-contiguous cases.
   }
   \label{fig:data-distribution}
\end{figure}
\begin{definition}[Data distribution $\cD$]\label{def:data}
The data distribution $\mathcal{D}$ has \(K= O(\polylog(P))\) different clusters $\{\cD_k\}_{k=1}^{K}$. For every cluster $\cD_k, k\in[K]$, there is a corresponding partition of $\cP$ into $N_k$ disjoint subsets $\cP =\bigcup^{N_k}_{j=1}\cP_{k,j}$ which we call \textbf{areas}. 
For each sample \(X = (X_\pb)_{\pb\in\cP}\), its sampling process is as follows:
   \begin{itemize}
   \item We draw \(\cD_k\) uniformly at random from all clusters and draw a sample \(X\) from $\cD_k$.
   \item Given $k\in [K]$, for any $j\in [N_k]$, all patches \(X_\pb\) in the area \(\cP_{k,j}\) are given the same content $X_\pb = v_{k,j} z_{j}(X)$, where $v_{k,j} \in \mathbb{R}^d$ is the {\em visual} feature and $z_{j}(X)$ is the latent variable. We assume $\bigcup_{k=1}^{K}\bigcup^{N_k}_{j=1}\{v_{k,j}\}$ are orthogonal to each other with unit norm. 
   \item Given $k\in [K]$, for any $j\in [N_k]$, $z_j(X)\in [L,U]$, where $0\leq L<U$ are on the order of $\Theta(1)$.\footnote{The distribution of $z_j(X)$ can be arbitrary within the above support set.} 
\end{itemize}
\end{definition}

\paragraph{Area-wide patterns and features.} 
Image data naturally contains two types of features: the global features and the local features. For instance, in an image of an object, global features can capture the shape and texture of the object, such as the fur color of an animal, whereas local features describe specific details of local areas, such as the texture of leaves in the background. Recent empirical studies on self-supervised pretraining with ViTs~\citep{park2023what,wei2022contrastive} and observations in \Cref{fig:diverse-pattern} collectively show that masked pretraining exhibits the capacity to avoid attention collapse concentrating towards those global shapes by identifying diverse local attention patterns. 
Consequently, unraveling their mechanisms necessitates a thorough examination of data characteristics that embody both global and local features. 
In this paper, we characterize these two types of features by the following assumption on the data.

\begin{assumption}[Global feature vs local feature] \label{assup:feature}
Let \(\cD_k\) with \( k\in[K]\) be a cluster from \(\cD\). We let $\cP_{k,1}$ be the \textbf{global area} of cluster $\cD_k$, and all the other areas $\cP_{k,j}, j \in [N_k]\setminus\{1\}$ be the \textbf{local areas}. Since each area corresponds to an assigned feature, we also call them the {\it global} and {\it local} features, respectively. Moreover, we assume:
\begin{itemize}
   \item Global area: given $k\in [K]$, 
   we set $C_{k,1}= |\cP_{k,1}| = \Theta(P^{\kappa_c})$ with $\kappa_c\in [0.5005, 1]$, where $C_{k,1}$ is the number of patches in the global area $\cP_{k,1}$. 
   \item Local area: given $k\in [K]$, we choose $C_{k,j}= |\cP_{k,j}| = \Theta(P^{\kappa_s})$ with $\kappa_s\in [0.001, 0.5]$ for $j>1$, where $C_{k,j}$ denotes the number of patches in the local area $\cP_{k,j}$. 
\end{itemize}
\end{assumption}
The rationale for defining the global feature in this manner stems from observing that patches representing global features ($C_{k,1}$) typically occur more frequently than those representing local features ($C_{k,j}$, for $j>1$), since global features capture the primary visual information of an image, offering a dominant view, while local features focus on subtler details within the image. 
{Our empirical observations (see \Cref{fig:enter-label}) further substantiate the significance of distinguishing between global and local patterns in data distributions, which is essential for elucidating the distinct behaviors exhibited by MAE and CL.}

\begin{figure}[tb]
   \centering
   \includegraphics[width=.9\linewidth]{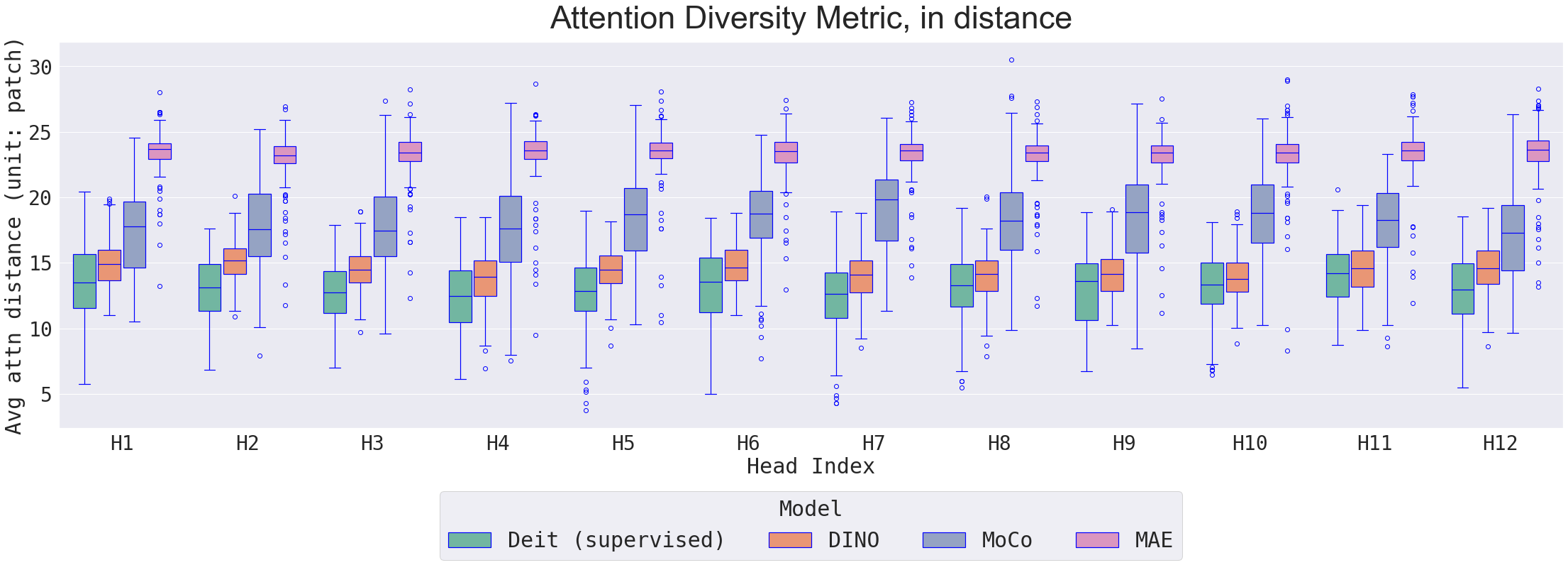} 
   \caption{{\textbf{Attention Diversity Metric:} We design a novel empirical metric, the {\bf attention diversity metric}, to probe the last layer of ViTs trained by masked reconstructions (MAE), CL (MoCo), another discriminative SSL approach (DINO), and supervised learning (DeiT). 
   Lower values of this metric signify focused attention on a similar area across different patches, reflecting a global pattern of focus. Conversely, higher values suggest that attention is dispersed, focusing on different, localized areas.   The results show that the MAE model excels in capturing {\it diverse local patterns} compared to discriminative methods like CL.  (see \Cref{sec:exp} for details). 
}}
   \label{fig:enter-label}
\end{figure}

\subsection{Masked reconstruction with transformers}

\paragraph{Transformer architecture.}
A transformer block \citep{vaswani2017attention, dosovitskiy2020image} consists of a self-attention layer followed by an MLP layer. The self-attention layer has multiple heads, each of which consists of the following components: a query matrix $W^{Q}$, a key matrix $W^{K}$, and a value matrix $W^{V}$. Given an input $X$,
the output of one head in the self-attention layer can be described by the following mapping:
\begin{align}
   G(X; W^Q,W^K,W^V) = \operatorname{softmax}\left({(W^Q X)^{\top} W^K X}\right)\cdot (W^V X)^{\top} ,\label{SA}
\end{align}
where 
the $\operatorname{softmax}(\cdot)$ function is applied row-wisely and  for a vector input $z\in\mathbb{R}^{P}$, the $i$-th entry of $\operatorname{softmax}(z)$ is given by $\textstyle\frac{\exp(z_i)}{\sum_{s=1}^P\exp(z_s)}$.

To simplify the theoretical analysis, we consolidate the product of query and key matrices \((W^Q)^\top W^K\) into one weight matrix denoted as $Q$. Furthermore, we set $W^{V}$ to be the identity matrix and fixed during the training. These simplifications are often taken in recent theoretical works~\citep{jelassi2022vision,huang2023context, zhang2023trained} in order to allow tractable analysis. With these simplifications in place,
\cref{SA} can be rewritten as
\begin{align}
   G(X; Q) =   \operatorname{softmax}\left({  X^{\top} Q X}\right)\cdot X^{\top}. \label{SA-simplified}
\end{align}
Note that the input tokens in transformers are indistinguishable without explicit spatial information. Therefore, positional encodings should be added to the input embeddings to retain this crucial positional context as in practices~\citep{dosovitskiy2020image, he2022masked}. Our assumptions regarding the positional encodings are as follows: 
\begin{assumption}[Positional encoding]\label{asssump:pos-emb}
   We assume fixed positional encodings, which is consistent with the implementation in MAE~\citep{he2022masked}: $E = (e_{{\pb}})_{{\pb} \in \cP}\in\mathbb{R}^{d\times P}$ where positional embedding vectors $ e_{{\pb}}$ are orthogonal to each other and to all the features $v_{k,j}$, and are of unit-norm. 
\end{assumption}
We now include positional embeddings in \cref{SA-simplified} and introduce the network architecture for masked reconstruction used in this study.
\begin{definition}[ViT architecture for MAE]\label{def:model-arch}
   We assume that our vision transformer $F^{\texttt{mae}}(X;Q)$ consists of a single-head self-attention layer with an attention weight matrix \(Q\in\R^{d\times d}\). For an input image \(X\sim \cD\), we add positional encoding by letting \(\widetilde{X} = X + E\). The attention score from patch $X_\pb$ to patch $X_{{\qb}}$ is denoted by
\begin{align}
       \textstyle \score^{\texttt{m}}_{{\pb}\to {\qb}}(X;Q) \coloneqq  \frac{e^{\widetilde{X}_{{\pb}}^{\top}Q\widetilde{X}_{{\qb}}} }{\sum_{{\rb}\in \cP}e^{\widetilde{X}_{{\pb}}^{\top}Q\widetilde{X}_{{\rb}}} }, \quad \textrm{for } \pb,\qb \in \cP.  \label{def:attn}
\end{align}
The output of the transformer is given by
   \begin{align}\label{model}
       [F^{\texttt{mae}}(X; Q)]_{{\pb}} = \sum_{{\qb} \in \cP} \score^{\texttt{m}}_{{\pb}\to {\qb}}(X;Q)\cdot X_{{\qb}}, \quad \textrm{for } \pb \in \cP. 
   \end{align}
\end{definition}

Then we formally define the masking operation and the objective for our masked pretraining task.

\begin{definition}[Random masking]\label{def:mask}
    Let $\mask(X) \to \R^{d\times P}$ denote the random masking operation, which randomly selects (without replacement) a subset of patches $\cM$ in $X$ with a masking ratio $\gamma=\Theta(1)\in (0,1)$ and masks them to be $\mathsf{M} := \mathbf{0} \in \R^d$. The masked samples obey \cref{eq:masked_image}. 
\end{definition}

\paragraph{MAE objective.}   To train the model  $F^{\texttt{mae}}(\mask(X);Q)$, 
following the methodology described in MAE practice~\citep{he2022masked}, 
we minimize the squared reconstruction error in  \cref{eq:gen-loss} only on masked patches, where $\mask(X)$ follows Definition~\ref{def:mask}. The training objective thus can be written as
\begin{align}\label{loss}
   \cL_{\texttt{mae}}(Q)&   \coloneqq  \frac{1}{2} {\mathbb{E}}\left[\sum_{{\pb}\in \cP}\ind\{{\pb}\in\cM\}\Big\|[F^{\texttt{mae}}(\mask(X); Q)]_{{\pb}}-X_{{\pb}} \Big\|^2\right] ,
\end{align}
where the expectation is with respect to both the data distribution and the masking.
Note that our objective remains highly nonconvex with the model defined in \Cref{def:model-arch}. 

\paragraph{Training algorithm.} The learning objective in \cref{loss} is minimized via GD with learning rate $\eta>0$. At $t=0$, we initialize $Q^{(0)}: = \mathbf{0}_{d\times d}$ as the zero matrix. The parameter is updated as follows:
\begin{align*}
  Q^{(t+1)} =Q^{(t)}-\eta \nabla_{Q} \cL_{\texttt{mae}}(Q^{(t)}).
\end{align*}
 Note that the initialization of $Q^{(0)}$ results in any query patch uniformly attending to all patches.

\subsection{Contrastive learning with transformers}
The transformer architecture used for CL is similar to that of MAE, but with a minor modification to accommodate contrastive loss, as outlined below.
   \begin{definition}[ViT architecture for CL]\label{def:model-arch-cl}
    We consider a vision transformer $F^{\texttt{cl}}(X;Q)$ consisting of a single-head self-attention layer with an attention weight matrix \(Q\in\R^{d\times d}\). For an input image \(X\), the attention score from patch $X_\pb$ to patch $X_{{\qb}}$ is denoted by
   \begin{align}
          \textstyle \score^{\texttt{c}}_{{\pb}\to {\qb}}(X;Q) \coloneqq  \frac{e^{e_{{\pb}}^{\top}Q {X}_{{\qb}}} }{\sum_{{\rb}\in \cP}e^{e_{{\pb}}^{\top}Q{X}_{{\rb}}} }, \quad \textrm{for } \pb,\qb \in \cP.  \label{def:attn-cl}
   \end{align}
The output of the transformer is then computed as
      \begin{align}\label{model-cl}
          F^{\texttt{cl}}(X; Q) = \frac{1}{P}\sum_{{\pb}, {\qb} \in \cP} \score^{\texttt{c}}_{{\pb}\to {\qb}}(X;Q)\cdot X_{{\qb}} \quad \in \mathbb{R}^d,
      \end{align}
which represents the average pooling of all the patches.
   \end{definition}
   The key distinction is that we separate the positional and patch embeddings within the attention mechanism for technical simplicity. However, it is important to emphasize that these two types of embeddings remain coupled for attention calculations.
   \begin{definition}[Data augmentation]\label{def-aug}
        For a sample $X\in\mathbb{R}^d$, we generate two new samples $X^{+}$ and $X^{++}$ by independently applying random masking as in Definition~\ref{def:mask} with a ratio $\gamma_0=\Theta(1)$, similar to the crop-resize operations used in practice. The unmasked sets for them are denoted as $\cU^+$ and $\cU^{++}$.  
   \end{definition}

\paragraph{CL objective.}    Given  a sample $X$, we first generate a pair of positive samples $\{X^{+}, X^{++}\}$ via Definition~\ref{def-aug}.  
   Then we generate a batch of i.i.d.  negative samples $\mathfrak{N}=\{X^{-,s}\}_{s\in[N_c]}$. Denoting   $\mathfrak{B}= \mathfrak{N}\cup \{X^{++}\}$, we minimize the expected 
   contrastive loss in \cref{eq:gen-loss-cl} with $\ell_2$-regularization:   
   \begin{align}
     \textstyle  \cL_{\texttt{cl}}(Q)&   \coloneqq  {\mathbb{E}}_{X^+, X^{++}, \mathfrak{N}}\left[-\tau\log \left(\frac{e^{\operatorname{\mathsf{Sim}}_{F^{\texttt{cl}}}\left(X^{+}, X^{++}\right)/\tau }}{\sum_{X' \in \mathfrak{B}} e^{\operatorname{\mathsf{Sim}}_{F^{\texttt{cl}}}\left(X^{+}, X'\right)/\tau }}\right)\right] 
     + \frac{\lambda}{2}\|Q\|_{F}^2,
\label{obj-cl}
   \end{align}
   where $\|\cdot\|_{F}$ denotes the Frobenius norm, $ \lambda>0$ is the regularization parameter
   , and the similarity of the representations of $X$ and $X^{\prime}$ obtained by $F^{\texttt{cl}}(\cdot; Q)$ is defined as
$$\operatorname{\mathsf{Sim}}_{F^{\texttt{cl}}}\left(X, X^{\prime}\right):=\left\langle F^{\texttt{cl}}(X; Q),\, \operatorname{\mathsf{StopGrad}}\left(F^{\texttt{cl}}\left(X^{\prime}; Q\right)\right)\right\rangle.$$
The \( \operatorname{\mathsf{StopGrad}}(\cdot) \)  operator ensures that no gradient is computed for this term. Additionally, no augmentation is applied to the negative samples. Both practices are standard in the literature on the theory of contrastive learning~\citep{wen2021toward, wen2022mechanism}. 
   Similar to MAE, we update $Q$ by GD with zero-initialization:
   \begin{align}
       \textstyle Q^{(t+1)}=Q^{(t)}-\eta
         \nabla_{Q}\cL_{\texttt{cl}}(Q^{(t)}). \label{eq: gd-cl}
   \end{align}
   In the following, any variable with a superscript $^{(t)}$ represents that variable at the $t$-th step of training.


\section{Warmup of Attention Patterns}\label{sec:attn-cor}
To show the significance of the data distribution design and understand the nature of our self-supervised learning tasks,  
in this section,  we will provide some preliminary implications of the spatial structures in Definition~\ref{def:data}. 
Intuitively, for MAE, for a given cluster \(\cD_k\), to reconstruct a missing patch $\pb \in \cP_{k,j}\cap\cM$, the attention head should exploit all {\it unmasked} patches in the {\it target} area $\cP_{k,j}$ to find the same visual feature $v_{k,j}$ to fill in the blank, which emphasizes the {\it locality} for $\pb$ in different areas. However,  CL focuses on any discriminative patterns regardless of the location of $\pb$,  which can align positive pairs but may lead to collapsed attention patterns.   
We will elaborate on these points by describing the \emph{area attentions} and illustrating the intuition about how they can be learned via {\it attention correlations} (Definition~\ref{def:attn-dynamics}). 

\paragraph{Area attention.} We first define a new notation for a cleaner presentation. For \(X\sim\cD\) and \(\pb \in \cP\), we write the attention of patch \(X_{\pb}\) to a subset $\cA \subset \cP$ of patches by
\begin{align*}
 \textstyle  \Score^{\dagger}_{{\pb}\to \cA}(X;Q)  \coloneqq  \sum_{\qb \in \cA}\score^{\dagger}_{\pb\to\qb}(X;Q), \quad \text{ for } \dagger\in\{\texttt{m}, \texttt{c}\}.  
\end{align*}
\paragraph{MAE's ability to learn locality with ViTs.}
Let us first explain why the above notion of area attention matters in understanding how attention works in masked reconstruction. Suppose we have a sample $X$ picked from $\cD_k$, and the patch $X_\pb$ with $\pb\in\cP_{k,j}$ is masked. Then the prediction of $X_\pb$ given masked input $\mask(X)$ can be written as
\begin{align*}
  & [F^{\texttt{mae}}(\mask(X); Q)]_{{\pb}} \textstyle= \sum_{{\qb} \in \cP} \mask(X)_{{\qb}}\cdot\score^{\texttt{m}}_{{\pb}\to {\qb}}(\mask(X);Q) \\
   &~~~~~ \textstyle= \sum_{i\in[N_{k}]} z_{i}(X)v_{k,i} \cdot\Score^{\texttt{m}}_{{\pb}\to \cU \cap \cP_{k,i} }(\mask(X);Q) \qquad \text{ (since $\mask(X)_\qb = \mathbf{0}$ if $\qb\in\cM$)}.
\end{align*}
To reconstruct the original patch $X_\pb$, the transformer should not only focus on the correct area $\cP_{k,j}$,  but must also prioritize attention to the \emph{unmasked} patches within this area.  This specificity is denoted by the area attention $\Score^{\texttt{m}}_{{\pb}\to \cU \cap \cP_{k,j} }$ over $\cU \cap \cP_{k,j}$, a requirement imposed by masking operations. 
We refer to these location-dependent attention patterns as {\bf locality}.

To further explain how ViTs perform such prioritization, we introduce the following quantities, which capture the major insights of our analysis to distinguish between MAE and contrastive learning.

\begin{definition}\label{def:attn-dynamics}(Attention correlations)
 Let ${\pb}\in \cP$, and 
 we define 
 attention correlations as:
 \begin{enumerate}
     \item Feature-Position (FP) Correlation: $\Phi_{{\pb}\to v_{k,m}} \coloneqq  e^{\top}_{{\pb}}Qv_{k,m}$, for $k\in[K]$ and $m\in[N_k]$;
     \item Position-Position (PP) Correlation: $\Upsilon_{{\pb}\to {\qb}} \coloneqq e_{{\pb}}^{\top}Qe_{{\qb}},\ \forall \qb\in\cP.$
 \end{enumerate}
Due to our (zero) initialization of $Q^{(0)}$, we have $\Phi^{(0)}_{{\pb}\to {v_{k,m}}}=\Upsilon^{(0)}_{{\pb}\to {\qb}}=0$.

\end{definition}

{These two types of attention correlations, FP correlation \(\Phi_{\pb \to v_{k,m}}\) and PP correlation \(\Phi_{\pb \to \qb}\), act as the exponent terms within the \(\operatorname{softmax}\) calculations for attention scores. 
Given \(\pb \in \cP_{k,j}\) is masked, 
the (unnormalized) attention $\score^{\texttt{m}}_{\pb\to\qb}$ directed towards an {\em unmasked} patch $\qb$ is influenced jointly by these correlations. Hence, the described attention pattern for MAE can emerge from either a substantial FP correlation \(\Phi_{\pb \to v_{k,j}}\) or a significant PP correlation \(\Phi_{\pb \to \qb}\) for \(\qb\) in the same area as \(\pb\). However, in our setting, the latter mechanism—learning via PP correlation—fails to produce desired attention patterns: 
{\em i).} such a mechanism inadvertently directs attention to the {\em masked} patches, which is not desirable; {\em ii).} such position association could be vulnerable to the variation across different clusters, i.e., 
$\pb,\qb\in\cP_{k,j}$  does not necessarily hold for all $k\in [K]$. 
This also highlights that prior work~\citep{jelassi2022vision} that relied solely on  pure positional attention cannot fully explain the ViTs' ability to learn locality when the patch-wise associations are not fixed.
}

\paragraph{Why CL may fail to learn the locality.}
Now turning to CL, for $\cX\in\cD_k$, we have  the following form of similarity between the positive pair:
\begin{align*}
  & \langle F^{\texttt{cl}}(X^+; Q), F^{\texttt{cl}}(X^{++}; Q)\rangle\\
  &~~~~= \frac{1}{P^2}\sum_{{\pb}, {\pb'} \in \cP} \sum_{i=1}^{N_k}\Score^{\texttt{c}}_{{\pb}\to \cU^+ \cap \cP_{k,i} }(X^+;Q)\Score^{\texttt{c}}_{{\pb'}\to \cU^{++} \cap \cP_{k,i} }(X^{++};Q).
\end{align*}
Thus, to align the positive representations effectively, the optimal strategy is also to direct attention toward a specific area for each patch $\pb$, i.e., 
greedily ensuring that only one area attention  \(\Score^{\texttt{c}}_{{\pb} \to \cU^+ \cap \cP_{k,i}}\) is activated for some \( i \in [N_k] \). {However, the above expression suggests that the selected area by the optimal strategy may not necessarily depend on the location $\pb$, 
which could lead to a collapsed attention scenario where all patches focus on the same area.} 
Regarding attention correlations, the attention mechanism defined in \cref{def:attn-cl} requires us to handle only the FP correlations among different features for CL. \Cref{thm:cl-convergence} in the next section confirms that a collapsed solution indeed occurs: ViTs trained with CL concentrate attention on the global area across all patches by exclusively capturing global FP correlations across all patches, i.e.,  $\Phi_{\pb\to v_{k,1}}$ becomes large for all $ \pb\in\cP$. 

\section{Statements of Main Results}

In this section, we present our main theorems on the learning processes of ViTs in MAE and CL. 
We begin by introducing notations that will be used in theorem presentations.

\paragraph{Information gap and a technical condition.} Based on our data model in \Cref{sec:data}, we 
introduce a notion of {\em information gap} to quantify the degree of imbalance between global and local areas (cf.~Assumption~\ref{assup:feature}). Denoted as $\Delta$, the information gap is 
defined as follows:
\begin{equation}
   \Delta \coloneqq  (1-\kappa_s)-2(1-\kappa_c). \label{eq:gap}
\end{equation}
Broadly speaking, a larger $\Delta$ means that the number of global features is much greater than local ones, indicating a significant imbalance. In contrast, a smaller value reflects only a slight imbalance.\footnote{Our study focuses on the regime where $\Delta$ is not too close to zero, i.e., $|\Delta| = \Omega(1)$, which allows for cleaner induction arguments. This condition could be potentially relaxed via more involved analysis.}

\paragraph{Unmasked area attention.} 
Based on the crucial role of those unmasked patches for both reconstruction task and positive contrastive pairs, we further define the \textit{unmasked area attention} as follows:
\begin{align*}
\textstyle   \Attn^{\dagger}_{{\pb}\to  \cP_{k,m} }(X; Q) \coloneqq \Score^{\dagger}_{{\pb}\to \cU \cap \cP_{k,m}}(X;Q), \text{ for } \dagger\in\{{\texttt{m}}, {\texttt{c}}\}.
\end{align*}

\subsection{MAE learns diverse attention patterns}

 Our results are structured into two parts: {\em i).} analysis of convergence (\Cref{thm:positive}), which includes the global convergence guarantee of the masked reconstruction loss and characterization of the attention pattern at the end of training to demonstrate the diverse locality; {\em ii).} learning dynamics of attention correlations (\Cref{thm:dynamics}), which shows how transformers capture target FP correlations while downplaying PP correlations 
as discussed in \Cref{sec:attn-cor}. 

To properly evaluate the reconstruction performance, we further introduce the following notion of the reconstruction loss with respect to a specific patch $\pb\in\cP$:
\begin{align}
  \textstyle \cL_{\texttt{mae}, \pb}(Q) &= \frac{1}{2}\mathbb{E}\left[\ind\{{\pb}\in\cM\}\Big\|[F^{\texttt{mae}}(\mask(X);Q)]_{{\pb}}-X_{{\pb}} \Big\|^2\right].\label{eq-obj-n}
\end{align}
Now we present our first main result regarding the convergence of MAE.  
\begin{theorem}[Training convergence]
\label{thm:positive} 
Suppose the information gap $\Delta\in [-0.5,-\Omega(1)]\cup[\Omega(1),1]$. For any $0<\epsilon<1$, suppose  $\polylog(P)\gg \log(\frac{1}{\epsilon})$. We train the ViTs in Definition~\ref{def:model-arch} by GD to minimize reconstruction loss in \cref{loss} with $\eta\ll \poly(P)$. 
Then for each patch $\pb\in\cP$, 
we have
\begin{enumerate}[label={\arabic*.}]
\item Loss converges: $\cL_{\texttt{mae}, \pb}(Q^{(T^{\star})})-\cL_{\texttt{mae}, \pb}^\star \leq \epsilon$ in $T^{\star}=O\Big(\frac{1}{\eta}\log(P)P^{\max\{2(\frac{U}{L}-1),1\}(1-\kappa_s)}+\frac{1}{\eta\epsilon}\log\big(\frac{P}{\epsilon}\big)\Big)$ iterations, where $\cL_{\texttt{mae}, \pb}^{\star}$  is the global minimum of the patch-level reconstruction loss in \eqref{eq-obj-n}.
\item {\bf Area-wide} pattern of attention: given cluster $k\in[K]$, and  $\pb\in \cP_{k,j}$ for some $j\in [N_k]$,   if $X_{\pb}$ is masked, then the one-layer transformer nearly ``pays all attention" to all {unmasked} patches in the same area $\cP_{k,j}$ as $\pb$, i.e.,
$$\textstyle\Big(1-\Attn^{\texttt{m}}_{{\pb}\to  \cP_{k,j}}\big(X; Q^{(T^{\star})}\big)\Big)^2 \leq O(\epsilon). $$
\end{enumerate}
\end{theorem} 

\Cref{thm:positive} indicates that, at the time of convergence,  for any masked query patch $X_{\pb}$ in the $k$-th cluster, the transformer exhibits an \emph{area-wide} pattern of attention, concentrating on those unmasked patches within the area that $\pb$ lies in, 
as demonstrated in \Cref{sec:attn-cor}.  The location of the patch determines such area-wide attention and can be achieved no matter if $\pb$ belongs to the global or local areas, which jointly highlight the {\bf diverse local patterns} for masked vision pretraining no matter degree of the imbalance. 

\begin{figure}[tp]
   \centering    \includegraphics[width=.9\linewidth]{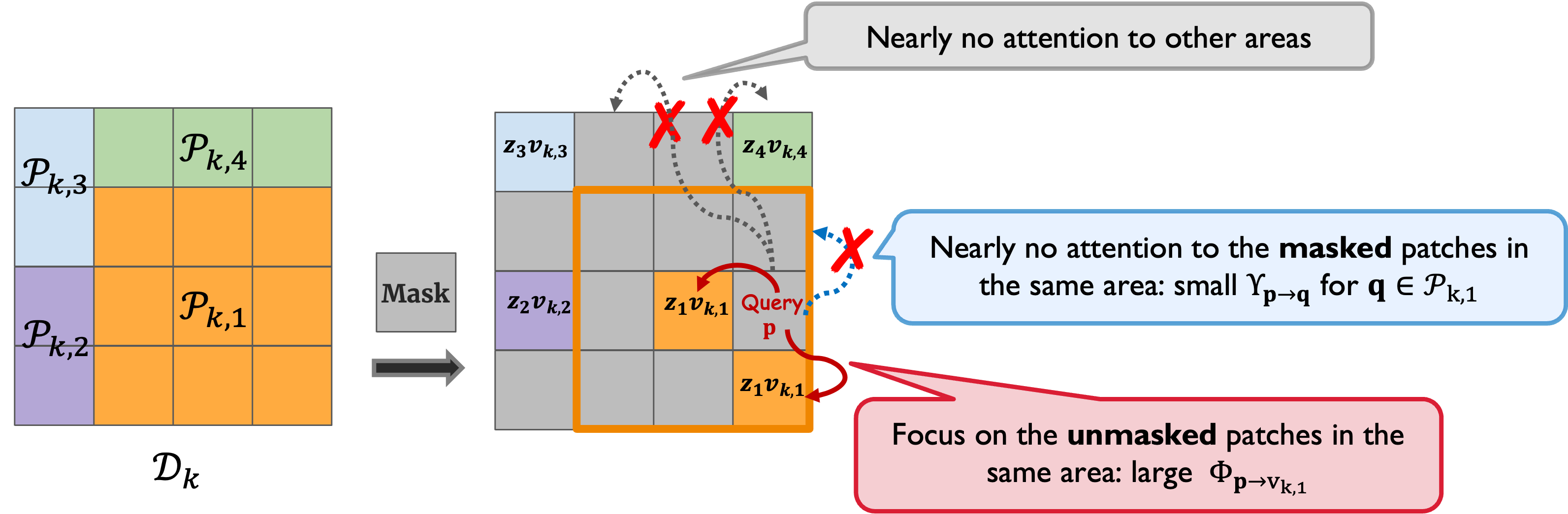}
   \caption{The mechanism of how the masked patch attends to other patches via attention correlations in MAE. 
   }
   \label{fig:attention}
\end{figure}

Next, we detail the training phases of attention correlations in the following theorem, which explicitly 
confirms that the model {\bf learns} target  FP correlations while {\bf ignoring} PP correlations to achieve the desirable area-wide attention patterns as suggested in \Cref{sec:attn-cor} (illustrated in \Cref{fig:attention}). 
\begin{theorem}[Learning Feature-Position correlations]
\label{thm:dynamics} 
   Following the same assumptions in \Cref{thm:positive}, for $\pb\in\cP$, given $k\in[K]$, if $\pb\in\cP_{k,j}$ for some $j\in [N_k]$, we have\\

\noindent  For positive information gap $\Delta\in [\Omega(1),1]$:
   \begin{enumerate}[label={\alph*.}]
       \item {Global areas ($j=1$) learn FP correlation in \bf one-phase:} 
       $\Phi^{(t)}_{\pb\to v_{k,1}}$ monotonically increases to  $O(\log(P/{\epsilon}))$ throughout the training, with all other attention correlations remain close to $0$.
       \item {Local areas ($j > 1$) learn FP correlation in {\bf  two-phase}}:  
       In phase one, FP correlation $\Phi^{(t)}_{\pb\to v_{k,1}}$ between local area and the global area feature quickly decreases to $-\Theta(\log(P))$ whereas all other attention correlation stay close to zero; 
    In phase two, FP correlation $\Phi^{(t)}_{\pb\to v_{k,j}}$ for the target local area starts to grow until convergence with all other attention correlations nearly unchanged. 
    \end{enumerate}
        For negative information gap $\Delta\in [-0.5,-\Omega(1)]$:
        \begin{enumerate}
        \item[c.] {All areas
learn FP correlation through \bf one-phase:} $\Phi^{(t)}_{\pb\to v_{k,j}}$ monotonically increases to  $O(\log(P/{\epsilon}))$ throughout the training, with all other attention correlations remain close to $0$.
   \end{enumerate}

\end{theorem}
The training dynamics are different depending on whether $\Delta$ is positive or negative, and further vary for positive $\Delta$ depending on whether $X_{\pb}$ is situated in global or local areas. Typically, the target FP correlations are learned directly in a single phase. However, for a positive information gap $\Delta$, when patch $\pb$ is located in a local area, the learning process contains an additional decoupling phase, to reduce the FP correlation with the non-target global features. 



\subsection{Contrastive learning collapses to global attention patterns}
In contrast to MAE's ability to learn diverse local features regardless of the information gap, our results in this section demonstrate that CL inevitably collapses to global attention patterns by solely learning global FP correlations, even under a slight structural imbalance. To prevent trivial solutions in CL, we adopt a noisy variant of the data distribution.



\begin{assumption}[Noisy data]
We assume that the data used for contrastive learning is sampled from \( \cD^{\texttt{cl}} \). Specifically, to generate a sample \( X \sim \cD^{\texttt{cl}} \), we first draw \( Z \sim \cD \), then add independent and identically distributed (i.i.d.) noise \( \zeta_{\pb} \sim \mathcal{N}(0, \sigma_0^2 I_d) \) to each patch \( Z_{\pb} \). The resulting sample is defined as \( X_{\pb} = Z_{\pb} + \zeta_{\pb} \). We denote \( X \in \cD_{k}^{\texttt{cl}} \) if \( Z \in \cD_{k} \).

\end{assumption}


\begin{theorem}[Learning with contrastive objective]
\label{thm:cl-convergence} 
Suppose the information gap $\Delta\in [-0.5,-\Omega(1)]\cup[\Omega(1),1]$.  We train the ViTs in Definition~\ref{def:model-arch-cl} by GD to minimize  \cref{obj-cl} with $\eta\ll \poly(P)$, $\sigma_0^2=\frac{1}{d}$, $\tau=O(\frac{1}{\log d})$. Then after $T^{\star}=O(\frac{\poly(P)\log P}{\eta})$ iterations, 
we have
\begin{enumerate}[label={\arabic*.}]
\item {Loss converges:} $
\cL_{\texttt{cl}}(Q^{(T^{\star})}) \leq \cL_{\texttt{cl}}^{\star} + \frac{1}{\poly(P)}$, where $\cL_{\texttt{cl}}^{\star}$  is the global minimum of the contrastive loss in \cref{obj-cl}. 
\item  Attention concentration on the {\bf global} area: given  $X\in\cD^{\texttt{cl}}_k$ with $k\in[K]$, for any $\pb\in\cP$, with high probability,  we have $1-\Attn^{\texttt{c}}_{{\pb}\to  \cP_{k,1}}(X'; Q^{(T^{\star})})=\frac{1}{\poly(P)}$ for $X'\in\{X^{+}, X^{++}\}$.\footnote{This also holds when no data augmentation is applied to \( X \).} 
\item All patches learn global FP correlation: given $k\in [K]$, for any $\pb\in\cP$, $t\in [0, T^{\star}]$,  $\Phi^{(t)}_{\pb\to v_{k,1}}\gg \Phi^{(t)}_{\pb\to v_{k,m}}$ with $m>1$, and at the  convergence, $\Phi^{(T^{\star})}_{\pb\to v_{k,1}}=\Theta(\log P),  \Phi^{(T^{\star})}_{\pb\to v_{k,m}}=o(1)$. 
\end{enumerate}
\end{theorem} 
\paragraph{Intuition behind learning global correlations.} 
As discussed in \Cref{sec:attn-cor}, the optimal alignment of two positive representations \( F^{\texttt{cl}}(Q; X^+) \) and \( F^{\texttt{cl}}(Q; X^{++}) \) involves directing attention towards the same feature for each patch $\pb$, possibly irrespective of its location. 
As long as the imbalanced structure, where global features dominate the data distribution, exists—even to a small degree—it leads to an order-wise stronger concentration of attention on global areas at initialization. Consequently, global FP correlations receive larger gradients compared to local ones. Therefore, global FP correlations are learned first, and focusing on these global correlations is sufficient for the CL objective to converge.

\paragraph{Significance of the results.} \Cref{thm:positive} and \Cref{thm:cl-convergence}  address a critical gap in understanding self-supervised pretraining by offering the first theoretical framework for learning with ViTs, one of the most advanced architectures in vision practice, whereas prior studies have primarily focused on linear models, CNNs, or MLPs~\citep{wen2021toward, ji2023power, pan2022towards}. Moreover, by identifying the collapsed solution in CL and emphasizing the effectiveness of MAE in capturing diverse attention patterns, we provide a qualitative comparison between MAE and contrastive learning, validating a non-trivial empirical observation~\citep{park2023what}. This offers a comprehensive theoretical analysis of self-supervised learning with ViTs.

\section{Overview of the Proof Techniques}\label{sec:overview}
In this section, we explain our key proof techniques in analyzing the self-supervised pretraining of transformers, using MAE as an example.  We focus on the reconstruction of a specific patch $X_{\pb}$ for $\pb\in\cP$. We aim to elucidate the training phases through which the model learns FP correlations related to the area associated with  $\pb$ across different clusters $k\in [K]$.  

Our characterization of training phases differentiates between whether $X_{\pb}$ is located in the global or local areas and further varies based on whether $\Delta$ is positive or negative.  Specifically, for $\Delta\in[\Omega(1),1]$, we observe distinct learning dynamics for FP correlations between local and global areas:
\begin{itemize}
 \item {Local} area attends to FP correlation in {two-phase}: given $k\in[K]$, if $a_{k,\pb}\neq 1$, then
\begin{enumerate}
 \item $\Phi^{(t)}_{\pb\to v_{k,1}}$ first quickly decreases whereas all other $\Phi^{(t)}_{\pb\to v_{k,m}}$ with $m\not=1$ and $\Upsilon^{(t)}_{\pb\to\qb}$ do not change much;
 \item after some point, 
 the increase of $\Phi^{(t)}_{\pb\to v_{k,a_{k,\pb}}}$ takes dominance. Such $\Phi^{(t)}_{\pb\to v_{k,a_{k,\pb}}}$ will keep growing until convergence with all other FP and PP attention correlations nearly unchanged. 
\end{enumerate}
\item {Global} areas learn FP correlation in {one-phase}: given $k\in[K]$, if $a_{k,\pb}=1$,  the update of $\Phi^{(t)}_{\pb\to v_{k,1}}$ will dominate throughout the training, whereas all other $\Phi^{(t)}_{\pb\to v_{k, m}}$ with $m\not=1$ and learned PP correlations remain close to $0$.
\end{itemize}
For $\Delta\in [-0.5,-\Omega(1)]$, the behaviors of learning FP correlations are uniform for all areas. Namely, all areas learn FP correlation through {one-phase}: given $k\in[K]$,  throughout the training,  the increase of $\Phi^{(t)}_{\pb\to v_{k,a_{k,\bp}}}$  dominates, whereas all other $\Phi^{(t)}_{\pb\to v_{k,m}}$ with $m\not=a_{k,\pb}$ and PP correlations $\Upsilon^{(t)}_{\pb\to\qb}$ remain close to $0$. 

For clarity, this section will mainly focus on the learning of {\em local} feature correlations with a positive information gap $\Delta\geq \Omega(1)$ in \Cref{sec:p1,sec:p2}, which exhibits a two-phase process. 
The other scenarios will be discussed briefly in \Cref{sec:other}.   

\subsection{Gradient dynamics of attention correlations}\label{sec-prep}
Based on the crucial roles that attention correlations play in determining the reconstruction loss, the main idea of our analysis is to track the dynamics of those attention correlations.  We first provide the following GD updates of $\Phi^{(t)}_{{\pb}\to v_{k,m}}$ and $\Upsilon_{\pb\to\qb}^{(t)}$ (see \Cref{sec-formal-gd} for formal statements). 

\begin{lemma}[FP correlations, informal]\label{lemma-feature-gd}
Given $k\in[K]$, for ${\pb}\in \cP$, denote $n=a_{k,\pb}$, let $ \alpha^{(t)}_{{\pb}\to {v_{k,m}}}=\frac{1}{\eta}\big(\Phi^{(t+1)}_{{\pb}\to v_{k,m}}-\Phi^{(t)}_{{\pb}\to v_{k,m}}\big)$ for $m\in[N_k]$, and suppose $X_{\pb}$ is masked.  Then
\begin{enumerate}
 \item for the same area, $\alpha^{(t)}_{{\pb}\to {v_{k,n}}}\approx  \Attn^{(t)}_{{\pb}\to \cP_{k,n}} \left(1-\Attn^{(t)}_{{\pb}\to \cP_{k,n}}\right)^2;$
   \item if $k\in\cB_{\pb}$, for the global area, 
\begin{align*}
  \alpha^{(t)}_{{\pb}\to {v_{k,1}}}&  \approx -\Attn^{(t)}_{{\pb}\to \cP_{k,1}} \cdot\Bigg(\Attn^{(t)}_{{\pb}\to \cP_{k,1}}\left(1-\Attn^{(t)}_{{\pb}\to \cP_{k,1}} \right)+\Attn^{(t)}_{{\pb}\to \cP_{k,n}}\left(1-\Attn^{(t)}_{{\pb}\to \cP_{k,n}} \right)\Bigg);
   \end{align*}
          \item for other area $m\notin\{n\}\cup \{1\}$,  
\begin{align*}
    & \alpha^{(t)}_{{\pb}\to {v_{k,m}}} \approx  \Attn^{(t)}_{{\pb}\to \cP_{k,m}}\Bigg( \ind{\{n\not=1\}}\left( \Attn^{(t)}_{{\pb}\to \cP_{k,1}}\right)^2-
    \left(1-\Attn^{(t)}_{{\pb}\to \cP_{k,n}} \right)\Attn^{(t)}_{{\pb}\to \cP_{k,n}}\Bigg).
   \end{align*}
\end{enumerate}

\end{lemma}
From Lemma~\ref{lemma-feature-gd}, it is observed that for ${\pb}\in \cP_{k,n}$, the feature correlation $\Phi^{(t)}_{{\pb}\to v_{k,n}}$ exhibits a monotonically increasing trend over time because $\alpha_{{\pb}\to v_{k,n}}^{(t)}\geq 0$. 
Furthermore, if $n>1$, i.e., $\cP_{k,n}$ is the local area, $\Phi^{(t)}_{{\pb}\to v_{k,1}}$ will monotonically decrease. 
\begin{lemma}[PP attention correlations, informal]\label{lemma-pos-gd}
Given  ${\pb}, {\qb}\in\cP$, let $ \beta^{(t)}_{{\pb}\to {\qb}}=\frac{1}{\eta}\big(\Upsilon^{(t+1)}_{{\pb}\to {\qb}}-\Upsilon^{(t)}_{{\pb}\to {\qb}}\big)$, and suppose $X_{\pb}$ is masked. Then
$
\beta^{(t)}_{{\pb}\to {\qb}}=\sum_{k\in[N]} \beta^{(t)}_{k, {\pb}\to {\qb}}$, where $\beta^{(t)}_{k, {\pb}\to {\qb}}$ satisfies
\begin{enumerate}
\item if $a_{k,\pb}=a_{k,\qb}=n$,\ $\beta^{(t)}_{k, {\pb}\to{\qb}}\approx \score^{(t)}_{{\pb}\to {\qb}} \left(1-\Attn^{(t)}_{{\pb}\to \cP_{k,n}}\right)^2;$
\item if $k\in\cB_{\pb}\cap \cC_{\qb}$, where  $a_{k,\pb}=n>1$ and $a_{k,\qb}=1$:  
\begin{align*}
    \beta^{(t)}_{k, {\pb}\to {\qb}}&\approx -\score^{(t)}_{{\pb}\to \qb} \cdot\Bigg(\Attn^{(t)}_{{\pb}\to \cP_{k,1}}\left(1-\Attn^{(t)}_{{\pb}\to \cP_{k,1}} \right)+\Attn^{(t)}_{{\pb}\to \cP_{k,n}}\left(1-\Attn^{(t)}_{{\pb}\to \cP_{k,n}} \right)\Bigg);
\end{align*}
\item if  $a_{k,\qb}=m\notin\{n\}\cup \{1\}$, where $a_{k,\pb}=n$, 
\begin{align*}
    &\beta^{(t)}_{k, {\pb}\to \qb} \approx  \score^{(t)}_{{\pb}\to \qb}\cdot\Bigg( \ind{\{n\not=1\}}\left( \Attn^{(t)}_{{\pb}\to \cP_{k,1}}\right)^2-
    \left(1-\Attn^{(t)}_{{\pb}\to \cP_{k,n}} \right)\Attn^{(t)}_{{\pb}\to \cP_{k,n}}\Bigg).
\end{align*}
\end{enumerate}
\end{lemma}
\vspace{-0.3cm}
Based on the above gradient update for $\Upsilon^{(t)}_{\pb\to\qb}$, we further introduce the following auxiliary quantity $\Upsilon^{(t)}_{k, \pb\to\qb}$, which can be interpreted as the PP attention correlation ``projected" on the $k$-th cluster $\cD_k$, and will be useful in the later proof:
\begin{align}
\Upsilon^{(t+1)}_{k, \pb\to\qb} \coloneqq  \Upsilon^{(t)}_{k, \pb\to\qb}+\eta \beta^{(t)}_{k, \pb\to\qb},  \quad \text{ with   } \Upsilon^{(0)}_{k, \pb\to\qb}=0. 
\end{align}
We can directly verify that $\Upsilon^{(t)}_{\pb\to\qb}=\sum_{k\in [K]}\Upsilon^{(t)}_{k, \pb\to\qb}$.

The key observation by comparing Lemma~\ref{lemma-feature-gd} and \ref{lemma-pos-gd} is that the gradient of projected PP attention $\beta^{(t)}_{k,{\pb}\to \qb}$ is smaller than the corresponding FP gradient $\alpha_{\pb\to v_{k,a_{k,\qb}}}^{(t)}$ in magnitude since $\textstyle\score^{(t)}_{{\pb}\to \qb}\approx \frac{\Attn^{(t)}_{{\pb}\to \cP_{k,a_{k,\qb}}}}{(1-\gamma)C_{k,a_{k,\qb}}}$. We will show that the interplay between the increase of $\Phi^{(t)}_{{\pb}\to v_{k,n}}$ and the decrease of $\Phi^{(t)}_{{\pb}\to v_{k,1}}$ determines the learning behaviors for the local patch $\pb\in\cP_{k,n}$ with $n>1$, and which effect will happen first depends on the initial attention, which is also determined by the value of information gap $\Delta$.

\subsection{Phase I: decoupling the global FP correlations}\label{sec:p1}
We now explain how the attention correlations evolve at the initial phase of the training to decouple the correlations of the non-target global features when $\pb$ is located in the local area for the $k$-th cluster. This phase can be further divided into the following two stages.
\paragraph{Stage 1.}At the beginning of training, $\Phi^{(0)}_{\pb\to v_{k,m}}=\Upsilon^{(0)}_{k, \pb\to\qb}=0$, and hence $\score^{(0)}_{\pb\to{\qb}}=\frac{1}{P}$ for any ${\qb}\in\cP$, which implies that the transformer equally attends to each patch.  However, with high probability,  the number of unmasked global features in the global area $\cP_{k,1}$ is much larger than others. Hence, $\Attn_{\pb\to\cP_{k,1}}^{(0)}= \frac{|\cU\cap \cP_{k,1}|}{P}\geq \Omega(\frac{1}{P^{1-\kappa_c}})\gg \Theta(\frac{1}{P^{1-\kappa_s}})=\Attn_{\pb\to\cP_{k,m}}^{(0)}$ for $m>1$. Therefore, by Lemma~\ref{lemma-feature-gd} and \ref{lemma-pos-gd}, we immediately obtain
\begin{itemize}
   \item $\alpha_{\pb\to v_{k,1}}^{(0)} =-\Theta\Big(\frac{1}{P^{2(1-\kappa_c)}}\Big)$, whereas $\alpha_{\pb\to v_{k,a_{k,\pb}}}^{(0)}=\Theta\Big(\frac{1}{P^{(1-\kappa_s)}}\Big)$;
   \item all other FP correlation gradients $\alpha_{\pb\to v_{k,m}}^{(0)}$ with $m\not=1, a_{k,\pb}$ are small;
   \item all projected PP correlation gradients $\beta^{(0)}_{k, \pb\to\qb}$ are small. 
\end{itemize}
 Since $\Delta=(1-\kappa_s)-2(1-\kappa_c)\geq \Omega(1)$, it can be seen that $\Phi_{\pb\to v_{k,1}}^{(t)}$ enjoys a much larger decreasing rate initially. This captures the decoupling process of the feature correlations with the global feature $v_{k,1}$ in the global area for $\pb$. It can be shown that such an effect will dominate over a certain period that defines stage 1 of phase I.  At the end of this stage, we will have  $\Phi_{\pb\to v_{k,1}}^{(t)}\leq - \Omega\left(\log(P)\right)$, whereas all FP attention correlation $\Phi_{\pb\to v_{k,m}}^{(t)}$ with $m>1$ and all projected PP correlations $\Upsilon
_{k,\pb\to\qb}^{(t)}$ stay close to $0$ (see \Cref{app:sec:p1-1}). 

During stage 1, the significant decrease of the global FP correlation $\Phi_{\pb\to v_{k,1}}^{(t)}$  leads to a reduction in the attention score $\Attn_{\pb\to \cP_{k,1}}^{(t)}$. Meanwhile, attention scores $\Attn^{(t)}_{\pb\to \cP_{k,m}}$ (where $m>1$) for other patches remain consistent, reflecting a uniform distribution over unmasked patches within each area. By the end of stage 1, $\Attn_{\pb\to \cP_{k,1}}^{(t)}$ drops to a certain level, resulting in a decrease in $|\alpha_{\pb\to v_{k,1}}^{(t)}|$ as it approaches  $\alpha_{\pb\to v_{k,n}}^{(t)}$, which indicates that stage 2 begins. 

\paragraph{Stage 2.} Soon as stage 2 begins, the dominant effect switches as $|\alpha^{(t)}_{\pb\to v_{k,1}}|$ reaches the same order of magnitude as $\alpha_{\pb\to v_{k, a_{k,\pb}}}^{(t)}$. The following result shows that $\Phi_{\pb\to v_{k, a_{k,\pb}}}^{(t)}$ must update during stage 2. 
\begin{lemma}[Switching of dominant effects (See \Cref{app:sec-p1-2})]\label{lem-p1-s2}
Under the same conditions as \Cref{thm:positive}, for $\pb\in \cP$, there exists $\widetilde{T}_{1}$, such that at iteration $t=\widetilde{T}_{1}+1$, we have 
   \begin{enumerate}
\item $\Phi_{\pb\to v_{k,a_{k,\pb}}}^{(\widetilde{T}_{1}+1)}\geq \Omega\left(\log(P)\right)$, and $\Phi_{\pb\to v_{k,1}}^{(\widetilde{T}_{1}+1)}= -\Theta(\log(P))$; 
\item all other FP correlations $\Phi_{\pb\to v_{k,m}}^{(t)}$ with $m\not=1, a_{k,\pb}$ are small;
   \item all projected PP correlations $\Upsilon^{(t)}_{k, \pb\to\qb}$ are small. 
\end{enumerate}
\end{lemma}
\vspace{-0.5cm}
\paragraph{Intuition of the transition.} 
Once $\Phi_{\pb\to v_{k,1}}^{(t)}$  decreases to $-\frac{\Delta}{2L}\log(P)$, we observe that $|\alpha_{\pb\to v_{k,1}}^{(t)}|$ is approximately equal to $\alpha_{\pb\to v_{k,a_{k,\pb}}}^{(t)}$. After this point, reducing $\Phi_{\pb\to v_{k,1}}^{(t)}$  further is more challenging compared to the increase in $\Phi_{\pb\to v_{k,a_{k,\pb}}}^{(t)}$. 
To illustrate, a minimal decrease of $\Phi_{\pb\to v_{k,1}}^{(t)}$ by an amount of $\frac{0.001}{L}\log(P)$ will yield $|\alpha_{\pb\to v_{k,1}}^{(t)}|\leq O(\frac{\alpha_{\pb\to v_{k,n}}^{(t)}}{P^{0.002}})$. Such a discrepancy triggers the switch of the dominant effect.

\subsection{Phase II: growth of target local FP correlation}  \label{sec:p2}
Moving beyond phase I, FP correlation $\Phi_{\pb\to v_{k, a_{k,\pb}}}^{(t)}$ within the target local area $\pb$ already enjoys a larger gradient $\alpha_{\pb\to v_{k, a_{k,\pb}}}^{(t)}$ than other $\Phi_{\pb\to v_{k, m}}^{(t)}$ with $m\not=a_{k,\pb}$ and all projected PP correlations  $\Upsilon_{k,\pb\to\qb}^{(t)}$. We can show that the growth of $\Phi_{\pb\to v_{k, a_{k,\pb}}}^{(t)}$ will continue to dominate until the end of training by recognizing the following two stages.
\paragraph{Rapid growth stage.} At the beginning of phase II, $\alpha_{\pb\to v_{k, a_{k,\pb}}}^{(t)}$ is mainly driven by $\Attn_{\pb\to\cP_{k, a_{k,\pb}}}^{(t)}$ since  $1-\Attn_{\pb\to\cP_{k, a_{k,\pb}}}^{(t)}$ remains at the constant order. Therefore, the growth of $\Phi_{\pb\to v_{k,a_{k,\pb}}}^{(t)}$ naturally results in a boost in $\Attn_{\pb\to\cP_{k, a_{k,\pb}}}^{(t)}$, thereby promoting an increase in its own gradient $\alpha_{\pb\to v_{k,a_{k,\pb}}}^{(t)}$, which defines the rapid growth stage. 
On the other hand, we can prove that the following gap holds for FP and projected  PP correlation gradients (see \Cref{sec:p2-s1}):
   \begin{itemize}
   \item all other FP correlation gradients $\alpha_{\pb\to v_{k,m}}^{(t)}$ with $m\not=a_{k,\pb}$ are small;
   \item all projected PP correlation gradients $\beta^{(t)}_{k, \pb\to\qb}$ are small. 
\end{itemize}

\paragraph{Convergence stage.}  After the rapid growth stage, the desired local pattern with a high target feature-position correlation $\Phi_{\pb\to v_{k,a_{k,\pb}}}^{(t)}$ is learned.  In this last stage,  it is demonstrated that the above conditions for non-target FP and projected PP correlations remain valid,  while the growth of $\Phi_{\pb\to v_{k,a_{k,\pb}}}^{(t)}$ starts to decelerate as $\Phi_{\pb\to v_{k,a_{k,\pb}}}^{(t)}$ reaches $\Theta(\log(P))$,  resulting in $\Attn_{\pb\to\cP_{k,n}}^{(t)}\approx\Omega(1)$, which leads to convergence (see \Cref{sec:p2-s2}). 


\subsection{Learning processes in other scenarios}\label{sec:other}
In this section, we talk about the learning process in other settings, including learning FP correlations for the local area when the information gap is negative, learning FP correlations for the global area, and failure to learn PP correlations. 

\paragraph{What is the role of positive information gap?} As described in stage 1 of phase 1 in \Cref{sec:p1}, the decoupling effect happens at the beginning of the training because $\alpha_{\pb\to v_{k,1}}^{(0)}\gg \alpha_{\pb\to v_{k,a_{k,\pb}}}^{(0)}$ attributed to $\Delta\geq \Omega(1)$. However, in cases where $\Delta\leq -\Omega(1)$, this relationship reverses, with $\alpha_{\pb\to v_{k,1}}^{(0)}$ becoming significantly smaller than $\alpha_{\pb\to v_{k,a_{k,\pb}}}^{(0)}$. Similarly, other FP gradients $\alpha_{\pb\to v_{k,m}}^{(0)}$ with $m\not=1, a_{k,\pb}$  and all the projected gradients of PP correlation $\beta_{\pb\to\qb}^{(0)}$ are small in magnitude.  Consequently, $\Phi_{\pb\to v_{k,a_{k,\pb}}}^{(t)}$ starts with a larger gradient, eliminating the need to decouple FP correlations for the global area. As a result, training skips the initial phase, and moves directly into Phase II, during which $\Phi_{\pb\to v_{k,a_{k,\pb}}}^{(t)}$ continues to increase until it converges (see \Cref{sec:back:neg}). 

\paragraph{Learning FP correlations for the global area.}
When the patch  $X_{\pb}$ is located in the global area of cluster $k$, i.e., $a_{k,\pb}=1$, the attention score $\Attn^{(0)}_{\pb\to\cP_{k,1}}$ directed towards the target area $\cP_{k,1}$ is initially higher compared to other attention scores due to the presence of a significant number of unmasked patches in the global area. This leads to an initially larger gradient $\alpha_{\pb\to v_{k,a_{k,\pb}}}^{(0)}$. Such an effect is independent of the value of $\Delta$. As a result, the training process skips the initial phase, which is typically necessary for the cases where $a_{k,\pb}>1$ with a positive information gap, and moves directly into Phase II (see \Cref{sec:glo}).

\paragraph{All PP  correlations are small.}
Integrating the analysis from all previous discussions, we establish that for every cluster $k\in [K]$, regardless of its association with $\cC_{\pb}$ (global area) or $\cB_{\pb}$ (local area), and for any patch $X_{\qb}$  with $\qb\in \cP$, the projected PP correlation $\Upsilon_{k,\pb\to\qb}^{(t)}$ remains nearly zero in comparison to the significant changes observed in the FP correlation, because the gradient $\beta_{k,\pb\to\qb}^{(t)}$ is relatively negligible.  Therefore, the overall PP correlation $\Upsilon_{\pb\to\qb}^{(t)} = \sum_{k=1}^{K} \Upsilon_{k,\pb\to\qb}^{(t)}$ also stays close to zero, given that the number of clusters $K= \Theta(1)$.

\section{Related Work}\label{sec-related}

\paragraph{Empirical studies of transformers in vision.} A number of works have aimed to understand the transformers in vision from different perspectives: comparison with CNNs~\citep{raghu2021vision,ghiasi2022vision,park2022vision}, robustness~\citep{bhojanapalli2021understanding, paul2022vision}, and role of positional embeddings~\citep{melas2021you,trockman2022patches}. 
Recent studies~\citep{xie2023revealing,wei2022contrastive,park2023what} have delved into ViTs with self-supervision to uncover the mechanisms at play, particularly through visualization and analysis of metrics related to self-attention. 
\cite{xie2023revealing} compared the masked image modeling (MIM) method with supervised models, 
revealing MIM's capacity to enhance diversity and locality across all ViT layers, w which significantly boosts performance on tasks with weak semantics following fine-tuning. Building on MIM's advantages, 
\cite{wei2022contrastive} further 
proposed a simple feature distillation method that incorporates locality into various self-supervised methods, leading to an overall improvement in the finetuning performance. \cite{park2023what} conducted a detailed comparison between masked image modeling (MIM) and contrastive learning. They demonstrated that contrastive learning will make the self-attentions collapse into homogeneity for all query patches due to the nature of discriminative learning, while MIM leads to a diverse self-attention map since it focuses on local patterns.

\paragraph{Theory of self-supervised learning.} A major line of theoretical studies falls into one of the most successful
self-supervised learning approaches, contrastive learning~\citep{wen2021toward,robinson2021can,chen2021intriguing,arora2019theoretical}, and its variant non-contrastive self-supervised learning~\citep{wen2022mechanism,pokle2022contrasting,wang2021towards}. Some other works study the mask prediction approach~\citep{lee2021predicting,wei2021pretrained, liu2022masked}, which is the focus of this paper. \cite{lee2021predicting} provided statistical downstream guarantees for reconstructing missing patches. \cite{wei2021pretrained} studied the benefits of head and prompt tuning with masked pretraining under a Hidden Markov Model framework.   \cite{liu2022masked} provided a parameter identifiability view to understand the benefit of masked prediction tasks, which linked the masked reconstruction tasks to the informativeness of the representation via identifiability techniques from tensor decomposition.  



\paragraph{Theory of transformers and attention models.} Prior work has studied the theoretical properties of transformers from various aspects: representational power~\citep{yun2019transformers,edelman2022inductive,vuckovic2020mathematical,wei2022statistically, sanford2024transformers}, internal mechanism~\citep{ tarzanagh2023transformers,weiss2021thinking}, limitations~\citep{hahn2020theoretical,sanford2024representational}, and PAC learning~\citep{chen2024provably}.  Recently, there has been a growing body of research studying in-context learning with transformers due to the remarkable emergent in-context ability of large language models~\citep{zhang2023and, von2023transformers, giannou2023looped,ahn2023transformers,zhang2023trained, huang2023context, nichani2024transformers, li2024training}. Regarding the training dynamics of attention-based models, \cite{li2023theoretical} studied the training process of shallow ViTs in a classification task. Subsequent research expanded on this by exploring the graph transformer with positional encoding~\citep{li2023improves} and in-context learning performance of transformers with nonlinear self-attention and nonlinear MLP~\citep{li2024training}.  
However, all of these analyses
rely crucially on stringent assumptions on the initialization of transformers and hardly generalize to our setting. \cite{tian2023scan} mathematically described how the attention map evolves trained by SGD for one-layer transformer but did not provide any convergence guarantee, and the follow-up work~\cite{tian2024joma} considered a generalized case with multiple layers. \cite{tarzanagh2023maxmargin,vasudeva2024implicit} investigated the implicit bias for self-attention models trained with GD. 
Furthermore, \cite{huang2023context} proved the in-context convergence of a one-layer softmax transformer trained via GD and illustrated the attention dynamics throughout the training process. \cite{yang2024context} generalized such an in-context learning problem to a mult-head setting with non-linear task functions. \cite{nichani2024transformers} studied  GD dynamics on a simplified two-layer
attention-only transformer and proved that it can encode the causal structure in the first attention
layer.  However, none of the previous studies analyzed the training of transformers under self-supervised learning, which is the focus of this paper. 

\section{Experiments}\label{sec:exp}

Previous studies on the attention mechanisms of ViT-based pre-training approaches have mainly utilized a metric known as the attention distance~\citep{dosovitskiy2020image}. Such a metric quantifies the average spatial distance between the query and key tokens, weighted by their self-attention coefficients. The general interpretation is that larger attention distances indicate global understanding, and smaller values suggest a focus on local features. However, such a metric does not adequately determine if the self-attention mechanism is identifying a unique global pattern. A high attention distance could result from different patches focusing on varied distant areas, which does not necessarily imply that global information is being effectively synthesized. To address this limitation, we introduce a novel and revised version of average attention distance,  called the attention diversity metric, which is designed to assess whether various patches are concentrating on a similar region, thereby directly capturing global information. 

\paragraph{Attention diversity metric, in distance.} This metric is computed for self-attention with a single head of the specific layer. For a given image divided into $N\times N$ patches, the process unfolds as follows: for each patch, it is employed as the query patch to calculate the attention weights towards all $N^2$ patches, and those with the top-$n$ attention weights are selected. Subsequently, the coordinates (e.g. $(i,j)$ with $ i,j\in [N]$)  of these top-$n$ patches are concatenated in sequence to form a $2\times n$-dimensional vector. The final step computes the average distance between all these $2n$-dimensional vectors, i.e., $N^2\times N^2$ vector pairs.  

\paragraph{Setup.} In this work, we compare the performance of ViT-B/16 encoder pre-trained on ImageNet-1K~\citep{russakovsky2015imagenet} among the following four models: masked reconstruction model (MAE), contrastive learning model (MoCo v3~\citep{Chen2021AnES}), other self-supervised model (DINO~\cite{caron2021emerging}), and supervised model (DeiT~\cite{touvron2021training}).   We focus on $12$ different attention heads in the last layer of ViT-B on different pre-trained models. The box plot visualizes the distribution of the top-10 averaged attention focus across 152 example images, as similarly done in \cite{dosovitskiy2020image}.

\paragraph{Implications.} The experiment results based on our new metric are provided in \Cref{fig:enter-label}.  Lower values of the attention diversity metric signify a focused attention on a coherent area across different patches, reflecting a global pattern of focus. On the other hand, higher values suggest that attention is dispersed, focusing on different, localized areas. It can be seen that the masked pretraining model is particularly effective in learning more diverse attention patterns, setting it apart from other models that prioritize a uniform global information with less attention diversity. This aligns with and provides further evidence for the findings in \cite{park2023what}.

\section{Conclusion}
In this work, we study the training process of MAE and CL with one-layer softmax-based ViTs.
Our key contribution is providing the first end-to-end convergence guarantees for these two prominent self-supervised approaches with transformer architectures. We characterize the attention patterns at convergence and show that MAE exhibits diverse attention patterns by learning feature-position correlations across all features, even with highly skewed feature distributions. In contrast, CL collapses to global attention patterns by focusing solely on global feature-position correlations, despite minimal distributional deviations between features. This provides theoretical justification for the behavior gap of MAE and CL observed in practice. Our proof techniques use phase decomposition based on the interplay between feature-position and position-wise correlations, avoiding the need to disentangle patches and positional encodings as in prior work. We anticipate that our theory will be valuable for future studies of spatial or temporal structures in state-of-the-art transformers and will advance theoretical research in deep learning.

\section*{Acknowledgements}

The work of Y. Huang is partially supported in part by the ONR grant N00014-22-1-2354, and the NSF grants CCF-2418156 and DMS-2143215. The work of Z. Wen and Y. Chi is supported in part by NSF under CCF-1901199, CCF-2007911, DMS-2134080, DMS-2134133, and by ONR under N00014-19-1-2404. The work of Y. Liang is supported in part by NSF under RINGS-2148253, CCF-1900145, and DMS-2134145.

\bibliography{mybib}
\bibliographystyle{alphaabbr}

\clearpage 
\appendix

\section{Proof of Main Theorems for MAE}\label{sec:proof:main}

\subsection{Preliminaries}
In this section, we will introduce warm-up gradient computations and probabilistic lemmas that establish essential properties of the data and the loss function, which are pivotal for the technical proofs in the upcoming sections for masked pretraining. 
{Throughout the appendix, we assume $N_{k}=N$ and $C_{k,n}=C_n$ for all $k\in[K]$  for simplicity. We will also omit the explicit dependence on $X$ for $z_{n}(X)$. We use $k_{X}\in [K]$ to denote the cluster index that a given image $X$ is drawn from.  } Furthermore, we will abbreviate $\cL_{\texttt{mae}} (\cL_{\texttt{mae}, \pb})$  as $\cL (\cL_{\pb})$, and $F^{\texttt{mae}}$ as $F$ for simplicity, when the context makes it clear. We abbreviate $\Attn^{\texttt{m}}_{{\pb}\to  \cP_{k,m} }(X;Q^{(t)})$  ($\score^{\texttt{m}}_{\pb\to\qb}(X;Q^{(t)})$) as $\Attn_{{\pb}\to  \cP_{k,m} }^{(t)}$( $\score_{\pb\to\qb}^{(t)}$), when the context makes it clear.  

\subsubsection{Gradient computations}
We first calculate the gradient with respect to $Q$. We omit the superscript `$(t)$' and write $\cL(Q)$ as $\cL$ here for simplicity. 
\begin{lemma}\label{lem-general-gd}
The gradient of the loss function with respect to $Q$ is given by
\begin{align*}
 \frac{\partial \cL}{\partial Q} 
=-\EE&\left[\sum_{{\pb}\in\cM}\right.\sum_{{\qb}}\score_{{\pb}\to{\qb}}\mask(X)^{\top}_{{\qb}}(X_{{\pb}} - [F(\mask(X); Q)]_{{\pb}})\cdot\\
&\qquad\left.\mmX_{{\pb}}\left(\mmX_{{\qb}}-\sum_{{\rb}}\score_{{\pb}\to{\rb}}\mmX_{{\rb}}\right)^{\top}\right] .
\end{align*}
\end{lemma}
\begin{proof}
We begin with the chain rule and obtain
\begin{align}
\frac{\partial \cL}{\partial Q} &=\EE \left[\sum_{{\pb}\in\cM}\frac{\partial [F(\mask(X); Q)]_{{\pb}}}{\partial Q} ([F(\mask(X);Q)]_{{\pb}}-X_{{\pb}}) \right]\notag\\
&=\EE \left[\sum_{{\pb}\in\cM}\sum_{{\qb}}\frac{\partial \score_{{\pb}\to{\qb}}}{\partial Q} \mask(X)^{\top}_{{\qb}}([F(\mask(X); Q)]_{{\pb}}-X_{{\pb}}) \right] .\label{eq:chain}
\end{align}
We focus on the gradient for each attention score: 
\begin{align*}
\frac{\partial \score_{{\pb}\to{\qb}}}{\partial Q}
&=\sum_{{\rb}}\frac{\exp\left(\mmX^{\top}_{{\pb}}Q(\mmX_{{\rb}}+\mmX_{{\qb}})\right)}{\left(\sum_{{\rb}}\exp(\mmX^{\top}_{{\pb}}Q\mmX_{{\rb}})\right)^2} \mmX_{{\pb}}(\mmX_{{\qb}}-\mmX_{{\rb}})^{\top}\\
&=\score_{{\pb}\to{\qb}}\sum_{{\rb}}\score_{{\pb}\to{\rb}}\mmX_{{\pb}}(\mmX_{{\qb}}-\mmX_{{\rb}})^{\top}\\
    &=\score_{{\pb}\to{\qb}}\mmX_{{\pb}}\cdot \left[\mmX_{{\qb}}-\sum_{{\rb}}\score_{{\pb}\to{\rb}}\mmX_{{\rb}}\right]^{\top}.
\end{align*}
Substituting the above equation into \eqref{eq:chain}, we complete the proof.
\end{proof}

Recall that the quantities $\Phi^{(t)}_{{\pb}\to v_{k,m}}$ and $\Upsilon^{(t)}_{{\pb}\to {\qb}}$ are defined in Definition~\ref{def:attn-dynamics}. These quantities are associated with the attention weights for each token, and they play a crucial role in our analysis of learning dynamics. We will restate their definitions here for clarity.

\begin{definition}(Attention correlations)
Given ${\pb}, {\qb}\in \cP$ 
, for $t\geq 0$, we define two types of  attention correlations as follows:
\begin{enumerate}
  \item Feature Attention Correlation: $\Phi^{(t)}_{{\pb}\to v_{k,m}} \coloneqq e^{\top}_{{\pb}}Q^{(t)}v_{k,m}$  for $k\in [K]$ and $m\in [N]$;
  \item Positional Attention Correlation: $\Upsilon^{(t)}_{{\pb}\to {\qb}} \coloneqq e_{{\pb}}^{\top}Q^{(t)}e_{{\qb}}$.
\end{enumerate}
By our initialization, we have $\Phi^{(0)}_{{\pb}\to v_{k,m}}=\Upsilon^{(0)}_{{\pb}\to {\qb}}=0$.
\end{definition}

Next, we will apply the expression in Lemma~\ref{lem-general-gd} to compute the gradient dynamics of these attention correlations.
\subsubsection{Formal statements and proofs of Lemma~\ref{lemma-feature-gd} and \ref{lemma-pos-gd}}\label{sec-formal-gd}
We first introduce some notation. Given ${\rb}\in \cU$, for ${\pb}\in \cP$, $k\in [K]$ and $n\in[N]$ define the following quantities:
\begin{align*}
    J^{{\pb}}_{{\rb}}& \coloneqq  \mask(X)_{{\rb}}^{\top}(X_{{\pb}}-[F(\mask(X); Q)]_{{\pb}}), \\
    I^{{\pb},k,n}_{{\rb}}& \coloneqq \left(\mmX_{{\rb}}-\sum_{\wb\in\cP}\score_{{\pb}\to \wb }\mmX_{\wb}\right)^{\top}v_{k,n} , \\
    K^{{\pb},{\qb}}_{{\rb}}& \coloneqq \left(\mmX_{{\rb}}-\sum_{\wb\in\cP}\score_{{\pb}\to\wb}\mmX_{\wb}\right)^{\top}e_{{\qb}} . 
\end{align*}

\begin{lemma}[Formal statement of Lemma~\ref{lemma-feature-gd}]\label{app-lemma-feature-gd}
Given $k\in[K]$, for ${\pb}\in\cP$, denote $n=a_{k,\pb}$. Letting $ \alpha^{(t)}_{{\pb}\to {v_{k,m}}}=\frac{1}{\eta}\big(\Phi^{(t+1)}_{{\pb}\to v_{k,m}}-\Phi^{(t)}_{{\pb}\to v_{k,m}}\big)$ for $m\in[N_k]$, then
\begin{enumerate}
 \item for $m=n$, 
     \begin{align*}
    \alpha^{(t)}_{{\pb}\to {v_{k,n}}}=\EE&\Bigg[\1\{{\pb}\in\mathcal{M}, k_X=k\}\Attn^{(t)}_{{\pb}\to \cP_{k,n}}\cdot\\
    &\quad \Big(z_n^3\left(1-\Attn^{(t)}_{{\pb}\to \cP_{k,n}}\right)^2+ \sum_{a\not=n} z_a^2z_n \left(\Attn^{(t)}_{{\pb}\to \cP_{k,a}}\right)^2   \Big)\Bigg];
   \end{align*}
   \item for $m\not=n$, 
\begin{align*}
    \alpha^{(t)}_{{\pb}\to {v_{k,m}}} &=\EE\Bigg[\1{\{{\pb}\in\mathcal{M}, k_X=k\}} \Attn^{(t)}_{{\pb}\to \cP_{k,m}}\cdot\Bigg( \sum_{a\not=m,n} z_a^2z_m\left(\Attn^{(t)}_{{\pb}\to \cP_{k,a}}\right)^2-\\
    &\quad \left(z_mz_n^2\left(1-\Attn^{(t)}_{{\pb}\to \cP_{k,n}} \right)\Attn^{(t)}_{{\pb}\to \cP_{k,n}}+z_m^3 \left(1-\Attn^{(t)}_{{\pb}\to \cP_{k,m}} \right)\Attn^{(t)}_{{\pb}\to \cP_{k,m}}   \right)\Bigg)\Bigg].
   \end{align*}
\end{enumerate}

\end{lemma}

\begin{proof} From Lemma~\ref{lem-general-gd}, we have 
\begin{align*}
  \alpha^{(t)}_{{\pb}\to {v_{k,m}}}&= e_{{\pb}}^{\top}(-\frac{\partial \cL}{\partial Q}){v_{k,m}}\\
  &=\EE[\1{\{{\pb}\in\mathcal{M}\}}\sum_{{\rb}\in\mathcal{U}}\score_{{\pb}\to{\rb}}J^{{\pb}}_{{\rb}}\cdot I^{{\pb},k,m}_{{\rb}}]\\
  &=\EE[\1{\{{\pb}\in\mathcal{M}, k_X=k\}}\sum_{{\rb}\in\mathcal{U}}\score_{{\pb}\to{\rb}}J^{{\pb}}_{{\rb}}\cdot I^{{\pb},k,m}_{{\rb}}] ,
\end{align*}
where the last equality holds since when $k_X\not= k$, $I^{{\pb},k,m}_{{\rb}}=0$ due to orthogonality.  Thus, in the following, we only need to consider the case $k_X=k$. 
\paragraph{Case 1: $m=n$.}
\begin{itemize}
        \item  For ${\rb}\in\mathcal{U}\cap\cP_{k,n}$, since $v_{k,n^{\prime}}\perp v_{k,n}$ for $n^{\prime}\not=n$, and $v_{k,n}\perp \{e_{{\qb}}\}_{{\qb}\in\cP}$ we have 
        \begin{align*}
            J^{{\pb}}_{{\rb}}&= z_nv_{k,n}^{\top} \Big(z_nv_{k,n}-\sum_{{\qb}\in\mathcal{U}\cap\cP_{k,n}} \score_{{\pb}\to{\qb}} z_nv_{k,n} \Big) = z_n^2\left(1-\Attn_{{\pb}\to \cP_{k,n}} \right) , \\
            I^{{\pb},k,n}_{{\rb}}&=(z_nv_{k,n}-\sum_{{\qb}\in\mathcal{U}\cap\cP_{k,n}} \score_{{\pb}\to{\qb}} z_nv_{k,n})^{\top}v_{k,n}=J^{{\pb}}_{{\rb}}/z_n .
        \end{align*}
        \item For ${\rb}\in\mathcal{U}\cap\cP_{k,n^{\prime}}$ with $n^{\prime}\not=n$
       \begin{align*}
            J^{{\pb}}_{{\rb}}&= z_{n^{\prime}}v_{k,n^{\prime}}^{\top}\Big(z_nv_{k,n}-\sum_{{\qb}\in\mathcal{U}\cap\cP_{k,n^{\prime}}} \score_{{\pb}\to{\qb}} z_{n^{\prime}}v_{k,n^{\prime}} \Big) =-z_{n^{\prime}}^2\Attn_{{\pb}\to \cP_{k,n^{\prime}}}, \\
            I^{{\pb},k,n}_{{\rb}}&=\Big(z_{n^{\prime}}v_{k,n^{\prime}}-\sum_{{\qb}\in\mathcal{U}\cap\cP_{k,n}}\score_{{\pb}\to{\qb}} z_nv_{k,n} \Big)^{\top}v_{k,n} =-z_n\Attn_{{\pb}\to \cP_{k,n}}. 
        \end{align*}
       \end{itemize}  
       Putting it together,   we obtain:
       \begin{align*}
         e_{{\pb}}^{\top}(-\frac{\partial L}{\partial Q})v_{k,n} & =\EE\Bigg[\1\{\{{\pb}\in\mathcal{M}, k_X=k\}\} \Attn^{(t)}_{{\pb}\to \cP_{k,n}} \cdot\\
    & \quad \quad \Big(z_n^3\left(1-\Attn^{(t)}_{{\pb}\to \cP_{k,n}}\right)^2+ \sum_{a\not=n} z_a^2z_n \left(\Attn^{(t)}_{{\pb}\to \cP_{k,a}}\right)^2   \Big)\Bigg] .
       \end{align*}
       \paragraph{Case $2$: $m\not=n$.}   Similarly,
\begin{itemize}
\item  For ${\rb}\in\mathcal{U}\cap\cP_{k,n}$
\begin{align*}
    J^{{\pb}}_{{\rb}}&= z_nv_{k,n}^{\top}\Big(z_nv_{k,n}-\sum_{{\qb}\in\mathcal{U}\cap\cP_{k,n}}\score_{{\pb}\to{\qb}} z_nv_{k,n} \Big) =z_n^2(1-\Attn_{{\pb}\to\cP_{k,n}}), 
     \\ 
    I^{{\pb},k,m}_{{\rb}}&=\Big(z_nv_{k,n}-\sum_{{\qb}\in\mathcal{U}\cap\cP_{k,m}} \score_{{\pb}\to{\qb}} z_mv_{k,m}\Big)^{\top}v_{k,m} =-z_m \Attn_{{\pb}\to\cP_{k,m}} .
\end{align*}

\item For ${\rb}\in\mathcal{U}\cap\cP_{k,m}$
\begin{align*}
    J^{{\pb}}_{{\rb}}&= z_mv_{k,m}^{\top}\Big(z_nv_{k,n}-\sum_{{\qb}\in\mathcal{U}\cap\cP_{k,m}} \score^{(t)}_{{\pb}\to{\qb}}  z_mv_{k,m} \Big) =-z_m^2\Attn_{{\pb}\to\cP_{k,m}} , \\
    I^{{\pb},k,n}_{{\rb}}&=\Big(z_mv_{k,m}-\sum_{{\qb}\in\mathcal{U}\cap\cP_{k,m}} \score^{(t)}_{{\pb}\to{\qb}}  z_mv_{k,m} \Big)^{\top}v_{k,m} =z_n(1-\Attn_{{\pb}\to\cP_{k,m}} ) .
\end{align*}
\item For ${\rb}\in\mathcal{U}\cap\cP_{k,a}$, $a\not=n,m$
\begin{align*}
    J^{{\pb}}_{{\rb}}&= z_av_{k,a}^{\top}\Big(z_nv_{k,n}-\sum_{{\qb}\in\mathcal{U}\cap\cP_{k,a}} \score^{(t)}_{{\pb}\to{\qb}}  z_av_{k,a} \Big) =-z_a^2 \Attn_{{\pb}\to\cP_{k,a}} , \\
    I^{{\pb},k,n}_{{\rb}}&=\Big(z_av_{k,a}-\sum_{{\qb}\in\mathcal{U}\cap\cP_{k,m}} \score^{(t)}_{{\pb}\to{\qb}}  z_mv_{k,m} \Big)^{\top}v_{k,m} =-z_m \Attn_{{\pb}\to\cP_{k,m}}.
\end{align*}
\end{itemize}    
Putting them together, then we complete the proof.  
\end{proof}

\begin{lemma}[Formal statement of Lemma~\ref{lemma-pos-gd}]\label{app-lemma-pos-gd}
 Given  ${\pb}, {\qb}\in\cP$, let $ \beta^{(t)}_{{\pb}\to {\qb}}=\frac{1}{\eta}\big(\Upsilon^{(t+1)}_{{\pb}\to {\qb})}-\Upsilon^{(t)}_{{\pb}\to {\qb})}\big)$, then
 $$
 \beta^{(t)}_{{\pb}\to {\qb}}=\sum_{k\in[K]} \beta^{(t)}_{k, {\pb}\to {\qb}}, 
 $$
 where $\beta^{(t)}_{k, {\pb}\to {\qb}}$ satisfies
 \begin{enumerate}
     \item if $a_{k,\pb}=a_{k,\qb}=n$,
     \begin{align*}
\beta^{(t)}_{k, {\pb}\to{\qb}}=\EE&\Bigg[\1{\{{\pb}\in\mathcal{M}, k_X=k\}}\score^{(t)}_{{\pb}\to {\qb}}\cdot \\
        &\Bigg(\sum_{a\not=n} z^2_a\left(\Attn^{(t)}_{{\pb}\to \cP_{k,a}}\right)^2+ z^2_n\left(1-\Attn^{(t)}_{{\pb}\to \cP_{k,n}}\right)\left(\1{\{{\qb}\in \mathcal{U}\}}-\Attn^{(t)}_{{\pb}\to \cP_{k,n}}\right)\Bigg)\Bigg];
    \end{align*}
    \item for $a_{k,\pb}=n\not=m=a_{k,\qb}$,
    \begin{align*}
    \beta^{(t)}_{k,{\pb}\to{\qb}}&=\EE\Bigg[\1{\{{\pb}\in\mathcal{M},, k_X=k\}}\score^{(t)}_{{\pb}\to {\qb}}\cdot\\
    &\Bigg(\sum_{a\not=n} z_a^2\left(\Attn^{(t)}_{{\pb}\to \cP_{k,a}}\right)^2- \left(z_n^2\left(1-\Attn^{(t)}_{{\pb}\to \cP_{k,n}}\right)\Attn^{(t)}_{{\pb}\to \cP_{k,n}}+\1{\{{\qb}\in\cU \}}z_m^2\Attn^{(t)}_{{\pb}\to \cP_{k,m}} \right)\Bigg)\Bigg].
\end{align*}
 \end{enumerate}
\end{lemma}
\begin{proof} First,
    \begin{align*}
         \beta^{(t)}_{{\pb}\to{\qb}}&=e_{{\pb}}^{\top}\left(-\frac{\partial \cL}{\partial Q}\right)e_{{\qb}}
  = \EE\Bigg[\1{\{{\pb}\in\mathcal{M}\}}\sum_{{\rb}\in\mathcal{U}}\score^{(t)}_{{\pb}\to{\rb}} J^{{\pb}}_{{\rb}} K^{{\pb},{\qb}}_{{\rb}} \Bigg] .
    \end{align*}
    Then we let 
    \begin{align*}
        \beta^{(t)}_{k, {\pb}\to{\qb}} \coloneqq  \EE\Bigg[\1{\{{\pb}\in\mathcal{M}, k_X=k\}}\sum_{{\rb}\in\mathcal{U}}\score^{(t)}_{{\pb}\to{\rb}} J^{{\pb}}_{{\rb}} K^{{\pb},{\qb}}_{{\rb}}\Bigg].
    \end{align*}
    In the following, we denote  $a_{k,\pb}=n$ and $a_{k,\qb}=m$ for simplicity. 
    \paragraph{Case 1: $m=n$.}
    If ${\qb}\in \mathcal{U}\cap\cP_{k,n}$:
    \begin{itemize}
        \item For ${\rb}={\qb}$
        \begin{align*}
            J^{{\pb}}_{{\rb}}&= z_nv_{k,n}^{\top}\Big(z_nv_{k,n}-\sum_{\wb\in\mathcal{U}\cap\cP_{k,n}} \score_{{\pb}\to \wb} z_nv_{k,n} \Big)
            =z_n^2\left(1-\Attn_{{\pb}\to\cP_{k,n}}\right) , \\
            K^{{\pb},{\qb}}_{{\rb}}&=(e_{{\qb}}-(\score_{{\pb}\to {\qb}} e_{{\qb}}+\sum_{\wb\not={\qb}}\score_{{\pb}\to \wb}  e_{\wb }))^{\top}e_{{\qb}} =1-\score_{{\pb}\to {\qb}}.
        \end{align*}
        \item  For ${\rb}\in\mathcal{U}\cap\cP_{k,n}$, and ${\rb}\not={\qb}$
        \begin{align*}
            J^{{\pb}}_{{\rb}}&= z_nv_{k,n}^{\top}\Big(z_nv_{k,n}-\sum_{\wb\in\mathcal{U}\cap\cP_{k,n}}\score_{{\pb}\to \wb} z_nv_{k,n} \Big) 
            =z_n^2\Big(1-\Attn_{{\pb}\to\cP_{k,n}}\Big), \\
            K^{{\pb},{\qb}}_{{\rb}}&=\Big(e_{{\rb}}-(\score_{{\pb}\to {\qb}} e_{{\qb}}+\sum_{\wb\not={\qb}}\score_{{\pb}\to \wb}  e_{\wb}) \Big)^{\top}e_{{\qb}}
             =-\score_{{\pb}\to {\qb}} .
        \end{align*}
        Thus 
        \begin{align*}
            \sum_{{\rb}\in\mathcal{U}\cap \cP_{k,n}}\score_{{\pb}\to{\rb}}J^{{\pb}}_{{\rb}}\cdot K^{{\pb},{\qb}}_{{\rb}}
      &=z_n^2\Big(1-\sum_{\wb\in\mathcal{U}\cap\cP_{k,n}} \score_{{\pb}\to \wb} \Big) 
      \cdot \Big(-\sum_{{\rb}\in\mathcal{U}\cap \cP_{k,n}}\score_{{\pb}\to{\rb}}\score_{{\pb}\to {\qb}}+ \score_{{\pb}\to {\qb}}\Big) \\
      &=z_n^2\left(1-\Attn_{{\pb}\to\cP_{k,n}}\right)^2\score^{(t)}_{{\pb}\to {\qb}}  .
        \end{align*}
        \item For ${\rb}\in\mathcal{U}\cap\cP_{k,a}$, $a\not=n$
       \begin{align*}
            J^{{\pb}}_{{\rb}}&= z_av_{k,a}^{\top} \Big(z_nv_{k,n}-\sum_{\wb\in\mathcal{U}\cap\cP_{k,a}} \score_{{\pb}\to \wb} z_av_{k,a} \Big) 
            =-z_a^2  \sum_{\wb\in\mathcal{U}\cap\cP_{k,a}} \score_{{\pb}\to \wb} ,\\
            K^{{\pb},{\qb}}_{{\rb}}&= \Big(e_{{\rb}}-(\score_{{\pb}\to {\qb}}e_{{\qb}}+\sum_{\wb\not={\qb}}\score_{{\pb}\to \wb} e_{\wb}) \Big)^{\top}e_{{\qb}}
            =-\score_{{\pb}\to {\qb}} .
        \end{align*}
       \end{itemize} 
       Thus 
       \begin{align*}
\sum_{{\rb}\in\mathcal{U}}\score_{{\pb}\to{\rb}} J^{{\pb}}_{{\rb}} K^{{\pb},{\qb}}_{{\rb}}
        &=\score_{{\pb}\to {\qb}}\cdot \Big(z_n^2\left(1-\Attn_{{\pb}\to \cP_{k,n}}\right)^2+ \sum_{a\not=n} z_a^2\left(\Attn_{{\pb}\to \cP_{k,a}}\right)^2 \Big) .
    \end{align*}
    If ${\qb}\in \mathcal{M}\cap\cP_{k,n}$:
    \begin{itemize}
        \item  For ${\rb}\in\mathcal{U}\cap\cP_{k,n}$, 
        \begin{align*}
            J^{{\pb}}_{{\rb}}&= z_nv_{k,n}^{\top} \Big(z_nv_{k,n}-\sum_{\wb\in\mathcal{U}\cap\cP_{k,n}}\score_{{\pb}\to \wb} z_nv_{k,n} \Big)
            =z_n^2\left(1-\Attn_{{\pb}\to\cP_{k,n}}\right), \\
            K^{{\pb},{\qb}}_{{\rb}}&=\Big(e_{{\rb}}-(\score_{{\pb}\to {\qb}} e_{{\qb}}+\sum_{\wb\not={\qb}}\score_{{\pb}\to \wb}  e_{\wb}) \Big)^{\top}e_{{\qb}} 
            =-\score_{{\pb}\to {\qb}} .
        \end{align*}
        \item For ${\rb}\in\mathcal{U}\cap\cP_{k,a}$, $a\not=n$
       \begin{align*}
            J^{{\pb}}_{{\rb}}&= z_av_{k,a}^{\top} \Big(z_nv_{k,n}-\sum_{\wb\in\mathcal{U}\cap\cP_{k,a}} \score_{{\pb}\to \wb} z_av_{k,a}  \Big)
            =-z_a^2  \sum_{\wb\in\mathcal{U}\cap\cP_{k,a}} \score_{{\pb}\to \wb}, \\
            K^{{\pb},{\qb}}_{{\rb}}&=\Big(e_{{\rb}}-(\score_{{\pb}\to {\qb}}e_{{\qb}}+\sum_{\wb\not={\qb}}\score_{{\pb}\to \wb} e_{\wb}) \Big)^{\top}e_{{\qb}} 
            =-\score_{{\pb}\to {\qb}} .
        \end{align*}
       \end{itemize} 
       Thus 
       \begin{align*}
        &\sum_{{\rb}\in\mathcal{U}}\score_{{\pb}\to{\rb}} J^{{\pb}}_{{\rb}} K^{{\pb},{\qb}}_{{\rb}}\\
        &=\score_{{\pb}\to {\qb}}\cdot \Big(z_n^2\left(1-\Attn_{{\pb}\to \cP_{k,n}}\right)^2-z_n^2\left(1-\Attn_{{\pb}\to \cP_{k,n}}\right)+ \sum_{a\not= n} z_a^2\left(\Attn_{{\pb}\to \cP_{k,a}}\right)^2\Big) . 
    \end{align*}
    Putting it together,
    \begin{align*}
        \beta^{(t)}_{k, {\pb}\to{\qb}}& =\EE\Bigg[\1{\{{\pb}\in\mathcal{M}, k_X=k\}}\score_{{\pb}\to {\qb}}\cdot \\
        & \qquad \Big(-z_n^2\left(1-\Attn_{{\pb}\to \cP_{k,n}}\right)\1{\{{\qb}\in \mathcal{M}\}}+z_n^2\left(1-\Attn_{{\pb}\to \cP_{k,n}}\right)^2+ \sum_{m\not=n} z_m^2\left(\Attn_{{\pb}\to \cP_{k,m}}\right)^2\Big)\Bigg] .
    \end{align*}
    \paragraph{Case 2: $m\not=n$.} Similarly,  if ${\qb}\in \mathcal{U}\cap\cP_{k,m}$:
    \begin{itemize}
        \item  For ${\rb}\in\mathcal{U}\cap\cP_{k,n}$,
        \begin{align*}
            J^{{\pb}}_{{\rb}}&= z_nv_{k,n}^{\top} \Big(z_nv_{k,n}-\sum_{\wb\in\mathcal{U}\cap\cP_{k,n}} \score_{{\pb}\to \wb} z_nv_{k,n} \Big) 
            =z_n^2(1- \Attn_{{\pb}\to \cP_{k,n}} ), \\
           K^{{\pb},{\qb}}_{{\rb}}&=\Big(e_{{\rb}}-\score_{{\pb}\to {\qb}} e_{{\qb}}-\sum_{\wb\not={\qb}}\score_{{\pb}\to \wb} e_{\wb} \Big)^{\top}e_{{\qb}}
            =-\score_{{\pb}\to {\qb}} .
        \end{align*}
        \item For ${\rb}={\qb}$
        \begin{align*}
            J^{{\pb}}_{{\rb}}&= z_mv_{k,m}^{\top}\Big(z_nv_{k,n}-\sum_{\wb\in\mathcal{U}\cap\cP_{k,m}} \score_{{\pb}\to \wb}z_mv_{k,m} \Big)
            =-z_m^2 \Attn_{{\pb}\to \cP_{k,m}} , \\
            K^{{\pb},{\qb}}_{{\rb}}&=\Big(e_{{\qb}}-\score_{{\pb}\to {\qb}} e_{{\qb}}-\sum_{\wb\not=\wb} \score_{{\pb}\to \wb} e_{\wb} \Big)^{\top}e_{{\qb}}
            =1- \score_{{\pb}\to {\qb}}.
        \end{align*}
        \item For ${\rb}\in\mathcal{U}\cap\cP_{k,a}$, $a\not=n$, and ${\rb}\not={\qb}$
       \begin{align*}
        J^{{\pb}}_{{\rb}}&= z_av_{k,a}^{\top}\Big(z_nv_{k,n}-\sum_{\wb\in\mathcal{U}\cap\cP_{k,a}} \score_{{\pb}\to \wb} z_av_{k,a} \Big)
        =-z_a^2  \Attn_{{\pb}\to \cP_{k,a}} , \\
        K^{{\pb},{\qb}}_{{\rb}}&=\Big(e_{{\rb}}-\score_{{\pb}\to {\qb}} e_{{\qb}}-\sum_{\wb\not={\qb}}\score_{{\pb}\to \wb} e_{\wb} \Big)^{\top}e_{{\qb}}
        =- \score_{{\pb}\to {\qb}}. 
        \end{align*}
       \end{itemize}  
       Thus 
       \begin{align*}
&\sum_{{\rb}\in\mathcal{U}}\score_{{\pb}\to{\rb}} J^{{\pb}}_{{\rb}} K^{{\pb},{\qb}}_{{\rb}}\\
        &=\score_{{\pb}\to {\qb}}\cdot
        \Big(-z_n^2\left(1-\Attn_{{\pb}\to \cP_{k,n}}\right)\Attn_{{\pb}\to \cP_{k,n}}-z_m^2\Attn_{{\pb}\to \cP_{k,m}}+ \sum_{a\not=n} z_a^2\left(\Attn_{{\pb}\to \cP_{k,a}}\right)^2 \Big).
    \end{align*}
    If ${\qb}\in \mathcal{M}\cap\cP_{k,m}$:
    \begin{itemize}
        \item  For ${\rb}\in\mathcal{U}\cap\cP_{k,n}$,
        \begin{align*}
            J^{{\pb}}_{{\rb}}&= z_nv_{k,n}^{\top}\Big(z_nv_{k,n}-\sum_{\wb\in\mathcal{U}\cap\cP_{k,n}} \score_{{\pb}\to \wb} z_nv_{k,n} \Big) 
            =z_n^2(1- \Attn_{{\pb}\to \cP_{k,n}} ) , \\
           K^{{\pb},{\qb}}_{{\rb}}&=\Big(e_{{\rb}}-\score_{{\pb}\to {\qb}} e_{{\qb}}-\sum_{\wb\not={\qb}}\score_{{\pb}\to \wb} e_{\wb} \Big)^{\top}e_{{\qb}}
            =-\score_{{\pb}\to {\qb}}  .
        \end{align*}
        \item For ${\rb}\in\mathcal{U}\cap\cP_{k,a}$, $a\not=n$
       \begin{align*}
        J^{{\pb}}_{{\rb}}&= z_av_{k,a}^{\top}\Big(z_nv_{k,n}-\sum_{\wb\in\mathcal{U}\cap\cP_{k,a}} \score_{{\pb}\to \wb} z_av_{k,a} \Big)
        =-z_a^2  \Attn_{{\pb}\to \cP_{k,a}}, \\
        K^{{\pb},{\qb}}_{{\rb}}&=\Big(e_{{\rb}}-\score_{{\pb}\to {\qb}} e_{{\qb}}-\sum_{\wb\not={\qb}}\score_{{\pb}\to \wb} e_{\wb} \Big)^{\top}e_{{\qb}}
        =- \score_{{\pb}\to {\qb}} . 
        \end{align*}
       \end{itemize}  
       Thus 
       \begin{align*}
        &\sum_{{\rb}\in\mathcal{U}}\score_{{\pb}\to{\rb}} J^{{\pb}}_{{\rb}} K^{{\pb},{\qb}}_{{\rb}}\\
        &=\score_{{\pb}\to {\qb}}\cdot \Big(-z_n^2\left(1-\Attn_{{\pb}\to \cP_{k,n}}\right)\Attn_{{\pb}\to \cP_{k,n}}+ \sum_{a\not=n} z_a^2\left(\Attn_{{\pb}\to \cP_{k,a}}\right)^2 \Big).
    \end{align*}
       Therefore
       \begin{align*}
       \beta^{(t)}_{k, {\pb}\to{\qb}}& =\EE\Bigg[\1{\{{\pb}\in\mathcal{M}, k_X=k\}}\score_{{\pb}\to {\qb}}\cdot \\
        & \Big(-z_n^2\left(1-\Attn_{{\pb}\to \cP_{k,n}}\right)\Attn_{{\pb}\to \cP_{k,n}}-\1{\{{\qb}\in\cU \}}z_m^2\Attn_{{\pb}\to \cP_{k,m}} + \sum_{a\not=n}z_a^2\left(\Attn_{{\pb}\to \cP_{k,a}}\right)^2 \Big) \Bigg].
    \end{align*}
\end{proof}

Based on the above gradient update for $\Upsilon^{(t)}_{\pb\to\qb}$, we further introduce the following auxiliary quantity, which will be useful in the later proof:
\begin{align}
    \Upsilon^{(t+1)}_{k, \pb\to\qb} \coloneqq  \Upsilon^{(t)}_{k, \pb\to\qb}+\eta \beta^{(t)}_{k, \pb\to\qb},  \quad \text{ with } \Upsilon^{(0)}_{k, \pb\to\qb}=0.
\end{align}
It is easy to verify that $\Upsilon^{(t)}_{\pb\to\qb}=\sum_{k\in [K]}\Upsilon^{(t)}_{k, \pb\to\qb}$.

\subsubsection{High-probability events}
We first introduce the following exponential bounds for the hypergeometric
distribution \text{Hyper} $(m, D, M)$. \text{Hyper} $(m, D, M)$ describes the probability of certain successes (random draws for which the object drawn has a specified feature) in $m$ draws, without replacement, from a finite population of size $M$ that contains exactly $D$  objects with that feature, wherein each draw is either a success or a failure.
\begin{proposition}[\cite{greene2017exponential}]  Suppose $S \sim$ \text{Hyper} $(m, D, M)$ with $1 \leq m, D \leq M$. Define $\mu_M \coloneqq  D / M$. Then for all $t>0$
$$
P\left(|S-m\mu_M|>t\right) \leq 2\exp \left(-\frac{t^2}{4m\mu_M+2t} \right).
$$
\end{proposition}
We then utilize this property to prove the high-probability set introduced in \Cref{sec-prep}. 
\begin{lemma}\label{app:lem:prob1}
    For $k\in[K]$ $n\in[N]$, define 
    \begin{align}
        \cE_{k,n}(\gamma,P) \coloneqq \{\mask: |\cP_{k,n}\cap \cU | =\Theta(C_n) \},
    \end{align}
    we have 
    \begin{align}
        \mathbb{P}(\mask \in \cE_{k,n})\geq 1-2\exp(-c_{n,1}C_n),
        \end{align}
        where $c_{n,0}>0$ is some constant. 
\end{lemma}
    
\begin{proof}
 Under the random masking strategy, given $k\in [K]$ and $n\in [N]$,  $Y_{k,n}=|\cU\cap\cP_{k,n}|$ follows the hypergeometric distribution, i.e. $Y_{k,n} \sim \text{Hyper}((1-\gamma)P, C_n, P)$.  Then by tail bounds, for $t>0$, we have:
     \begin{align*}
        \mathbb{P}[|Y_{k,n}-(1-\gamma)C_n|> t]&\leq 2\exp\left(-\frac{t^2}{4(1-\gamma)C_n+2t} \right)
     \end{align*}
     Letting $t=\Theta(C_n)$,   we have 
     \begin{align*}
        \mathbb{P}[Y_{k,n}= \Theta(C_n)]\geq 1-2e^{-c_{n,1}C_n}.
     \end{align*}

\end{proof}
 We further have the following fact, which will be useful for proving the property of loss objective in the next subsection. 

 \begin{lemma}\label{app:lem:prob2}
    For $k\in[K]$ and $n\in[N]$, we have 
    \begin{align}\label{eq:prob2}
        \mathbb{P}(|\cU\cap \cP_{k,n}|=0)\leq \exp(-c_{n,0}C_n),
        \end{align}
        where $c_{n,0}>0$ is some constant. 
\end{lemma}
\begin{proof}
By the form of probability density for $\text{Hyper}((1-\gamma)P, C_n, P)$, we have
        \begin{align*}
  \mathbb{P}(|\cU\cap \cP_{k,n}|=0)&=\frac{ {C_n \choose 0}{(P-C_n) \choose (1-\gamma)P}}{{P\choose (1-\gamma)P}} \leq \gamma^{C_n} =\exp(- c_{n,0}C_n)).
        \end{align*}        
\end{proof}

\subsubsection{Properties of loss functions}

Recall the training and regional reconstruction loss  we consider are given by:
\begin{align}
    \cL(Q)& \coloneqq  \frac{1}{2}\mathbb{E}\left[\sum_{{\pb}\in \cP}\ind\{{\pb}\in\cM\}\left\|[F(\mask(X); Q, E)]_{{\pb}}-X_{{\pb}} \right\|^2\right], \label{app-eq:loss}\\
     \cL_{\pb}(Q) &= \frac{1}{2}\mathbb{E}\left[\ind\{{\pb}\in\cM\}\left\|[F(\mask(X),E)]_{{\pb}}-X_{{\pb}} \right\|^2\right] . \label{app-eq:loss-n}
\end{align}

In this part, we will present several important lemmas for such a training objective. We first single out the following lemma, which connects the loss form with the attention score.

    \begin{lemma}[Loss Calculation]\label{app:lem:loss1}
    The population loss $ L(Q)$ can be decomposed into the following form:
        \begin{align*}
       \cL(Q)&=\sum_{\pb\in \cP}{\cL}_{\pb}(Q),      \text{ where }\\
   \cL_{\pb}(Q) &=\frac{1}{2}\sum_{k=1}^{K}\EE \Bigg[\1\{{\pb}\in\mathcal{M}, k_X=k\} \cdot \Big(z_{a_{k,\pb}}^2\left(1-\Attn^{(t)}_{{\pb}\to \cP_{k,a_{k,\pb}}}\right)^2+ \sum_{a\not= a_{k,\pb}} z_a^2 \left(\Attn^{(t)}_{{\pb}\to \cP_{k,a}}\right)^2   \Big) \Bigg].
 \end{align*}
\end{lemma}
\begin{proof}
\begin{align*}
   & \cL_{\pb}(Q) \\&= \frac{1}{2}\sum_{k=1}^{K}\mathbb{E}\left[\ind\{{\pb}\in\cM, k_X=k\}\left\|[F(\mask(X),E)]_{{\pb}}-X_{{\pb}} \right\|^2\right]\\
    &=\frac{1}{2}\sum_{k=1}^{K}\mathbb{E}\left[\ind\{{\pb}\in\cM,k_X=k\}\left\|\sum_{m\in[N]}\Attn_{{\pb}\to \cP_{k,m}} z_mv_{k,m}-z_{a_{k,\pb}}v_{k,a_{k,\pb}}\right\|^2\right]\\
    &\stackrel{(i)}{=}\frac{1}{2}\sum_{k=1}^{K}\mathbb{E}\left[\ind\{{\pb}\in\cM, k_X=k\}\left(z_{a_{k,\pb}}^2\left(1-\Attn_{{\pb}\to \cP_{k,a_{k,\pb}}}\right)^2+ \sum_{m\not=a_{k,\pb}} z_m^2 \left(\Attn_{{\pb}\to \cP_{k,m}}\right)^2 \right)\right],
\end{align*}
where $(i)$ follows since the features are orthogonal. 
\end{proof}
We then introduce some additional crucial notations for the loss objectives.
\begin{subequations}\label{eq:loss-aux}
        \begin{align}
       \cL^{*}_{\pb}  &=
\min_{Q\in\mathbb{R}^{d\times d}}\cL_{\pb}(Q),\label{app:eq:inf}
 \\
      \Loi_{\pb} &= \frac{1}{2}(\sigma_z^2+\frac{L^2}{N-1})\sum_{k\in [K]}\mathbb{P}\left( |\mathcal{U}\cap \cP_{k,z_{a_{k,\pb}}}|=0 \right) ,\label{app:eq:low}\\
     \tilde{\cL}_{\pb}(Q) & =\notag \sum_{k=1}^{K}\tilde{\cL}_{k, \pb}(Q),\quad \text{ where }\\
         \tilde{\cL}_{k, \pb}(Q)& = \frac{1}{2}\EE\Bigg[\1\{{\pb}\in\mathcal{M}, k_X=k, \mask\in\cE_{k,z_{a_{k,\pb}}}\}  \Big(z_{a_{k,\pb}}^2\left(1-\Attn^{(t)}_{{\pb}\to \cP_{k,a_{k,\pb}}}\right)^2+ \sum_{a\not= a_{k,\pb}} z_a^2 \left(\Attn^{(t)}_{{\pb}\to \cP_{k,a}}\right)^2   \Big)\Bigg].
 \end{align}
\end{subequations}
Here $\sigma_z^2=\mathbb{E}[Z_{n}(X)^2]$. $\cL_{\pb}^{\star}$ denotes the minimum value of the population loss in \eqref{app-eq:loss-n}, and $\Loi_{\pb}$ represents the  unavoidable errors for  $\pb\in\cP$, given that all the patches in $\cP_{k, a_{k,\pb}}$ are masked. We will show that  $\Loi_{\pb}$ serves as a lower bound for $\cL_{\pb}^{\star}$, and demonstrate that the network trained with GD will attain nearly zero error compared to $\Loi_{\pb}$. Our convergence will be established by the sub-optimality gap with respect to $\Loi_{\pb}$, which necessarily implies the convergence to $\cL_{\pb}^{\star}$. (It also implies $\cL_{\pb}^{\star}-\Loi_{\pb}$ is small.)

\begin{lemma}\label{app:lem:opt1}
    For ${\cL}_{\pb}^{\star}$ and $\Loi_{\pb}$ defined in \eqref{app:eq:inf} and \eqref{app:eq:low}, respectively,  we have $\Loi_{\pb}\leq {\cL}_{\pb}^{\star}$ and they are both on the order of $\Theta\Big(\exp\Big(-\big(c_{1}P^{\kappa_c}+\ind\big\{1\not\in \cup_{k\in[K]}\{a_{k,\pb}\}\big\} c_2 P^{\kappa_s}\big)\Big)\Big)$ where $c_{1}, c_2>0$ are some constants. 
\end{lemma}
\begin{proof}
    We first prove $\Loi_{\pb}\leq {\cL}_{\pb}^{\star}$:
    \begin{align*}
        {\cL}_{\pb}^{\star}=\min_{Q\in\mathbb{R}^{d\times d}}& \frac{1}{2}\sum_{k=1}^{K}\EE\Bigg[\1\{{\pb}\in\mathcal{M}, k_X=k\} \cdot\left(z_{a_{k,\pb}}^3 \Big(1-\Attn^{(t)}_{{\pb}\to \cP_{k,a_{k,\pb}}}\right)^2+ \sum_{a\not= a_{k,\pb}} z_a^2z_{a_{k,\pb}} \left(\Attn^{(t)}_{{\pb}\to \cP_{k,a}}\right)^2   \Big)\Bigg]\\
        \geq  \min_{Q\in\mathbb{R}^{d\times d}}& \frac{1}{2}\sum_{k=1}^{K}\EE\Bigg[\1\{{\pb}\in\mathcal{M}, k_X=k\}\ind\{ |\cU\cap \cP_{k, a_{\pb,k}}|=0\} \cdot\\
        & \quad \Big(z_{a_{k,\pb}}^3\left(1-\Attn^{(t)}_{{\pb}\to \cP_{k,a_{k,\pb}}}\right)^2+ \sum_{a\not= a_{k,\pb}} z_a^2z_{a_{k,\pb}} \left(\Attn^{(t)}_{{\pb}\to \cP_{k,a}}\right)^2   \Big)\Bigg].
    \end{align*}
    Notice that when all patches in $\cP_{k,a_{k,\pb}}$ are masked, $\Attn^{(t)}_{{\pb}\to \cP_{k,a_{k,\pb}}}=0$. Moreover, $$\sum\limits_{m\not=a_{k,\pb}}z_m^2\Attn^{(t)}_{{\pb}\to \cP_{k,m}}\geq \frac{L^2}{N-1}$$
   by Cauchy–Schwarz inequality. Thus
       \begin{align*}
        {\cL}_{\pb}^{\star}
        &\geq \frac{1}{2}\sum_{k=1}^{K}(\sigma_z^2+\frac{L^2}{N-1})\mathbb{P}\left( |\mathcal{U}\cap \cP_{k,a_{k,\pb}}|=0 \right) = \Loi_{\pb}.
    \end{align*}
Moreover,   $\Loi_{\pb}=\Theta\Big(\exp\Big(-\big(c_{1}P^{\kappa_c}+\ind\big\{1\not\in \cup_{k\in[K]}\{a_{k,\pb}\}\big\} c_2 P^{\kappa_s}\big)\Big)\Big)$ immediately comes from Lemma~\ref{app:lem:prob2}. 
   Furthermore, we only need to show ${\cL}_{\pb}^{\star}=O\Big(\exp\Big(-\big(c_{1}P^{\kappa_c}+\ind\big\{1\not\in \cup_{k\in[K]}\{a_{k,\pb}\}\big\} c_2 P^{\kappa_s}\big)\Big)\Big)$. This can be directly obtained by choosing $Q=\sigma I_{d}$ for some sufficiently large $\sigma$ and hence omitted here.
\end{proof}

\begin{lemma}\label{app:lem:opt2}
   Given $\pb\in\cP$,  for any  $Q$, we have 
    \begin{align*}
     \tilde{\cL}_{\pb}(Q)\leq {L}_{\pb}(Q)-\Loi_{\pb} \leq  \tilde{\cL}_{\pb}(Q)+O\Big(\exp\Big(-\big(c_{3}P^{\kappa_c}+\ind\big\{1\not\in \cup_{k\in[K]}\{a_{k,\pb}\}\big\} c_4 P^{\kappa_s}\big)\Big)\Big),
    \end{align*}
   where $c_{3},c_4>0$ are some constants.
\end{lemma}

\begin{proof} 
The lower bound is directly obtained by the definition and thus we only prove the upper bound. 
    \begin{align*}
      &  {L}_{\pb}(Q)- \tilde{\cL}_{\pb}(Q)\\
        &=\frac{1}{2}\sum_{k=1}^{K}\EE\left[\1\{{\pb}\in\mathcal{M}, k_X=k, \mask\in\cE^c_{k,z_{a_{k,\pb}}}\} \cdot  \Big(z_{a_{k,\pb}}^2\left(1-\Attn^{(t)}_{{\pb}\to \cP_{k,a_{k,\pb}}}\right)^2+ \sum_{a\not= a_{k,\pb}} z_a^2 \left(\Attn^{(t)}_{{\pb}\to \cP_{k,a}}\right)^2   \Big)\right]\\
        &\leq \sum_{k=1}^{K} U^2 \mathbb{P}(\mask\in \cE^c_{k,z_{a_{k,\pb}}})\\
        &\leq  O\Big(\exp\Big(-\big(c_{3}P^{\kappa_c}+\ind\big\{1\not\in \cup_{k\in[K]}\{a_{k,\pb}\}\big\} c_4 P^{\kappa_s}\big)\Big)\Big),
    \end{align*}  
where the last inequality follows from Lemma~\ref{app:lem:prob1}. 

\end{proof}

\subsection{Overall induction hypotheses and proof outline for MAE}
Our main proof utilizes the induction hypotheses. In this section, we introduce the main induction hypotheses for the positive and negative information gaps, which will later be proven to be valid throughout the entire learning process.
\subsubsection{Positive information gap}
We first state our induction hypothesis for the case that the information gap $\Delta$ is positive. 
\begin{hypothesis}
 For $t\leq T$, given $\pb,\qb\in\cP$, for $k\in[K]$, the following holds
     \begin{enumerate}[label={\alph*.}]
     \item $\Phi^{(t)}_{\bp\to v_{k, a_{k,\pb}}}$ is monotonically increasing, and $\Phi^{(t)}_{\bp\to v_{k, a_{k,\pb}}}\in[0, \tilde{O}(1)]$;\label{hypo-main-a}
     \item if $ a_{k,\pb}\not=1$, then $\Phi^{(t)}_{\bp\to v_{k, 1}}$ is monotonically decreasing and $\Phi^{(t)}_{\bp\to v_{k, 1}}\in [-\tilde{O}(1),0]$; \label{hypo-main-b}
     \item $|\Phi^{(t)}_{\bp\to v_{k, m}}|=
        \tilde{O}(\frac{1}{P^{1-\kappa_s}})$ for $m\notin \{1\}\cup \{a_{k,\pb}\}$; \label{hypo-main-c}
         \item for $\qb\not=\pb$,  $\Upsilon^{(t)}_{\pb\to\qb}=
        \tilde{O}(\frac{1}{P^{\kappa_s}})$; \label{hypo-main-d}
    \item $\Upsilon^{(t)}_{\pb\to\pb}=
        \tilde{O}(\frac{1}{P})$. \label{hypo-main-e}
     \end{enumerate}
     \label{hypo-main}
\end{hypothesis}


\subsubsection{Negative information gap}
Now we turn to the case that $\Delta\leq -\Omega(1)$. 
\begin{hypothesis}
For $t\leq T$, given $\pb,\qb\in\cP$, for $k\in[K]$, the following holds
     \begin{enumerate}[label={\alph*.}]
     \item $\Phi^{(t)}_{\bp\to v_{k, a_{k,\pb}}}$ is monotonically increasing, and $\Phi^{(t)}_{\bp\to v_{k, a_{k,\pb}}}\in[0, \tilde{O}(1)]$;\label{hypo-main-neg-a}
     \item if $ a_{k,\pb}\not=1$, then $\Phi^{(t)}_{\bp\to v_{k, 1}}$ is monotonically decreasing and $\Phi^{(t)}_{\bp\to v_{k, 1}}\in [-\tilde{O}(\frac{1}{P^{-\Delta}}),0]$;\label{hypo-main-neg-b}
     \item $|\Phi^{(t)}_{\bp\to v_{k, m}}|=
        \tilde{O}(\frac{1}{P^{1-\kappa_s}})$ for $m\notin \{1\}\cup \{a_{k,\pb}\}$;\label{hypo-main-neg-c}
         \item for $\qb\not=\pb$,  $\Upsilon^{(t)}_{\pb\to\qb}=
        \tilde{O}(\frac{1}{P^{\kappa_s}})$;\label{hypo-main-neg-d}
    \item $\Upsilon^{(t)}_{\pb\to\pb}=
        \tilde{O}(\frac{1}{P})$. \label{hypo-main-neg-e}
     \end{enumerate}
     \label{hypo-main-neg}
\end{hypothesis}

\subsubsection{Proof outline}
In both settings, we can classify the process through which transformers learn the feature attention correlation  $\Phi^{(t)}_{\bp\to v_{k, a_{k,\pb}}}$ into two distinct scenarios. These scenarios hinge on the spatial relation of the area \( \bp \) within the context of the \( k \)-th partition \( {\cD_k} \), specifically, whether \( \bp \) is located in the global area of the \( k \)-th cluster, i.e. whether \( a_{k,\bp}=1 \). The learning dynamics exhibit different behaviors of learning the local FP correlation in the local area with different $\Delta$, while the behaviors for features located in the global area are very similar, unaffected by the value of $\Delta$.  Therefore, through \Cref{sec:back:pos,sec:back:neg,sec:glo},  we delve into the learning phases and provide technical proofs for the local area with $\Delta\geq \Omega(1)$, local area with $\Delta\leq -\Omega(1)$ and the global area respectively. Finally, we will put this analysis together to prove that the \Cref{hypo-main}  (resp. \Cref{hypo-main-neg}) holds during the entire training process, thereby validating the main theorems in \Cref{sec:proof:main}. 


\subsection{Analysis for the local area with positive information gap}\label{sec:back:pos}
In this section, we focus on a specific patch $\pb\in\cP$ with the $k$-th cluster for $k\in[K]$, and present the analysis for the case that $X_{\pb}$ is located in the local area for the $k$-th cluster, i.e. \( a_{k,\bp}>1 \). We will analyze the case that $\Delta\geq \Omega(1)$. Throughout this section, we denote $a_{k,\bp}=n$ for simplicity. We will analyze the convergence of the training process via two phases of dynamics. 
At the beginning of each phase, we will establish an induction hypothesis, which we expect to remain valid throughout that phase. Subsequently, we will analyze the dynamics under such a hypothesis within the phase, aiming to provide proof of the hypothesis by the end of the phase.

\subsubsection{Phase I, stage 1}\label{app:sec:p1-1}
In this section, we shall discuss the initial stage of phase I. Firstly, we present the induction hypothesis in this stage.

We define the stage 1 of phase I as all iterations $t \leq T_{1}$, where
$$
T_{1} \triangleq \max \left\{t: \Phi_{\pb\to v_{k,n}}^{(t)} \geq -\frac{1}{U}\left(\frac{\Delta}{2}-0.01\right)\log(P)\right\}.
$$
We state the following induction hypotheses, which will hold throughout this period:
\begin{hypothesis}
    For each $0 \leq t \leq T_{1}$, $\qb\in\cP\setminus\{\pb\}$, the following holds:
     \begin{enumerate}[label={\alph*.}]
     \item $\Phi^{(t)}_{\bp\to v_{k, n}}$ is monotonically increasing, and $\Phi^{(t)}_{\bp\to v_{k, n}}\in[0, O\left(\frac{\left(\frac{\Delta}
     {2}-0.01\right)\log(P)}{P^{0.02}}\right)]$;
     \item $\Phi^{(t)}_{\bp\to v_{k, 1}}$ is monotonically decreasing and $\Phi^{(t)}_{\bp\to v_{k, 1}}\in [-\frac{1}{U}\left(\frac{\Delta}
     {2}-0.01\right)\log(P),0]$;
     \item $|\Phi^{(t)}_{\bp\to v_{k, m}}|=
        O\Big(\frac{\Phi^{(t)}_{\bp\to v_{k, n}}-\Phi^{(t)}_{\bp\to v_{k, 1}}}{P^{1-\kappa_s}}\Big)$ for $m\not=1,n$;
         \item  $\Upsilon^{(t)}_{k, \pb\to\qb}=
        O\Big(\frac{\Phi^{(t)}_{\bp\to v_{k, n}}}{C_n}\Big)$ for $a_{k,\qb}=n$, $|\Upsilon^{(t)}_{k, \pb\to\pb}|=
        O\Big(\frac{\Phi^{(t)}_{\bp\to v_{k, n}}-\Phi^{(t)}_{\bp\to v_{k, 1}}}{P}\Big)$;
        \item $|\Upsilon^{(t)}_{k, \pb\to\qb}|=O\Big(\frac{|\Phi^{(t)}_{\bp\to v_{k, 1}}|}{C_1}\Big)
       + O\Big(\frac{\Phi^{(t)}_{\bp\to v_{k, n}}-\Phi^{(t)}_{\bp\to v_{k, 1}}}{P}\Big)$ for $ a_{k,\qb}=1$; 
        \item $|\Upsilon^{(t)}_{k, \pb\to\qb}|=
        O\Big(\frac{\Phi^{(t)}_{\bp\to v_{k, n}}-\Phi^{(t)}_{\bp\to v_{k, 1}}}{P}\Big)$ for $ a_{k,\qb}\not=1,n$.
     \end{enumerate}
\label{hypothesis-p1.1} 
\end{hypothesis}

\paragraph{Properties of attention scores.}
We first introduce several properties of the attention scores if \Cref{hypo-main} and \Cref{hypothesis-p1.1} hold.

\begin{lemma}\label{lemma-tech1.1}
   For $n>1$,  if \Cref{hypo-main} and \Cref{hypothesis-p1.1} hold at iteration $t\leq T_{1}$,  then the following holds
    \begin{enumerate}
        \item $1-\Attn^{(t)}_{\pb\to \cP_{k,n}}-\Attn^{(t)}_{\pb\to \cP_{k,1}}\geq \Omega(1)$;
        \item If $\mask\in\cE_{k,n}$,
        $\Attn^{(t)}_{\pb\to \cP_{k,n}}=\Theta\Big(\frac{1}{P^{1-\kappa_s}}\Big)$
        ;
        \item Moreover, if $\mask\in\cE_{k,1}$, we have  $\Attn^{(t)}_{\pb\to \cP_{k,1}}=\Omega\Big(\frac{1}{P^{\frac{1-\kappa_s}{2}-0.01}}\Big)$;
        \item For ${\qb}\in \cM\cap(\cP_{k,n}\cup\cP_{k,1})$, $\score_{\pb\to{\qb}}^{(t)}=O\Big(\frac{1-\Attn^{(t)}_{\pb\to \cP_{k,1}}-\Attn^{(t)}_{\pb\to \cP_{k,n}}}{P}\Big)$.
    \end{enumerate}
\end{lemma}
\begin{lemma}\label{lemma-tech1.2}
    For $n>1$, if \Cref{hypo-main} and \Cref{hypothesis-p1.1} hold at iteration $t\leq T_{1}$, then for $m\not=n,1$
    , the following holds:
    \begin{enumerate}
        \item For any ${\qb}\in\cP_{k,m}$, $\score^{(t)}_{{\pb}\to{\qb}}\leq O\Big(\frac{1-\Attn^{(t)}_{\pb\to \cP_{k,1}}-\Attn_{\pb\to\cP_{k,n}}^{(t)} }{P} \Big)$.
        \item Moreover,  $\Attn^{(t)}_{\pb\to \cP_{k,m}}\leq O\Big(\frac{1-\Attn^{(t)}_{\pb\to \cP_{k,1}}-\Attn^{(t)}_{\pb\to \cP_{k,n}}}{N}\Big).
        $
    \end{enumerate}
\end{lemma}
The above properties can be easily verified through direct calculations by using the definition in \eqref{def:attn} and conditions in \Cref{hypothesis-p1.1}, which are omitted here for brevity. 

\paragraph{Bounding the gradient updates for FP correlations.} We have the following set of lemmas.

\begin{lemma}\label{lemma-cg1.1}
    For $n>1$, if \Cref{hypo-main} and \Cref{hypothesis-p1.1} hold at iteration $0\leq t\leq T_{1}$, then $\alpha_{\pb\to v_{k,n}}^{(t)}\geq 0$ and satisfies:
\begin{align*}
    \alpha^{(t)}_{\pb\to v_{k,n}}=\Theta\Big(\frac{ C_n}{P}\Big)=\Theta\Big(\frac{1}{P^{1-\kappa_s}}\Big).
\end{align*}
\end{lemma}
\begin{proof}
    By Lemma~\ref{lemma-pos-gd}, we have 
    \begin{align*}
      &  \alpha^{(t)}_{\pb\to v_{k,n}}\\
        &=\EE\left[\1\{k_X=k,  \pb\in\mathcal{M}\} \Attn^{(t)}_{\pb\to \cP_{k,n}}\cdot\left(z_n^3\left(1-\Attn^{(t)}_{\pb\to \cP_{k,n}}\right)^2+ \sum_{m\not=n} z_m^2z_n \left(\Attn^{(t)}_{\pb\to \cP_{k,m}}\right)^2   \right)\right]\\
        &= \EE\left[\1\{k_X=k, \cE_{k,n}\cap\pb\in\mathcal{M}\} \Attn^{(t)}_{\pb\to \cP_{k,n}}\cdot\left(z_n^3\left(1-\Attn^{(t)}_{\pb\to \cP_{k,n}}\right)^2+ \sum_{m\not=n} z_m^2z_n \left(\Attn^{(t)}_{\pb\to \cP_{k,m}}\right)^2   \right)\right]\\
        &\quad +\EE\left[\1\{k_X=k, \cE_{k,n}^c\cap\pb\in\mathcal{M}\} \Attn^{(t)}_{\pb\to \cP_{k,n}}\cdot\left(z_n^3\left(1-\Attn^{(t)}_{\pb\to \cP_{k,n}}\right)^2+ \sum_{m\not=n} z_m^2z_n \left(\Attn^{(t)}_{\pb\to \cP_{k,m}}\right)^2   \right)\right]\\
        &\leq \mathbb{P}(\mask\in\cE_{k,n})\\
        &\quad \cdot \EE\left[\1\{k_X=k, \pb\in\mathcal{M}\} \Attn^{(t)}_{\pb\to \cP_{k,n}}\cdot\left(z_n^3\left(1-\Attn^{(t)}_{\pb\to \cP_{k,n}}\right)^2+ \sum_{m\not=n} z_m^2z_n \left(\Attn^{(t)}_{\pb\to \cP_{k,m}}\right)^2   \right)\Big|\cE_{k,n}\right]\\
        &\quad + O(1)\cdot  \mathbb{P}(\mask\in\cE_{k,n}^c)\\
        &\leq O\Big(\frac{C_n}{P}  \Big) + O(\exp(-c_{n,1}C_n) ) \leq O\Big(\frac{C_n}{P}\Big) ,
       \end{align*}
       where the second inequality invokes Lemma~\ref{lemma-tech1.1} and Lemma~\ref{app:lem:prob1}, and the last inequality is due to $\exp(-c_{n,1}C_n)\ll \frac{C_n}{P}$. 
       Similarly, we can show that  $\alpha_{\pb\to v_{k,n}}^{(t)}\geq \Omega(\frac{ C_n}{P})$. 

\end{proof}
\begin{lemma}\label{lemma-cg1.2}
    For $n>1$, if \Cref{hypo-main} and \Cref{hypothesis-p1.1} hold at iteration $0\leq t\leq T_{1}$, then  $\alpha_{\pb\to v_{k,1}}^{(t)}<0$ and satisfies
    \begin{align*}
        |\alpha_{\pb\to v_{k,1}}^{(t)}|\geq \Omega\Big(\frac{1}{P^{2 (\frac{1-\kappa_s}{2}-0.01)}}\Big)=\Omega\Big(\frac{1}{P^{0.98-\kappa_s}}\Big). 
    \end{align*}
    \end{lemma}
\begin{proof}
    We first single out the following fact:
    \begin{align}
 &-z_1z_n^2\left(1-\Attn^{(t)}_{\pb\to \cP_{k,n}} \right)\Attn^{(t)}_{\pb\to \cP_{k,n}}- z_1^3 \left(1-\Attn^{(t)}_{\pb\to \cP_{k,1}} \right)\Attn^{(t)}_{\pb\to \cP_{k,1}} + \sum_{a\not=1,n} z_a^2z_1\left(\Attn^{(t)}_{\pb\to \cP_{k,a}}\right)^2\nonumber\\
 &\leq z_1\left( \max_{a\not=1,n}{z_a^2\Attn^{(t)}_{\pb\to \cP_{k,a}}}-z_n^2 \Attn^{(t)}_{\pb\to \cP_{k,n}}- z_1^2 \Attn^{(t)}_{\pb\to \cP_{k,1}} \right) (1-\Attn^{(t)}_{\pb\to \cP_{k,n}}-\Attn^{(t)}_{\pb\to \cP_{k,1}}) \nonumber\\
 &= - z_1(1-\Attn^{(t)}_{\pb\to \cP_{k,n}}-\Attn^{(t)}_{\pb\to \cP_{k,1}}) \left(z_n^2 \Attn^{(t)}_{\pb\to \cP_{k,n}}+z_1^2 \Attn^{(t)}_{\pb\to \cP_{k,1}} - \max_{a\not=1,n}{z_a^2\Attn^{(t)}_{\pb\to \cP_{k,a}}}\right). \label{eq-neg}
       \end{align}
       Therefore,  by Lemma~\ref{lemma-feature-gd}, we have 
       \begin{align*}
        \alpha_{\pb\to v_{k,1}}^{(t)}&\leq \EE\Bigg[\1{\{k_X=k, \cE_{k,1}\cap \pb\in\mathcal{M}\}} \Attn^{(t)}_{\pb\to \cP_{k,1}}\cdot\\
        &\quad \left(- z_1(1-\Attn^{(t)}_{\pb\to \cP_{k,n}}-\Attn^{(t)}_{\pb\to \cP_{k,1}}) \left(z_n^2 \Attn^{(t)}_{\pb\to \cP_{k,n}}+z_1^2 \Attn^{(t)}_{\pb\to \cP_{k,1}} - \max_{a\not=1,n}{z_a^2\Attn^{(t)}_{\pb\to \cP_{k,a}}}\right)\right)\Bigg]\\
        &\quad +\EE\left[\1{\{k_X=k, \cE_{k,1}^c\cap\pb\in\mathcal{M}\}} \Attn^{(t)}_{\pb\to \cP_{k,1}}\cdot \sum_{a\not=1,n} z_1^2z_a\left(\Attn^{(t)}_{\pb\to \cP_{k,a}}\right)^2\right]\\
        &\leq \mathbb{P}(\mask\in\cE_{k,1})\cdot \Big(-\big(\Omega(1)\cdot\Omega(\frac{1}{P^{2\times (\frac{1-\kappa_s}{2}-0.01)}})\big)\Big)+O(1)\cdot \mathbb{P}(\mask\in\cE_{k,1}^c) \\
        &\leq  -\Omega\Big(\frac{1}{P^{2\times (\frac{1-\kappa_s}{2}-0.01)}}\Big)= -\Omega\Big(\frac{1}{P^{0.98-\kappa_s}}\Big),
       \end{align*}
        where the second inequality invokes Lemma~\ref{lemma-tech1.1} and the last inequality comes from Lemma~\ref{app:lem:prob1}. 
\end{proof}
\begin{lemma}\label{lemma-cg1.3}
    At each iteration $t\leq T_{1}$, if \Cref{hypo-main} and \Cref{hypothesis-p1.1} hold, then for any $m>1$ with $m\not=n$, the following holds 
    \begin{align*}
        |\alpha_{\pb\to v_{k,m}}^{(t)}|\leq O\Big(\frac{\alpha_{\pb\to v_{k,n}}^{(t)}-\alpha_{\pb\to v_{k,1}}^{(t)}}{N}\Big)=O\Big( \frac{ \alpha^{(t)}_{\pb\to v_{k,n}}-\alpha_{\pb\to v_{k,1}}^{(t)}}{P^{1-\kappa_s}}\Big).
    \end{align*}
    \end{lemma}
    \begin{proof}
        By Lemma~\ref{lemma-feature-gd}, for $m\not=n$, we have
        \begin{align}
             \alpha^{(t)}_{\pb\to v_{k,m}}&\leq \EE\left[\1{\{k_X=k, \pb\in\mathcal{M}\}} \Attn^{(t)}_{\pb\to \cP_{k,m}}\cdot\left(\sum_{a\not=m,n} z_a^2z_m\left(\Attn^{(t)}_{\pb\to \cP_{k,a}}\right)^2   \right)\right], \label{eq-cg1.3-1}\\
                   -\alpha^{(t)}_{\pb\to v_{k,m}}&\leq \EE\Bigg[\1{\{k_X=k, \pb\in\mathcal{M}\}} \Attn^{(t)}_{\pb\to \cP_{k,m}}\cdot\left(z_mz_n^2\left(1-\Attn^{(t)}_{\pb\to \cP_{k,n}} \right)\Attn^{(t)}_{\pb\to \cP_{k,n}}\right.\notag\\
                   &\qquad \left.+ z_m^3 \left(1-\Attn^{(t)}_{\pb\to \cP_{k,m}} \right)\Attn^{(t)}_{\pb\to \cP_{k,m}}    \right)\Bigg].\label{eq-cg1.3-2}
            \end{align}
            For \eqref{eq-cg1.3-1},  we have
            \begin{align*}
                \alpha^{(t)}_{\pb\to v_{k,m}}
                &\leq  \EE\left[\1{\{k_X=k, \cE_{k,1}\cap \cE_{k,n}\cap \pb\in\mathcal{M}\}} \Attn^{(t)}_{\pb\to \cP_{k,m}}\cdot\left(\sum_{a\not=m,n} z_a^2z_m\left(\Attn^{(t)}_{\pb\to \cP_{k,a}}\right)^2   \right)\right]\\
                &\quad+  \EE\left[\1{\{k_X=k, (\cE_{k,1}\cap \cE_{k,n})^c\cap \pb\in\mathcal{M}\}} \Attn^{(t)}_{\pb\to \cP_{k,m}}\cdot\left(\sum_{a\not=m,n} z_a^2z_m\left(\Attn^{(t)}_{\pb\to \cP_{k,a}}\right)^2   \right)\right]\\
                &\leq  \EE\left[\1{\{k_X=k, \cE_{k,1}\cap \cE_{k,n}\cap \pb\in\mathcal{M}\}} O\Bigg(\frac{1-\Attn^{(t)}_{\pb\to \cP_{k,1}}-\Attn^{(t)}_{\pb\to \cP_{k,n}}}{N}\Bigg)\right.   \\   &\qquad \left.\cdot\left(z_1^2z_m\left(\Attn^{(t)}_{\pb\to \cP_{k,1}}\right)^2 +O\Big(\frac{1}{N}\Big)   \right)\right] + O(1)\cdot \mathbb{P}(\mask\in(\cE_{k,1}\cap\cE_{k, n})^c) \\
                &\leq O\Big(\frac{|\alpha_{\pb\to v_{k,1}}^{(t)}|}{N}\Big)+ O(1)\cdot \mathbb{P}(\mask\in(\cE_{k,1}\cap\cE_{k, n})^c) \\
                &\leq  O\Big(\frac{|\alpha_{\pb\to v_{k,1}}^{(t)}|}{P^{1-\kappa_s}}\Big),
            \end{align*}
            where the second inequality is due to Lemma~\ref{lemma-tech1.2}, the last inequality follows from Lemma~\ref{lemma-cg1.2} and Lemma~\ref{app:lem:prob1}. 
            
            On the other hand, for \eqref{eq-cg1.3-2}, we can use the similar argument by invoking Lemma~\ref{lemma-tech1.2} and Lemma~\ref{lemma-cg1.1}, and thus obtain  
            \begin{align*}
                -\alpha^{(t)}_{\pb\to v_{k,m}}
                &\leq  O\Big(\frac{\alpha_{\pb\to v_{k,n}}^{(t)}}{P^{1-\kappa_s}}\Big).
            \end{align*}
            Putting them together, we have 
            \begin{align*}
                |\alpha_{\pb\to v_{k,m}}^{(t)}|\leq O\Big( \frac{ \alpha^{(t)}_{\pb\to v_{k,n}}-\alpha_{\pb\to v_{k,1}}^{(t)}}{P^{1-\kappa_s}}\Big).
            \end{align*}
    \end{proof}

\paragraph{Bounding the gradient updates for positional correlations.}\label{sec:p1s1:pos} We have the following set of lemmas.

    \begin{lemma}\label{lemma-cg1.4}
        For $n>1$, if \Cref{hypo-main} and \Cref{hypothesis-p1.1} hold at iteration $0\leq t\leq T_{1}$, then for $\qb\in \cP\setminus\{\pb\}$ and $a_{k,\qb}=n$, we have  $\beta^{(t)}_{k, \pb\to\qb}\geq 0$ and satisfies:
        \begin{align*}
            \beta^{(t)}_{k, \pb\to\qb}= \Theta\Big(\frac{\alpha_{\pb\to v_{k,n}}^{(t)}}{C_n}\Big). 
        \end{align*}
        Furthermore, we have $|\beta^{(t)}_{k, \pb\to\pb}|= O\Big(\frac{\alpha_{\pb\to v_{k,n}}^{(t)}-\alpha_{\pb\to v_{k,1}}^{(t)}}{P}\Big)$. 
        \end{lemma}
\begin{proof}
    By Lemma~\ref{lemma-pos-gd}, for ${\qb}\in\cP_{k,n}$ with ${\qb}\not=\pb$, 
    we have
    \begin{align*}
        \beta^{(t)}_{k, \pb\to\qb}&=
        \underbrace{\EE\left[\1{\{k_X=k, \pb\in\mathcal{M}, {\qb}\in\cU\}}\score^{(t)}_{\pb\to {\qb}}\cdot\left(z_n^2\left(1-\Attn^{(t)}_{\pb\to \cP_{k,n}}\right)^2+\sum_{m\not=n} z_m^2\left(\Attn^{(t)}_{\pb\to \cP_{k,m}}\right)^2\right)\right]}_{H_1}\\
        &\quad +\underbrace{\EE\left[\1{\{k_X=k,  \pb\in\mathcal{M}, {\qb}\in\cM\}}\score^{(t)}_{\pb\to {\qb}}\cdot\left(-z_n^2\Attn^{(t)}_{\pb\to \cP_{k,n}}\left(1-\Attn^{(t)}_{\pb\to \cP_{k,n}}\right)\right)\right]}_{H_2}\\
        &\quad +\underbrace{\EE\left[\1{\{k_X=k, \pb\in\mathcal{M}, {\qb}\in\cM\}}\score^{(t)}_{\pb\to {\qb}}\cdot\left(\sum_{m\not=n} z_m^2\left(\Attn^{(t)}_{\pb\to \cP_{k,m}}\right)^2\right)\right]}_{H_3}. 
    \end{align*}
    Firstly, for $H_1$, notice that 
    \begin{align*}
        (C_n-1)H_1&=\EE\left[\1{\{k_X=k, \pb\in\mathcal{M}\}}\Attn^{(t)}_{\pb\to \cP_{k,n}}\cdot\left(z_n^2\left(1-\Attn^{(t)}_{\pb\to \cP_{k,n}}\right)^2+\sum_{m\not=n} z_m^2\left(\Attn^{(t)}_{\pb\to \cP_{k,m}}\right)^2\right)\right]\\
        &=\Theta(\alpha_{\pb\to v_{k,n}}^{(t)}).
    \end{align*}
    For $H_2$, since $\pb,{\qb}\in\mathcal{M}$, by Lemma~\ref{lemma-tech1.1}, we can upper bound $\score^{(t)}_{\pb\to {\qb}}$ by $O\Big(\frac{1}{P}\Big)$, thus 
    \begin{align*}
        -H_2\leq \EE\left[\1{\{k_X=k, \pb\in\mathcal{M}\}} O\Big(\frac{1}{P}\Big)\cdot\left(z_n^2\Attn^{(t)}_{\pb\to \cP_{k,n}}\left(1-\Attn^{(t)}_{\pb\to \cP_{k,n}}\right)\right)\right]\leq O\Big(\frac{\alpha_{\pb\to v_{k,n}}^{(t)}}{P}\Big). 
    \end{align*}
    Further notice that $H_3$ can be upper bounded by $O(H_1)$, putting it together, 
    we have
    \begin{align*}
        \beta^{(t)}_{k, \pb\to\qb}= \Theta\Big(\frac{\alpha_{\pb\to v_{k,n}}^{(t)}}{C_n}\Big).
    \end{align*}

    Turning to $\beta^{(t)}_{k,\pb\to\pb}$,  when ${\qb}=\pb$, it follows
    \begin{align*}
       \beta^{(t)}_{n}  &=
        \underbrace{\EE\left[\1{\{k_X=k, \pb\in\mathcal{M}\}}\score^{(t)}_{{\pb}\to {\pb}}\cdot\left(-z_n^2\Attn^{(t)}_{{\pb}\to \cP_{k,n}}\left(1-\Attn^{(t)}_{{\pb}\to \cP_{k,n}}\right)\right)\right]}_{J_2}\\
        &+\underbrace{\EE\left[\1{\{k_X=k, \pb\in\mathcal{M}\}}\score^{(t)}_{\pb\to \pb}\cdot\left(\sum_{m\not=n} z_m^2\left(\Attn^{(t)}_{\pb\to \cP_{k,m}}\right)^2\right)\right]}_{J_3}.
    \end{align*}
    We can bound $J_2$  in a similar way as $H_2$. Thus, we only focus on further bounding $J_3$:
    \begin{align*}
        J_3&\leq \EE\left[\1{\{k_X=k, \pb\in\mathcal{M}\}}O(\frac{1-\Attn^{(t)}_{\pb\to \cP_{k,1}}-\Attn^{(t)}_{\pb\to \cP_{k,n}}}{P})\cdot\left(\sum_{m\not=n} z_m^2\left(\Attn^{(t)}_{\pb\to \cP_{k,m}}\right)^2\right)\right]\\
        &\leq O\Bigg(\frac{|\alpha_{\pb\to v_{k,1}}^{(t)}|}{P}\Bigg). 
    \end{align*}
    where the first inequality holds by invoking Lemma~\ref{lemma-tech1.1} and the last inequality follows similar arguments as analysis for \eqref{eq-cg1.3-1}. 
\end{proof}
\begin{lemma}\label{lemma-cg1.5}
    For $n>1$, if \Cref{hypo-main} and \Cref{hypothesis-p1.1} hold at iteration $0\leq t\leq T_{1}$, then for $\qb\in \cP\setminus\{\pb\}$ and $a_{k,\qb}=1$, we have  $\beta^{(t)}_{k, \pb\to\qb}$ satisfies:
    \begin{align*}
        |\beta^{(t)}_{k, \pb\to\qb}|=  O\Bigg(\frac{|\alpha_{\pb\to v_{k,n}}^{(t)}-\alpha_{\pb\to v_{k,1}}^{(t)}|}{P}\Bigg)+O\Bigg(\frac{|\alpha_{\pb\to v_{k,1}}^{(t)}|}{C_1}\Bigg).
    \end{align*}
    \end{lemma}

    \begin{proof}
        By Lemma~\ref{lemma-pos-gd}, for ${\qb}\in\cP_{k,1}$, we have
        \begin{align}
            &\beta^{(t)}_{k, \pb\to\qb}=\notag\\
            &
            -\EE\Bigg[\1{\{k_X=k, \pb\in\mathcal{M}, {\qb}\in\cU\}}\score^{(t)}_{\pb\to {\qb}}\cdot \notag\\
            & \Big( z_1^2\Attn^{(t)}_{\pb\to \cP_{k,1}}(1-\Attn^{(t)}_{\pb\to \cP_{k,1}}) + z_n^2\left(1-\Attn^{(t)}_{\pb\to \cP_{k,n}}\right)\Attn^{(t)}_{\pb\to \cP_{k,n}}-\sum_{a\not=1,n} z_a^2\left(\Attn^{(t)}_{\pb\to \cP_{k,a}}\right)^2 \Big)\Bigg]
          \label{eq:G1}  \\
                        &\underbrace{-\EE\left[\1{\{k_X=k, \pb\in\mathcal{M},{\qb}\in\cM\}}\score^{(t)}_{\pb\to {\qb}}\cdot\left(z_n^2\left(1-\Attn^{(t)}_{\pb\to \cP_{k,n}}\right)\Attn^{(t)}_{\pb\to \cP_{k,n}}\right)\right]}_{G_2}\notag\\
            &\underbrace{+\EE\left[\1{\{k_X=k, \pb\in\mathcal{M},{\qb}\in\cM\}}\score^{(t)}_{\pb\to {\qb}}\cdot\left(\sum_{a\not=n} z_a^2\left(\Attn^{(t)}_{\pb\to \cP_{k,a}}\right)^2 \right)\right]}_{G_3}\notag.
        \end{align}
        For \eqref{eq:G1} denoted as $G_1$, following the direct calculations,  we have 
        \begin{align*}
            -(C_1-1)G_1&=\Theta(\alpha_{\pb\to v_{k,1}}^{(t)})
        \end{align*}
        We can further bound $G_2$ and $G_3$ in a similar way as $H_2$ and $H_3$ in Lemma~\ref{lemma-cg1.4} and thus obtain  
        \begin{align*}
            -G_2
            &\leq O\Big(\frac{\alpha_{\pb\to v_{k,n}}^{(t)}}{P}\Big), \\
             G_3&\leq  O\Bigg(\frac{|\alpha_{\pb\to v_{k,1}}^{(t)}|}{P}\Bigg),
        \end{align*}
        which completes the proof.
    \end{proof}
\begin{lemma}\label{lemma-cg1.6}
       For $n>1$, if \Cref{hypo-main} and \Cref{hypothesis-p1.1} hold at iteration $0\leq t\leq T_{1}$, then for $\qb\in \cP\setminus\{\pb\}$ and  $n\not=a_{k,\qb}$, $\beta^{(t)}_{k,\pb\to\qb}$ satisfies:
    \begin{align*}
        |\beta^{(t)}_{k,\pb\to\qb}|= O\Big(\frac{\alpha_{\pb\to v_{k,n}}^{(t)}-\alpha_{\pb\to v_{k,1}}^{(t)}}{P}\Big).
    \end{align*}
    \end{lemma}
\begin{proof}
 By Lemma~\ref{lemma-pos-gd}, for ${\qb}\in\cP_{k,m}$, we have
         \begin{align}
            &\beta^{(t)}_{k,\pb\to\qb}=\notag\\
            &
            -\EE\Bigg [\1{\{k_X=k, \pb\in\mathcal{M}, {\qb}\in\cU\}}\score^{(t)}_{\pb\to {\qb}}\cdot \notag\\
            & \Big( z_m^2\Attn^{(t)}_{\pb\to \cP_{k,m}}(1-\Attn^{(t)}_{\pb\to \cP_{k,m}}) + z_n^2\left(1-\Attn^{(t)}_{\pb\to \cP_{k,n}}\right)\Attn^{(t)}_{\pb\to \cP_{k,n}}-\sum_{a\not= n,m} z_a^2\left(\Attn^{(t)}_{\pb\to \cP_{k,a}}\right)^2 \Big)\Bigg]
          \label{eq:I1}  \\
                        &\underbrace{-\EE\left[\1{\{k_X=k, \pb\in\mathcal{M},{\qb}\in\cM\}}\score^{(t)}_{\pb\to {\qb}}\cdot\left(z_n^2\left(1-\Attn^{(t)}_{\pb\to \cP_{k,n}}\right)\Attn^{(t)}_{\pb\to \cP_{k,n}}\right)\right]}_{I_2}\notag\\
            &\underbrace{+\EE\left[\1{\{k_X=k, \pb\in\mathcal{M},{\qb}\in\cM\}}\score^{(t)}_{\pb\to {\qb}}\cdot\left(\sum_{a\not=n} z_a^2\left(\Attn^{(t)}_{\pb\to \cP_{k,a}}\right)^2 \right)\right]}_{I_3}\notag.
        \end{align}
   It follows that \eqref{eq:I1} can be upper bounded by $O\Big(\frac{|\alpha^{(t)}_{\pb\to v_{k,m}}|}{C_m}\Big)=O\Big(\frac{|\alpha_{\pb\to v_{k,1}}^{(t)}-\alpha_{\pb\to v_{k,1}}^{(t)}|}{NC_m}\Big)=O\Big(\frac{|\alpha_{\pb\to v_{k,1}}^{(t)}-\alpha_{\pb\to v_{k,1}}^{(t)}|}{P}\Big)$, where the first equality holds by invoking Lemma~\ref{lemma-cg1.3}.  $I_2$ and $I_3$ can be bounded similarly as $G_2$ and $G_3$, which is omitted here.
\end{proof}

\paragraph{At the end of Phase I, stage 1.} We have the following lemma.
\begin{lemma}
 For $n>1$, if \Cref{hypo-main} and \Cref{hypothesis-p1.1} hold for all $0\leq t\leq T_{1}=O\Big(\frac{\log(P)P^{0.98-\kappa_s}}{\eta}\Big)$, At iteration $t=T_{1}+1$, we have 
\begin{enumerate}[label={\alph*.}]
    \item $\Phi^{(T_{1}+1)}_{\pb\to v_{k,1}}\leq - \frac{1}{U}\left(\frac{\Delta}{2}-0.01\right)\log(P)$;\label{p1.e.a}
     \item $\Attn^{(T_{1}+1)}_{\pb\to \cP_{k,1}}=O\Big(\frac{1}{P^{(1-\kappa_c)+ \frac{L}{U}(\frac{\Delta}{2}-0.01)}}\Big)$.\label{p1.e.b}
 \end{enumerate}
\end{lemma}
\begin{proof}
 By comparing Lemma~\ref{lemma-cg1.1} and Lemma~\ref{lemma-cg1.2}, we have  $|\alpha_{\pb\to v_{k,1}}^{(t)}|\gg \alpha_{\pb\to v_{k,n}}^{(t)} $.   Then the  existence of $T_{1,k}=O\Big(\frac{\log(P)P^{0.98-\kappa_s}}{\eta}\Big)$ directly follows from Lemma~\ref{lemma-cg1.2}.
\end{proof}

\subsubsection{Phase I, stage 2}\label{app:sec-p1-2}
During stage 1, $\Phi_{\pb\to v_{k,1}}^{(t)}$ significantly decreases to decouple the FP correlations with the global feature, resulting in a decrease in $\Attn_{\pb\to \cP_{k,1}}^{(t)}$, while other $\Attn^{(t)}_{\pb\to \cP_{k,n}}$ with $m>1$ remain approximately at the order of $O\left(\frac{1}{P^{1-\kappa_s}}\right)$ ($\Theta\left(\frac{1}{P^{1-\kappa_s}}\right)$). 
By the end of phase I, $(\Attn_{\pb\to \cP_{k,1}}^{(t)})^2$ decreases to $O(\frac{1}{P^{1.96-2\kappa_s}})$, leading to a decrease in $|\alpha_{\pb\to v_{k,1}}^{(t)}|$ as it approaches towards $\alpha_{\pb\to v_{k,n}}^{(t)}$. At this point, stage 2 begins. Shortly after entering this phase, the prior dominant role of the decrease of $\Phi_{\pb\to v_{k,1}}^{(t)}$ in learning dynamics diminishes as $|\alpha_{\pb\to v_{k,1}}^{(t)}|$ reaches the same order of magnitude as $\alpha_{\pb\to v_{k,n}}^{(t)}$.

We define stage 2 of phase I as all iterations $T_{1}<t \leq \tilde{T}_{1}$, where
$$
\tilde{T}_{1} \triangleq \max \left\{t> T_{1}: \Phi_{\pb\to v_{k,n}}^{(t)}-\Phi_{\pb\to v_{k,1}}^{(t)} \leq \left(\frac{\Delta}{2L}+\frac{0.01}{L}+\frac{c_1^*(1-\kappa_s)}{U}\right)\log(P)\right\}
$$
for some  small constant $c^*_1>0$.

For computational convenience, we make the following assumptions for $\kappa_c$ and $\kappa_s$, which can be easily relaxed with the cost of additional calculations:
\begin{subequations}
    \begin{align}
        \frac{\Delta}{2}\Big(\frac{1}{L}-\frac{1}{U}\Big)+\frac{0.01}{L}+\frac{0.01}{U}&\leq \frac{c_0^*(1-\kappa_s)}{U}, \label{app-ass1}\\
        (1-\frac{c_1^*L}{U})(1-\kappa_s)&\leq (1-\kappa_c)+ \frac{U}{L}(\frac{\Delta}{2}+0.01). \label{app-ass2}
    \end{align}
\end{subequations}
Here $c_0^*$ is some small constant.
We state the following induction hypotheses, which will hold throughout this period.
\begin{hypothesis}
  For each $T_{1}<t \leq \tilde{T}_{1}$, $\qb\in\cP\setminus\{\pb\}$, the following holds:
     \begin{enumerate}[label={\alph*.}]
     \item $\Phi^{(t)}_{\pb\to v_{k,n}}$ is monotonically increasing, and $\Phi^{(t)}_{\pb\to v_{k,n}}\in[0, \frac{c^{*}_0+c^{*}_1}{U}\log(P)]$;
     \item $\Phi^{(t)}_{\pb\to v_{k,1}}$ is monotonically decreasing and $\Phi^{(t)}_{\pb\to v_{k,1}}\in [-\frac{1}{L}\left(\frac{\Delta}
     {2}+0.01\right)\log(P), -\frac{1}{U}\left(\frac{\Delta}
     {2}-0.01\right)\log(P)]$;
          \item $|\Phi^{(t)}_{\pb\to v_{k,m}}|=
        O\Big(\frac{\Phi^{(t)}_{\pb\to v_{k,n}}-\Phi^{(t)}_{\pb\to v_{k,1}}}{P^{1-\kappa_s}}\Big)$ for $m\not=1,n$;
         \item  $\Upsilon^{(t)}_{k, \pb\to\qb}=
        O\Big(\frac{\Phi^{(t)}_{\bp\to v_{k, n}}}{C_n}\Big)$ for $a_{k,\qb}=n$, $|\Upsilon^{(t)}_{k, \pb\to\pb}|=
        O\Big(\frac{\Phi^{(t)}_{\bp\to v_{k, n}}-\Phi^{(t)}_{\bp\to v_{k, 1}}}{P}\Big)$;
        \item $|\Upsilon^{(t)}_{k, \pb\to\qb}|=O\Big(\frac{|\Phi^{(t)}_{\bp\to v_{k, 1}}|}{C_1}\Big)
       + O\Big(\frac{\Phi^{(t)}_{\bp\to v_{k, n}}-\Phi^{(t)}_{\bp\to v_{k, 1}}}{P}\Big)$ for $ a_{k,\qb}=1$.; 
        \item $|\Upsilon^{(t)}_{k, \pb\to\qb}|=
        O\Big(\frac{\Phi^{(t)}_{\bp\to v_{k, n}}-\Phi^{(t)}_{\bp\to v_{k, 1}}}{P}\Big)$ for $ a_{k,\qb}\not=1,n$.
     \end{enumerate}
\label{hypothesis-p1.2} 
\end{hypothesis}

\paragraph{Property of attention scores.}
We first single out several properties of attention scores
that will be used for the proof of \Cref{hypothesis-p1.2}.

\begin{lemma}\label{lemma-tech-2.1}
    if \Cref{hypo-main} and \Cref{hypothesis-p1.2} hold at iteration $T_{1} +1\leq t \leq \tilde{T}_{1}$,  then the following holds
    \begin{enumerate}
        \item $1-\Attn^{(t)}_{\pb\to \cP_{k,n}}-\Attn^{(t)}_{\pb\to \cP_{k,1}}\geq \Omega(1)$;
        \item if $\mask\in\cE_{k,n}$,
        $\Attn^{(t)}_{\pb\to \cP_{k,n}}\in \Big[\Omega\big(\frac{1}{P^{1-\kappa_s}}\big), O\big(\frac{1}{P^{(1-c_1^*-c_0^*)(1-\kappa_s)}}\big)\Big]$
        ;
        \item Moreover, $\Attn^{(t)}_{\pb\to \cP_{k,1}}=O\Big(\frac{1}{P^{(1-\kappa_c)+ \frac{L}{U}(\frac{\Delta}{2}-0.01)}}\Big)$;  if $\mask\in\cE_{k,1}$, we have  $\Attn^{(t)}_{\pb\to \cP_{k,1}}=\Omega\Big(\frac{1}{P^{(1-\kappa_c)+ \frac{U}{L}(\frac{\Delta}{2}+0.01)}}\Big)$; 
        \item for ${\qb}\in \cM\cap(\cP_{k,n}\cup\cP_{k,1})$, $\score_{{\pb}\to{\qb}}^{(t)}=O\Big(\frac{1-\Attn^{(t)}_{\pb\to \cP_{k,n}}-\Attn^{(t)}_{\pb\to \cP_{k,1}}}{P}\Big)$. 
    \end{enumerate}
\end{lemma}
\begin{lemma}\label{lemma-tech2.2}
    if \Cref{hypo-main} and \Cref{hypothesis-p1.2} hold at iteration $T_{1} +1\leq t \leq \tilde{T}_{1}$, then for $m\not=n$, the following holds:
    \begin{enumerate}
        \item for any ${\qb}\in\cP_{k,m}$, $\score^{(t)}_{\pb\to{\qb}}\leq O\Big(\frac{1-\Attn^{(t)}_{\pb\to \cP_{k,n}}-\Attn^{(t)}_{\pb\to \cP_{k,1}}}{P}\Big)$;
        \item Moreover,  $\Attn^{(t)}_{\pb\to \cP_{k,m}}\leq O\Big(\frac{1-\Attn^{(t)}_{\pb\to \cP_{k,1}}-\Attn^{(t)}_{\pb\to \cP_{k,n}}}{N}\Big).
        $
    \end{enumerate}
\end{lemma}

\paragraph{Bounding the gradient updates of FP correlations.} We have the following set of lemma.

\begin{lemma}\label{lemma-cg2.1}
    For $n>1$,     if \Cref{hypo-main} and \Cref{hypothesis-p1.2} hold at iteration $T_{1} +1\leq t \leq \tilde{T}_{1}$, then $\alpha_{\pb\to v_{k,n}}^{(t)}\geq 0$ and satisfies:
\begin{align*}
\alpha_{\pb\to v_{k,n}}^{(t)}=\Omega\Big(\frac{1}{P^{1-\kappa_s}}\Big). 
\end{align*}
\end{lemma}
\begin{proof}
    By Lemma~\ref{lemma-pos-gd}, we have 
    \begin{align*}
        &\alpha_{\pb\to v_{k,n}}^{(t)}\\
        &=\EE\left[\1\{k_X=k, \pb\in\mathcal{M}\} \Attn^{(t)}_{\pb\to \cP_{k,n}}\cdot\left(z_n^3\left(1-\Attn^{(t)}_{\pb\to \cP_{k,n}}\right)^2+ \sum_{m\not=n} z_m^2z_n \left(\Attn^{(t)}_{\pb\to \cP_{k,m}}\right)^2   \right)\right]\\
        &= \EE\left[\1\{k_X=k, \cE_{k,n}\cap\pb\in\mathcal{M}\} \Attn^{(t)}_{\pb\to \cP_{k,n}}\cdot\left(z_n^3\left(1-\Attn^{(t)}_{\pb\to \cP_{k,n}}\right)^2+ \sum_{m\not=n} z_m^2z_n \left(\Attn^{(t)}_{\pb\to \cP_{k,m}}\right)^2   \right)\right]\\
        &\quad +\EE\left[\1\{k_X=k, \cE_{k,n}^c\cap\pb\in\mathcal{M}\} \Attn^{(t)}_{\pb\to \cP_{k,n}}\cdot\left(z_n^3\left(1-\Attn^{(t)}_{\pb\to \cP_{k,n}}\right)^2+ \sum_{m\not=n} z_m^2z_n \left(\Attn^{(t)}_{\pb\to \cP_{k,m}}\right)^2   \right)\right]\\
        &\gtrsim \mathbb{P}(\mask\in\cE_{k,n})\\
        &\quad \cdot \EE\left[\1\{k_X=k, \pb\in\mathcal{M}\} \Attn^{(t)}_{\pb\to \cP_{k,n}}\cdot\left(z_n^3\left(1-\Attn^{(t)}_{\pb\to \cP_{k,n}}\right)^2+ \sum_{m\not=n} z_m^2z_n \left(\Attn^{(t)}_{\pb\to \cP_{k,m}}\right)^2   \right)\Big|\cE_{k,n}\right]\\
        &\geq \Omega\Big(\frac{C_n}{P}\Big),
       \end{align*}
       where the last inequality invokes Lemma~\ref{lemma-tech-2.1}.  

\end{proof}
\begin{lemma}\label{lemma-cg2.2}
    For $n>1$,     if \Cref{hypo-main} and \Cref{hypothesis-p1.2} hold at iteration $T_{1} +1\leq t \leq \tilde{T}_{1}$, then  $\alpha_{\pb\to v_{k,1}}^{(t)}<0$ and satisfies
    \begin{align*}
        |\alpha_{\pb\to v_{k,1}}^{(t)}|\geq  \Omega\Big(\frac{1}{P^{2(1-\kappa_c)+ \frac{U}{L}(\Delta+0.02)}}\Big). 
    \end{align*}
    \end{lemma}
\begin{proof}
Following \eqref{eq-neg}, we have
    \begin{align*}
 &-z_1z_n^2\left(1-\Attn^{(t)}_{\pb\to \cP_{k,n}} \right)\Attn^{(t)}_{\pb\to \cP_{k,n}}- z_1^3 \left(1-\Attn^{(t)}_{\pb\to \cP_{k,1}} \right)\Attn^{(t)}_{\pb\to \cP_{k,1}} + \sum_{a\not=1,n} z_a^2z_1\left(\Attn^{(t)}_{\pb\to \cP_{k,a}}\right)^2\\
 &\leq  - z_1(1-\Attn^{(t)}_{\pb\to \cP_{k,n}}-\Attn^{(t)}_{\pb\to \cP_{k,1}}) \left(z_n^2 \Attn^{(t)}_{\pb\to \cP_{k,n}}+z_1^2 \Attn^{(t)}_{\pb\to \cP_{k,1}} - \max_{a\not=1,n}{z_a^2\Attn^{(t)}_{\pb\to \cP_{k,a}}}\right).
       \end{align*}
       Therefore,  by Lemma~\ref{lemma-feature-gd},  we obtain 
       \begin{align*}
        \alpha_{\pb\to v_{k,1}}^{(t)}& \leq \EE\Bigg[\1{\{k_X=k, \cE_{k,1}\cap \pb\in\mathcal{M}\}} \Attn^{(t)}_{\pb\to \cP_{k,1}}\cdot  \Big(- z_1(1-\Attn^{(t)}_{\pb\to \cP_{k,n}}-\Attn^{(t)}_{\pb\to \cP_{k,1}}) \\
       &  \cdot  (z_n^2 \Attn^{(t)}_{\pb\to \cP_{k,n}}+z_1^2 \Attn^{(t)}_{\pb\to \cP_{k,1}} - \max_{a\not=1,n}{z_a^2\Attn^{(t)}_{\pb\to \cP_{k,a}}} )\Big) \Bigg]\\
        &\quad +\EE\left[\1{\{k_X=k, \cE_{k,1}^c\cap\pb\in\mathcal{M}\}} \Attn^{(t)}_{\pb\to \cP_{k,1}}\cdot \sum_{a\not=1,n} z_1^2z_a\left(\Attn^{(t)}_{\pb\to \cP_{k,a}}\right)^2\right]\\
        &\leq \mathbb{P}(\mask\in\cE_{k,1})\cdot \Bigg(-\Omega(1)\cdot\Omega\Big(\frac{1}{P^{2(1-\kappa_c)+ \frac{2U}{L}(\frac{\Delta}{2}+0.01)}}\Big)\Bigg)+O(1)\cdot \mathbb{P}(\mask\in\cE_{k,1}^c) \\
        &\leq  -\Omega\Big(\frac{1}{P^{2(1-\kappa_c)+ \frac{U}{L}(\Delta+0.02)}}\Big),
       \end{align*}
       where the second inequality invokes Lemma~\ref{lemma-tech-2.1} and the last inequality comes from Lemma~\ref{app:lem:prob1}.  The upper bound can be obtained by using similar arguments and invoking the upper bound for $\Attn_{\pb\to \cP_{k,1}}^{(t)}$ in Lemma~\ref{lemma-tech-2.1}. 
\end{proof}
\begin{lemma}\label{lemma-cg2.3}
     For $n>1$,  if \Cref{hypo-main} and \Cref{hypothesis-p1.2} hold at iteration $T_{1} +1\leq t \leq \tilde{T}_{1}$, then for any $m>1$ with $m\not=n$, the following holds 
    \begin{align*}
        |\alpha_{\pb\to v_{k,m}}^{(t)}|\leq O\Big( \frac{ \alpha^{(t)}_{\pb\to v_{k,n}}-\alpha_{\pb\to v_{k,1}}^{(t)}}{P^{1-\kappa_s}}\Big). 
    \end{align*}
    \end{lemma}
    The proof is similar to Lemma~\ref{lemma-cg1.3}, and thus omitted here. 
\paragraph{Bounding the gradient updates of positional correlations.}
We then summarize the properties for gradient updates of positional correlations, which utilize the identical calculations as in Section~ \ref{sec:p1s1:pos}. 
\begin{lemma}
     For $n>1$,  if \Cref{hypo-main} and \Cref{hypothesis-p1.2} hold at iteration $T_{1} +1\leq t \leq \tilde{T}_{1}$, then
     \begin{enumerate}[label={\alph*.}]
         \item if $a_{k,\qb}=n$ and $\qb\not=\pb$,  $\beta^{(t)}_{k, \pb\to\qb}\geq 0$; $\beta^{(t)}_{k, \pb\to\qb}= \Theta\Big(\frac{\alpha_{\pb\to v_{k,n}}^{(t)}}{C_n}\Big)\text{ and }  |\beta^{(t)}_{n}|= O\Big(\frac{\alpha_{\pb\to v_{k,n}}^{(t)}-\alpha_{\pb\to v_{k,1}}^{(t)}}{P} \Big)$. 
         \item  if $a_{k,\qb}=1$, $|\beta^{(t)}_{k, \pb\to\qb}|=  O\Big(\frac{\alpha_{\pb\to v_{k,n}}^{(t)}-\alpha_{\pb\to v_{k,1}}^{(t)}}{P}\Big)+O\Big(\frac{|\alpha_{\pb\to v_{k,1}}^{(t)}|}{C_1}\Big)$.
         \item if $a_{k,\qb}=m$ and $m\not=1, n$, $|\beta^{(t)}_{k,\pb\to\qb}|= O\Big(\frac{\alpha_{\pb\to v_{k,n}}^{(t)}-\alpha_{\pb\to v_{k,1}}^{(t)}}{P}\Big)$. 
     \end{enumerate}
\end{lemma}

\paragraph{End of Phase I, stage 2.} We have the following lemma.

\begin{lemma}
\Cref{hypothesis-p1.2} holds for all iteration $T_{1} +1\leq t \leq \tilde{T}_{1}=T_{1}+O\Big(\frac{\log(P)P^{1-\kappa_s}}{\eta}\Big)$, and at iteration $t=\tilde{T}_{1}+1$, we have 
    \begin{enumerate}[label={\alph*}.]
    \item $\Phi_{\pb\to v_{k,n}}^{(\tilde{T}_{1}+1)}\geq \frac{c_1^{*}(1-\kappa_s)\log(P)}{U}$;
    \item $\Phi_{\pb\to v_{k,1}}^{(\tilde{T}_{1}+1)}\geq -(\frac{\Delta}{2L}+\frac{0.01}{L})\log(P)$.
    \end{enumerate}
\end{lemma}
\begin{proof}
    The  existence of $\tilde{T}_{1}=T_{1}+O\Big(\frac{\log(P)P^{1-\kappa_s}}{\eta}\Big)$ directly follows from Lemma~\ref{lemma-cg2.1} and Lemma~\ref{lemma-cg2.2}. Moreover, since $\alpha_{\pb\to v_{k,1}}^{(t)}<0$, then 
$$
\Phi_{\pb\to v_{k,n}}^{(\tilde{T}_{1}+1)}\leq \left(\frac{\Delta}{2L}+\frac{0.01}{L}+\frac{c_1^*(1-\kappa_s)}{U}\right)\log(P)-\frac{1}{U}(\frac{\Delta}{2}-0.01)\leq \frac{(c_0^*+c_1^*)(1-\kappa_s)}{U}\log(P),
$$
where the last inequality invokes \eqref{app-ass1}. 
    Now suppose $\Phi_{\pb\to v_{k,n}}^{(\tilde{T}_{1}+1)}<\frac{c_1^{*}(1-\kappa_s)\log(P)}{U}$,  then $\Phi_{\pb\to v_{k,1}}^{(\tilde{T}_{1}+1)}<-(\frac{\Delta}{2L}+\frac{0.01}{L})\log(P)$. Denote the first time that $\Phi_{\pb\to v_{k,1}}^{(t)}$ reaches $-(\frac{\Delta}{2L}+\frac{0.001}{L})\log(P)$ as $\tilde{T}$. Note that $\tilde{T}<\tilde{T}_{1}$ since $\alpha_{\pb\to v_{k,1}}^{(t)}$, the change  of $\Phi_{\pb\to v_{k,1}}^{(t)}$, satisfies $|\alpha_{\pb\to v_{k,1}}^{(t)}|\ll \log(P)$. Then for $t\geq \tilde{T}$, the following holds:
    \begin{enumerate}
        \item  $\Attn^{(t)}_{\pb\to \cP_{k,n}}\geq  \Omega\left(\frac{1}{P^{1-\kappa_s}}\right)$;
        \item  $\Attn^{(t)}_{\pb\to \cP_{k,1}}\leq O\Big(\frac{1}{P^{\frac{1-\kappa_s}{2}+0.001}}\Big)$.
        \end{enumerate}
        Therefore, 
               \begin{align*}
        |\alpha_{\pb\to v_{k,1}}^{(t)}|&\leq\EE\Bigg[\1{\{k_X=k, \cE_{k,1}\cap \pb\in\mathcal{M}\}} \Attn^{(t)}_{\pb\to \cP_{k,1}}\cdot \\
        &\quad z_1 \left(z_n^2 \Attn^{(t)}_{\pb\to \cP_{k,n}}(1-\Attn^{(t)}_{\pb\to \cP_{k,n}})+z_1^2 \Attn^{(t)}_{\pb\to \cP_{k,1}}(1-\Attn^{(t)}_{\pb\to \cP_{k,1}}) \right) \Bigg]\\
        &\quad +\EE\left[\1{\{k_X=k, \cE_{k,1}^c\cap\pb\in\mathcal{M}\}} \Attn^{(t)}_{\pb\to \cP_{k,1}}\cdot \sum_{a\not=1,n} z_1^2z_a\left(\Attn^{(t)}_{\pb\to \cP_{k,a}}\right)^2\right]\\
        &\leq O\Big(\frac{\alpha_{\pb\to v_{k,1}}^{(t)}}{P^{\frac{1-\kappa_s}{2}+0.001}} \Big)+\mathbb{P}(\mask\in\cE_{k,1})\cdot \bigg(O(1)\cdot O\Big(\frac{1}{P^{\frac{1-\kappa_s}{2}+0.001}}\Big)\bigg)+O(1)\cdot \mathbb{P}(\mask\in\cE_{k,1}^c) \\
        &\leq  O\Big(\frac{\alpha_{\pb\to v_{k,1}}^{(t)}}{P^{\frac{1-\kappa_s}{2}+0.001}}\Big)+O\Big(\frac{1}{P^{(1-\kappa_s)+0.002}}\Big).
       \end{align*}
       Lemma~\ref{lemma-cg2.1} still holds, and thus
        \begin{align*}
            |\alpha_{\pb\to v_{k,1}}^{(t)}|\leq O\Big(\frac{\alpha_{\pb\to v_{k,n}}^{(t)}}{P^{0.002}}\Big).
        \end{align*}
        Since $|\Phi_{\pb\to v_{k,1}}^{(\tilde{T}_{1}+1)}-\Phi_{\pb\to v_{k,1}}^{(\tilde{T})}|\geq \Omega\left(\log(P)\right)$, we have  $$\Phi_{\pb\to v_{k,n}}^{(\tilde{T}_{1}+1)}\geq  |\Phi_{\pb\to v_{k,1}}^{(\tilde{T}_{1}+1)}-\Phi_{\pb\to v_{k,1}}^{(\tilde{T})}|\cdot \Omega(P^{0.002})+\Phi_{\pb\to v_{k,n}}^{(\tilde{T})}\gg \Omega(P^{0.002}\log(P)),$$
        which contradicts the assumption that $\Phi_{\pb\to v_{k,n}}^{(\tilde{T}_{1}+1)}<\frac{c_1^{*}(1-\kappa_s)\log(P)}{U}$.
\end{proof}

\subsubsection{Phase II, stage 1}\label{sec:p2-s1}
For $n>1$, we define stage 1 of  phase II as all iterations $\tilde{T}_{1}+1\leq t \leq T_{2}$, where
$$
T_{2} \triangleq \max \left\{t: \Phi_{\pb\to v_{k,n}}^{(t)}\leq \frac{(1-\kappa_s)}{L}\log(P)\right\}.
$$

We state the following induction hypotheses, which will hold throughout this stage: 
     
\begin{hypothesis}
   For each $\tilde{T}_{1}+1\leq t \leq T_{2}$, $\qb\in\cP\setminus\{\pb\}$, the following holds:
     \begin{enumerate}[label={\alph*.}]
     \item $\Phi^{(t)}_{\pb\to v_{k,n}}$ is monotonically increasing, and $\Phi^{(t)}_{\pb\to v_{k,n}}\in\Big[\frac{c^{*}_1(1-\kappa_s)}{U}\log(P), \frac{(1-\kappa_s)}{L}\log(P)\Big]$;
     \item $\Phi^{(t)}_{\pb\to v_{k,1}}$ is monotonically decreasing and $$\Phi^{(t)}_{\pb\to v_{k,1}}\in \Big[-\frac{1}{L}\left(\frac{\Delta}
     {2}+0.01\right)\log(P)-o(1), -\frac{1}{U}\left(\frac{\Delta}
     {2}-0.01\right)\log(P)\Big];$$ 
          \item $|\Phi^{(t)}_{\pb\to v_{k,m}}|=
        O\Big(\frac{\Phi^{(t)}_{\pb\to v_{k,n}}-\Phi^{(t)}_{\pb\to v_{k,1}}}{P^{1-\kappa_s}}\Big)$ for $m\not=1,n$;
        \item  $\Upsilon^{(t)}_{k, \pb\to\qb}=
        O\Big(\frac{\Phi^{(t)}_{\bp\to v_{k, n}}}{C_n}\Big)$ for $a_{k,\qb}=n$, $|\Upsilon^{(t)}_{k, \pb\to\pb}|=
        O\Big(\frac{\Phi^{(t)}_{\bp\to v_{k, n}}-\Phi^{(t)}_{\bp\to v_{k, 1}}}{P}\Big)$;
        \item $|\Upsilon^{(t)}_{k, \pb\to\qb}|=O\Big(\frac{|\Phi^{(t)}_{\bp\to v_{k, 1}}|}{C_1}\Big)
       + O\Big(\frac{\Phi^{(t)}_{\bp\to v_{k, n}}-\Phi^{(t)}_{\bp\to v_{k, 1}}}{P}\Big)$ for $ a_{k,\qb}=1$.; 
        \item $|\Upsilon^{(t)}_{k, \pb\to\qb}|=
        O\Big(\frac{\Phi^{(t)}_{\bp\to v_{k, n}}-\Phi^{(t)}_{\bp\to v_{k, 1}}}{P}\Big)$ for $ a_{k,\qb}\not=1,n$.
     \end{enumerate}
\label{hypothesis-p1.3} 
\end{hypothesis}

\paragraph{Property of attention scores.}
We first single out several properties of attention scores
that will be used for the proof of \Cref{hypothesis-p1.3}.
\begin{lemma}\label{lemma-tech-3.1}
    if \Cref{hypo-main} and \Cref{hypothesis-p1.3} hold at iteration $\tilde{T}_{1}+1\leq t \leq T_{2}$,  then the following holds
    \begin{enumerate}
        \item if $\mask\in\cE_{k,n}$,
        $\Attn^{(t)}_{\pb\to \cP_{k,n}}\geq  \Omega\Big(\frac{1}{P^{(1-\frac{c_1^*L}{U})(1-\kappa_s)}}\Big)$
        . Moreover, if $\Attn^{(t)}_{\pb\to \cP_{k,n}}$ does not reach the constant level, $1-\Attn_{\pb\to\cP_{k,n}}^{(t)}=\Omega(1)$; otherwise, $1-\Attn_{\pb\to\cP_{k,n}}^{(t)}=\Omega\left(\frac{1}{P^{(\frac{U}{L}-1)(1-\kappa_s)}}\right)$.
        \item  $\Attn^{(t)}_{\pb\to \cP_{k,1}}=O\Big(\frac{1-\Attn^{(t)}_{\pb\to \cP_{k,n}}}{P^{(1-\kappa_c)+ \frac{L}{U}(\frac{\Delta}{2}-0.01)}}\Big)$;  if $\mask\in\cE_{k,1}$, we have  $\Attn^{(t)}_{\pb\to \cP_{k,1}}=\Omega\Big(\frac{1}{P^{(1-\kappa_c)+ \frac{U}{L}(\frac{\Delta}{2}+0.01)}}\Big)$; 
        \item for ${\qb}\in \cM\cap(\cP_{k,n}\cup\cP_{k,1})$, $\score_{\pb\to{\qb}}^{(t)}=O\Big(\frac{1-\Attn_{\pb\to\cP_{k,n}}^{(t)}}{P}\Big)$
    \end{enumerate}
\end{lemma}
\begin{lemma}\label{lemma-tech3.2}
     if \Cref{hypo-main} and \Cref{hypothesis-p1.3} hold at iteration $\tilde{T}_{1}+1\leq t \leq T_{2}$, then for $m\not=n$, the following holds:
    \begin{enumerate}
        \item for any ${\qb}\in\cP_{k,m}$, $\score^{(t)}_{\pb\to{\qb}}\leq O\Big(\frac{1-\Attn_{\pb\to\cP_{k,n}}^{(t)}}{P}\Big)$.
        \item Moreover,  $\Attn^{(t)}_{\pb\to \cP_{k,n}}\leq O\Big(\frac{1-\Attn^{(t)}_{\pb\to \cP_{k,1}}-\Attn^{(t)}_{\pb\to \cP_{k,n}}}{N}\Big).
        $
    \end{enumerate}
\end{lemma}

\paragraph{Bounding the gradient updates of FP correlations.} We have the following set of lemmas.

\begin{lemma}\label{lemma-cg3.1}
 if \Cref{hypo-main} and \Cref{hypothesis-p1.3} hold at iteration $\tilde{T}_{1}+1\leq t \leq T_{2}$, then $\alpha_{\pb\to v_{k,n}}^{(t)}\geq 0$ and satisfies:
\begin{align*}
    \alpha^{(t)}_{\pb\to v_{k,n}}\geq \min\bigg\{\Omega\left(\frac{1}{P^{(1-\frac{c_1^*L}{U})(1-\kappa_s)}} \right), \Omega\left(\frac{1}{P^{2(\frac{U}{L}-1)(1-\kappa_s)}}\right)\bigg\}. 
\end{align*}
\end{lemma}
\begin{proof}
    By Lemma~\ref{lemma-pos-gd}, we have 
    \begin{align*}
        &\alpha_{\pb\to v_{k,n}}^{(t)}\\
        &=\EE\left[\1\{k_X=k,  \pb\in\mathcal{M}\} \Attn^{(t)}_{\pb\to \cP_{k,n}}\cdot\left(z_n^3\left(1-\Attn^{(t)}_{\pb\to \cP_{k,n}}\right)^2+ \sum_{m\not=n} z_m^2z_n \left(\Attn^{(t)}_{\pb\to \cP_{k,m}}\right)^2   \right)\right]\\
        &= \EE\left[\1\{k_X=k, \cE_{k,n}\cap\pb\in\mathcal{M}\} \Attn^{(t)}_{\pb\to \cP_{k,n}}\cdot\left(z_n^3\left(1-\Attn^{(t)}_{\pb\to \cP_{k,n}}\right)^2+ \sum_{m\not=n} z_m^2z_n \left(\Attn^{(t)}_{\pb\to \cP_{k,m}}\right)^2   \right)\right]\\
        &\quad +\EE\left[\1\{k_X=k, \cE_{k,n}^c\cap\pb\in\mathcal{M}\} \Attn^{(t)}_{\pb\to \cP_{k,n}}\cdot\left(z_n^3\left(1-\Attn^{(t)}_{\pb\to \cP_{k,n}}\right)^2+ \sum_{m\not=n} z_m^2z_n \left(\Attn^{(t)}_{\pb\to \cP_{k,m}}\right)^2   \right)\right]\\
        & \gtrsim \mathbb{P}(\mask\in\cE_{k,n})\cdot\\
        &~~\EE\left[\1\{k_X=k, \pb\in\mathcal{M}\} \Attn^{(t)}_{\pb\to \cP_{k,n}}\cdot\left(z_n^3\left(1-\Attn^{(t)}_{\pb\to \cP_{k,n}}\right)^2+ \sum_{m\not=n} z_m^2z_n \left(\Attn^{(t)}_{\pb\to \cP_{k,m}}\right)^2   \right)\Big|\cE_{k,n}\right]\\
        &\quad + O(1)\cdot  \mathbb{P}(\mask\in\cE_{k,n}^c)\\
        &\gtrsim  \min\bigg\{\Omega\bigg(\frac{1}{P^{(1-\frac{c_1^*L}{U})(1-\kappa_s)}}\bigg), \Omega\left(\frac{1}{P^{2(\frac{U}{L}-1)(1-\kappa_s)}}\right)\bigg\},
       \end{align*}
       where the last inequality invokes Lemma~\ref{lemma-tech-3.1} by observing that for $\mask\in\cE_{k, n}$, 
       \begin{align*}
         \Attn_{\pb\to\cP_{k,n}}^{(t)}  (1-\Attn_{\pb\to\cP_{k,n}}^{(t)})^2 & \geq
         \min\bigg\{\Omega\left(\frac{1}{P^{(1-\frac{c_1^*L}{U})(1-\kappa_s)}}\right)\cdot\Omega(1), \Omega(1)\cdot \Omega\left(\frac{1}{P^{2\times (\frac{U}{L}-1)(1-\kappa_s)}}\right) \bigg\}. 
       \end{align*}
\end{proof}
\begin{lemma}\label{lemma-cg3.2}
    For $n>1$, if \Cref{hypo-main} and \Cref{hypothesis-p1.3} hold at iteration $\tilde{T}_{1}+1\leq t \leq T_{2}$, then  $\alpha_{\pb\to v_{k,1}}^{(t)}<0$ and satisfies
    \begin{align*}
             |\alpha_{\pb\to v_{k,m}}^{(t)}|&\geq  \min\bigg\{\Omega\left(\frac{1}{P^{(1-\frac{c_1^*L}{U})(1-\kappa_s)}}\right), \Omega\left(\frac{1}{P^{(\frac{U}{L}-1)(1-\kappa_s)}}\right)\bigg\}\cdot \Omega\Big(\frac{1}{P^{(1-\kappa_c)+ \frac{L}{U}(\frac{\Delta}{2}-0.01)}}\Big), \\
        |\alpha_{\pb\to v_{k,m}}^{(t)}|&\leq \max\Big\{O\Big(\frac{\alpha_{\pb\to v_{k,n}}^{(t)}}{P^{(1-\kappa_c)+ \frac{L}{U}(\Delta/2-0.01)}}\Big),O\Big(\frac{\alpha_{\pb\to v_{k,n}}^{(t)}}{P^{2(1-\kappa_c)+ \frac{L}{U}(\Delta-0.02)- (1-\frac{c_1^*L}{U})(1-\kappa_s)}}\Big)\Big\}. 
    \end{align*}
    \end{lemma}
\begin{proof}
  Following \eqref{eq-neg}, we have
    \begin{align*}
 &-z_1z_n^2\left(1-\Attn^{(t)}_{\pb\to \cP_{k,n}} \right)\Attn^{(t)}_{\pb\to \cP_{k,n}}- z_1^3 \left(1-\Attn^{(t)}_{\pb\to \cP_{k,1}} \right)\Attn^{(t)}_{\pb\to \cP_{k,1}} + \sum_{a\not=1,n} z_a^2z_1\left(\Attn^{(t)}_{\pb\to \cP_{k,a}}\right)^2\\
 &\leq  - z_1(1-\Attn^{(t)}_{\pb\to \cP_{k,n}}-\Attn^{(t)}_{\pb\to \cP_{k,1}}) \left(z_n^2 \Attn^{(t)}_{\pb\to \cP_{k,n}}+z_1^2 \Attn^{(t)}_{\pb\to \cP_{k,1}} - \max_{a\not=1,n}{z_a^2\Attn^{(t)}_{\pb\to \cP_{k,a}}}\right).
       \end{align*}
       Therefore,  by Lemma~\ref{lemma-feature-gd},  we obtain 
       \begin{align*}
        \alpha_{\pb\to v_{k,1}}^{(t)}& \leq \EE\Bigg[\1{\{k_X=k, \cE_{k,1}\cap \pb\in\mathcal{M}\}} \Attn^{(t)}_{\pb\to \cP_{k,1}}\cdot \\
        &\quad \left(- z_1(1-\Attn^{(t)}_{\pb\to \cP_{k,n}}-\Attn^{(t)}_{\pb\to \cP_{k,1}}) \left(z_n^2 \Attn^{(t)}_{\pb\to \cP_{k,n}}+z_1^2 \Attn^{(t)}_{\pb\to \cP_{k,1}} - \max_{a\not=1,n}{z_a^2\Attn^{(t)}_{\pb\to \cP_{k,a}}}\right)\right) \Bigg]\\
        &\quad +\EE\left[\1{\{k_X=k, \cE_{k,1}^c\cap\pb\in\mathcal{M}\}} \Attn^{(t)}_{\pb\to \cP_{k,1}}\cdot \sum_{a\not=1,n} z_1^2z_a\left(\Attn^{(t)}_{\pb\to \cP_{k,a}}\right)^2\right]\\
        &\leq  - \min\bigg\{\Omega\left(\frac{1}{P^{(1-\frac{c_1^*L}{U})(1-\kappa_s)}}\right), \Omega\left(\frac{1}{P^{(\frac{U}{L}-1)(1-\kappa_s)}}\right)\bigg\}\cdot \Omega\Big(\frac{1}{P^{(1-\kappa_c)+ \frac{L}{U}(\frac{\Delta}{2}-0.01)}}\Big),
       \end{align*}
        where the second inequality invokes Lemma~\ref{lemma-tech-3.1} and \eqref{app-ass2}.
       Moreover, 
       \begin{align*}
        |\alpha_{\pb\to v_{k,1}}^{(t)}| & \lesssim \EE \Bigg[\1{\{k_X=k,  \cE_{k,1}\cap\cE_{k,n}\cap \pb\in\mathcal{M}\}} \Attn^{(t)}_{\pb\to \cP_{k,1}}\cdot \\
        &\quad \left(z_1z_n^2\left(1-\Attn^{(t)}_{\pb\to \cP_{k,n}} \right)\Attn^{(t)}_{\pb\to \cP_{k,n}}+z_1^3 \left(1-\Attn^{(t)}_{\pb\to \cP_{k,1}} \right)\Attn^{(t)}_{\pb\to \cP_{k,1}}\right) \Bigg]\\
        &=\EE\left[\1{\{k_X=k, \cE_{k,1}\cap\cE_{k,n}\cap \pb\in\mathcal{M}\}} z_1z_n^2\Attn^{(t)}_{\pb\to \cP_{k,1}}\cdot \left(1-\Attn^{(t)}_{\pb\to \cP_{k,n}} \right)\Attn^{(t)}_{\pb\to \cP_{k,n}}\right]\\
        &\quad +\EE\left[\1{\{k_X=k, \cE_{k,1}\cap\cE_{k,n}\cap \pb\in\mathcal{M}\}} z_1^3(\Attn^{(t)}_{\pb\to \cP_{k,1}})^2\cdot \left(1-\Attn^{(t)}_{\pb\to \cP_{k,1}} \right)\right]\\
        &\leq  \max\Big\{O\Bigg(\frac{\alpha_{\pb\to v_{k,n}}^{(t)}}{P^{(1-\kappa_c)+ \frac{L}{U}(\frac{\Delta}{2}-0.01)}}\Big),O\Bigg(\frac{\alpha_{\pb\to v_{k,n}}^{(t)}}{P^{2(1-\kappa_c)+ \frac{2L}{U}(\frac{\Delta}{2}-0.01)- (1-\frac{c_1^*L}{U})(1-\kappa_s)}}\Bigg)\Bigg\},
       \end{align*}
           where the second inequality invokes Lemma~\ref{lemma-tech-3.1}. 
\end{proof}
\begin{lemma}\label{lemma-cg3.3}
     For $n>1$,  if \Cref{hypo-main} and \Cref{hypothesis-p1.3} hold at iteration $\tilde{T}_{1}+1\leq t \leq T_{2} $ for any $m>1$ with $m\not=n$, the following holds 
    \begin{align*}
        |\alpha_{\pb\to v_{k,m}}^{(t)}|\leq O\Big( \frac{ \alpha^{(t)}_{\pb\to v_{k,n}}-\alpha_{\pb\to v_{k,1}}^{(t)}}{P^{1-\kappa_s}}\Big). 
    \end{align*}
    \end{lemma}
    The proof is similar to Lemma~\ref{lemma-cg1.3}, and thus omitted here. 

\paragraph{Bounding the gradient updates of positional correlations}
We then summarize the properties for gradient updates of positional correlations, which utilizes the identical calculations as in Section~\ref{sec:p1s1:pos}. 
\begin{lemma}
     For $n>1$,  if \Cref{hypo-main} and \Cref{hypothesis-p1.3} hold at iteration $\tilde{T}_{1} +1\leq t \leq {T}_{2}$, then
   \begin{enumerate}[label={\alph*.}]
         \item if $a_{k,\qb}=n$ and $\qb\not=\pb$,  $\beta^{(t)}_{k, \pb\to\qb}\geq 0$; $\beta^{(t)}_{k, \pb\to\qb}= \Theta(\frac{\alpha_{\pb\to v_{k,n}}^{(t)}}{C_n})\text{ and }  |\beta^{(t)}_{n}|= O\Big(\frac{\alpha_{\pb\to v_{k,n}}^{(t)}-\alpha_{\pb\to v_{k,1}}^{(t)}}{P}\Big)$. 
         \item  if $a_{k,\qb}=1$, $|\beta^{(t)}_{k, \pb\to\qb}|=  O\Big(\frac{\alpha_{\pb\to v_{k,n}}^{(t)}-\alpha_{\pb\to v_{k,1}}^{(t)}}{P}\Big)+O\Big(\frac{|\alpha_{\pb\to v_{k,1}}^{(t)}|}{C_1}\Big)$.
         \item if $a_{k,\qb}=m$ and $m\not=1, n$, $|\beta^{(t)}_{k,\pb\to\qb}|= O\Big(\frac{\alpha_{\pb\to v_{k,n}}^{(t)}-\alpha_{\pb\to v_{k,1}}^{(t)}}{P}\Big)$. 
     \end{enumerate}
\end{lemma}

\paragraph{End of Phase II, stage 1.} We have the following lemma.
\begin{lemma}\label{lemma-cg3.4}
    \Cref{hypothesis-p1.3} holds for all  $\tilde{T}_{1}+1\leq t \leq T_{2}$, and at iteration $t=T_{2}+1$, we have 
    \begin{enumerate}[label={\alph*}.]
    \item $\Phi_{\pb\to v_{k,n}}^{(t)}> \frac{(1-\kappa_s)}{L}\log(P)$;
    \item $\Attn_{\pb\to\cP_{k,n}}^{(t)}=\Omega(1)$ if $\mask\in\cE_{k,n}$. 
    \end{enumerate}
\end{lemma}
\begin{proof}
 By comparing Lemma~\ref{lemma-cg3.1} and Lemma~\ref{lemma-cg3.2}-\ref{lemma-cg3.4}, we have  $\alpha_{\pb\to v_{k,n}}^{(t)}\gg |\alpha^{(t)}_{\pb\to v_{k,m}}|, |\beta_{k,\pb\to\qb}^{(t)}|$.   Then the  existence of $T_{2}=\tilde{T}_{1}+O\Big(\frac{\log(P)P^{\Lambda}}{\eta}\Big)$ directly follows from Lemma~\ref{lemma-cg3.1}, where 
 $$
 \Lambda= \max\Big\{(1-\frac{c_1^*L}{U}), 2(\frac{U}{L}-1)\Big\}\cdot (1-\kappa_s). 
 $$
 The second statement can be directly verified by noticing that $\Phi_{\pb\to v_{k,n}}^{(t)}> \frac{(1-\kappa_s)}{L}\log(P)$ while all other attention correlations are sufficiently small. 
\end{proof}

\subsubsection{Phase II, stage 2}\label{sec:p2-s2}
In this final stage, we establish that these structures indeed represent the solutions toward which the algorithm converges.
Given any  $0<\epsilon <1$, for $n>1$, define 
\begin{align}\label{eq-con}
    T^{\epsilon}_{2}\triangleq \max\left\{t>T_{2}: \Phi_{\pb\to v_{k,n}}^{(t)}\leq \log\left(c_{5}\left(\left(\frac{3}{\epsilon}\right)^{\frac{1}{2}}-1\right)N\right) \right\},
  \end{align}
  where $c_{5}$ is some largely enough constant. 

  We state the following induction hypotheses, which will hold throughout this stage:
\begin{hypothesis}
    For $n>1$, suppose $\operatorname{polylog}(P)\gg \log(\frac{1}{\epsilon})$, for each $T_{2}+1 \leq t \leq T_{2}^{\epsilon}$, $\qb\in\cP\setminus\{\pb\}$, the following holds:
     \begin{enumerate}[label={\alph*.}]
     \item $\Phi^{(t)}_{\pb\to v_{k,n}}$ is monotonically increasing, and $\Phi^{(t)}_{\pb\to v_{k,n}}\in[\frac{(1-\kappa_s)}{L}\log(P), O(\log(P/\epsilon))]$;
     \item $\Phi^{(t)}_{\pb\to v_{k,1}}$ is monotonically decreasing and $\Phi^{(t)}_{\pb\to v_{k,1}}\in \Big[-\frac{1}{L}\left(\frac{\Delta}
     {2}+0.01\right)\log(P)-o(1), -\frac{1}{U}\left(\frac{\Delta}
     {2}-0.01\right)\log(P)\Big]$; 
          \item $|\Phi^{(t)}_{\pb\to v_{k,m}}|=
        O\Big(\frac{\Phi^{(t)}_{\pb\to v_{k,n}}-\Phi^{(t)}_{\pb\to v_{k,1}}}{P^{1-\kappa_s}}\Big)$ for $m\not=1,n$;
          \item  $\Upsilon^{(t)}_{k, \pb\to\qb}=
        O\Big(\frac{\Phi^{(t)}_{\bp\to v_{k, n}}}{C_n}\Big)$ for $a_{k,\qb}=n$, $|\Upsilon^{(t)}_{k, \pb\to\pb}|=
        O\Big(\frac{\Phi^{(t)}_{\bp\to v_{k, n}}-\Phi^{(t)}_{\bp\to v_{k, 1}}}{P}\Big)$;
        \item $|\Upsilon^{(t)}_{k, \pb\to\qb}|=O\Big(\frac{|\Phi^{(t)}_{\bp\to v_{k, 1}}|}{C_1}\Big)
       + O\Big(\frac{\Phi^{(t)}_{\bp\to v_{k, n}}-\Phi^{(t)}_{\bp\to v_{k, 1}}}{P}\Big)$ for $ a_{k,\qb}=1$.; 
        \item $|\Upsilon^{(t)}_{k, \pb\to\qb}|=
        O\Big(\frac{\Phi^{(t)}_{\bp\to v_{k, n}}-\Phi^{(t)}_{\bp\to v_{k, 1}}}{P}\Big)$ for $ a_{k,\qb}\not=1,n$.
     \end{enumerate}
\label{hypothesis-p1.4} 
\end{hypothesis}

\paragraph{Property of attention scores.}
We first single out several properties of attention scores that will be used for the proof of \Cref{hypothesis-p1.4}.

\begin{lemma}\label{lemma-tech-4.1}
    if \Cref{hypo-main} and \Cref{hypothesis-p1.4} hold at iteration $T_{n,2}< t\leq T_{n,2}^{\epsilon}$,  then the following holds
    \begin{enumerate}
        \item if $\mask\in\cE_{k,n}$,
        $\Attn^{(t)}_{\pb\to \cP_{k,n}}=\Omega(1)$ and $(1-\Attn_{\pb\to\cP_{k,n}}^{(t)})^2\geq O(\epsilon)$.
        \item Moreover, $\Attn^{(t)}_{\pb\to \cP_{k,1}}=O\Big(\frac{1-\Attn^{(t)}_{\pb\to \cP_{k,n}}}{P^{(1-\kappa_c)+ \frac{L}{U}(\frac{\Delta}{2}-0.01)}}\Big)$;  if $\mask\in\cE_{k,1}$, we have  $\Attn^{(t)}_{\pb\to \cP_{k,1}}=\Omega\Big(\frac{1-\Attn_{\pb\to\cP_{k,n}}^{(t)}}{P^{(1-\kappa_c)+ \frac{U}{L}(\frac{\Delta}{2}+0.01)}}\Big)$; 
        \item for ${\qb}\in \cM\cap(\cP_{k,n}\cup\cP_{k,1})$, $\score_{\pb\to{\qb}}^{(t)}=O\Big(\frac{1-\Attn_{\pb\to\cP_{k,n}}^{(t)}}{P}\Big)$. 
    \end{enumerate}
\end{lemma}
\begin{lemma}\label{lemma-tech4.2}
    if \Cref{hypo-main} and \Cref{hypothesis-p1.4} hold at iteration $T_{n,2}< t\leq T^{\epsilon}_{n,2}$,then for $m\not=n$, the following holds:
    \begin{enumerate}
        \item for any ${\qb}\in\cP_{k,m}$, $\score^{(t)}_{\pb\to{\qb}}\leq O\Big(\frac{1-\Attn_{\pb\to\cP_{k,n}}^{(t)}}{P}\Big)$.
        \item  $\Attn^{(t)}_{\pb\to \cP_{k,n}}\leq O\Big(\frac{1-\Attn^{(t)}_{\pb\to \cP_{k,n}}}{N}\Big)$, and if $\mask\in\cE_{k, m}$, $\Attn^{(t)}_{\pb\to \cP_{k,n}}=\Theta\Big(\frac{1-\Attn^{(t)}_{\pb\to \cP_{k,n}}}{N}\Big)$.
    \end{enumerate}
\end{lemma}

\paragraph{Bounding the gradient updates of FP correlations.} We have the following lemmas.

\begin{lemma}\label{lemma-cg4.1}
    For $n>1$, if \Cref{hypo-main} and \Cref{hypothesis-p1.4} hold at iteration $T_{2}+1\leq t\leq T^{\epsilon}_{2}$, then $\alpha_{\pb\to v_{k,n}}^{(t)}\geq 0$ and satisfies:
\begin{align*}
    \alpha^{(t)}_{\pb\to v_{k,n}}\geq \Omega(\epsilon). 
\end{align*}
\end{lemma}
\begin{proof}
        By Lemma~\ref{lemma-pos-gd}, we have 
    \begin{align*}
        &\alpha_{\pb\to v_{k,n}}^{(t)}\\&=\EE\left[\1\{k_X=k,  \pb\in\mathcal{M}\} \Attn^{(t)}_{\pb\to \cP_{k,n}}\cdot\left(z_n^3\left(1-\Attn^{(t)}_{\pb\to \cP_{k,n}}\right)^2+ \sum_{m\not=n} z_m^2z_n \left(\Attn^{(t)}_{\pb\to \cP_{k,m}}\right)^2   \right)\right]\\
        &= \EE\left[\1\{k_X=k, \cE_{k,n}\cap\pb\in\mathcal{M}\} \Attn^{(t)}_{\pb\to \cP_{k,n}}\cdot\left(z_n^3\left(1-\Attn^{(t)}_{\pb\to \cP_{k,n}}\right)^2+ \sum_{m\not=n} z_m^2z_n \left(\Attn^{(t)}_{\pb\to \cP_{k,m}}\right)^2   \right)\right]\\
        &\quad +\EE\left[\1\{k_X=k, \cE_{k,n}^c\cap\pb\in\mathcal{M}\} \Attn^{(t)}_{\pb\to \cP_{k,n}}\cdot\left(z_n^3\left(1-\Attn^{(t)}_{\pb\to \cP_{k,n}}\right)^2+ \sum_{m\not=n} z_m^2z_n \left(\Attn^{(t)}_{\pb\to \cP_{k,m}}\right)^2   \right)\right]\\
        &\gtrsim \mathbb{P}(\mask\in\cE_{k,n})\cdot \\
        &\EE\left[\1\{k_X=k, \pb\in\mathcal{M}\} \Attn^{(t)}_{\pb\to \cP_{k,n}}\cdot\left(z_n^3\left(1-\Attn^{(t)}_{\pb\to \cP_{k,n}}\right)^2+ \sum_{m\not=n} z_m^2z_n \left(\Attn^{(t)}_{\pb\to \cP_{k,m}}\right)^2   \right)\Big|\cE_{k,n}\right]\\
        &\quad + O(1)\cdot  \mathbb{P}(\mask\in\cE_{k,n}^c)\\
        &\gtrsim  \Omega(\epsilon),
       \end{align*}
             where the last inequality invokes Lemma~\ref{lemma-tech-4.1}, Lemma~\ref{app:lem:prob1} and the fact that 
             \begin{align*}
{\epsilon}\geq {\exp(-\operatorname{polylog}(K))} \gg  \exp \left(-c_{n,1}C_n\right).
\end{align*}

\end{proof}
\begin{lemma}\label{lemma-cg4.2}
    For $n>1$, if \Cref{hypo-main} and \Cref{hypothesis-p1.4} hold at iteration $T_{n,3}< t\leq T^{\epsilon}_{n,4}$, then  $\alpha_{\pb\to v_{k,1}}^{(t)}<0$ and satisfies
    \begin{align*}
        |\alpha_{\pb\to v_{k,m}}^{(t)}|\leq \max\bigg\{O\Big(\frac{\alpha_{\pb\to v_{k,n}}^{(t)}}{P^{(1-\kappa_c)+ \frac{L}{U}(\Delta/2-0.01)}}\Big),O\bigg(\frac{\alpha_{\pb\to v_{k,n}}^{(t)}}{P^{2(1-\kappa_c)+ \frac{L}{U}(\Delta-0.02)- (1-\frac{c_1^*L}{U})(1-\kappa_s)}}\bigg)\bigg\}
    \end{align*}
    \end{lemma}
    The proof follows the similar arguments Lemma~\ref{lemma-cg3.2} by noticing that $\epsilon \gg \mathbb{P}(\mask\in \cE_{k, m}^c)$ for any $m\not=n$. 
\begin{lemma}\label{lemma-cg4.3}
    For $n>1$, if \Cref{hypo-main} and \Cref{hypothesis-p1.4} hold at iteration $T_{2}< t\leq T^{\epsilon}_{2}$,  then for any $m>1$ with $m\not=n$, the following holds 
    \begin{align*}
       -O( \frac{ \alpha^{(t)}_{\pb\to v_{k,n}}}{P^{1-\kappa_s}})\leq  \alpha_{\pb\to v_{k,m}}^{(t)}\leq 0
    \end{align*}
    \end{lemma}
    \begin{proof}
    We first note that
        \begin{align*}
            &-z_1z_n^2\left(1-\Attn^{(t)}_{\pb\to \cP_{k,n}} \right)\Attn^{(t)}_{\pb\to \cP_{k,n}}- z_m^3 \left(1-\Attn^{(t)}_{\pb\to \cP_{k,m}} \right)\Attn^{(t)}_{\pb\to \cP_{k,1}} + \sum_{a\not=1,n} z_a^2z_m\left(\Attn^{(t)}_{\pb\to \cP_{k,a}}\right)^2\\
            &\leq z_m\left( \max_{a\not=m,n}{z_a^2\Attn^{(t)}_{\pb\to \cP_{k,a}}}-z_n^2 \Attn^{(t)}_{\pb\to \cP_{k,n}}- z_m^2 \Attn^{(t)}_{\pb\to \cP_{k,m}} \right) \Big(1-\Attn^{(t)}_{\pb\to \cP_{k,n}}-\Attn^{(t)}_{\pb\to \cP_{k,m}}\Big) \\
            &\lesssim -\Omega(1-\Attn^{(t)}_{\pb\to \cP_{k,n}})
                  \end{align*}
                  since when $\mask\in\cE_{k, n}$, we have $\Attn^{(t)}_{\pb\to \cP_{k,n}}=\Omega(1)\gg \Attn^{(t)}_{\pb\to \cP_{k,a}}$. Thus, we have
            \begin{align*}
                0\geq \alpha^{(t)}_{\pb\to v_{k,m}}&\gtrsim  -\EE\left[\1{\{k_X=k, \cE_{k,n}\cap \pb\in\mathcal{M}\}} \Attn^{(t)}_{\pb\to \cP_{k,m}}\cdot\Omega(1-\Attn^{(t)}_{\pb\to \cP_{k,n}})\right]\\
                &\geq  -O\Big(\frac{\alpha_{\pb\to v_{k,n}}^{(t)}}{P^{1-\kappa_s}}\Big). 
            \end{align*}
    \end{proof}
    
    \paragraph{Bounding the gradient updates of positional correlations.}
We then summarize the properties for gradient updates of positional correlations, which utilizes the identical calculations as in Section~\ref{sec:p1s1:pos}. 
\begin{lemma}
     For $n>1$,  if \Cref{hypo-main} and \Cref{hypothesis-p1.4} hold at iteration ${T}_{2} +1\leq t \leq {T}^{\epsilon}_{2}$, then
     \begin{enumerate}[label={\alph*.}]
         \item if $a_{k,\qb}=n$ and $\qb\not=\pb$,  $\beta^{(t)}_{k, \pb\to\qb}\geq 0$; $\beta^{(t)}_{k, \pb\to\qb}= \Theta\Big(\frac{\alpha_{\pb\to v_{k,n}}^{(t)}}{C_n}\Big)\text{ and }  |\beta^{(t)}_{n}|= O\Big(\frac{\alpha_{\pb\to v_{k,n}}^{(t)}-\alpha_{\pb\to v_{k,1}}^{(t)}}{P}\Big)$. 
         \item  if $a_{k,\qb}=1$, $|\beta^{(t)}_{k, \pb\to\qb}|=  O\Big(\frac{\alpha_{\pb\to v_{k,n}}^{(t)}-\alpha_{\pb\to v_{k,1}}^{(t)}}{P}\Big)+O\Big(\frac{|\alpha_{\pb\to v_{k,1}}^{(t)}|}{C_1}\Big)$.
         \item if $a_{k,\qb}=m$ and $m\not=1, n$, $|\beta^{(t)}_{k,\pb\to\qb}|= O\Big(\frac{\alpha_{\pb\to v_{k,n}}^{(t)}-\alpha_{\pb\to v_{k,1}}^{(t)}}{P}\Big)$. 
     \end{enumerate}
\end{lemma}

\paragraph{End of Phase II, stage 2.} We have the following lemma.

\begin{lemma}\label{lem:end}
    For $n>1$, and $0<\epsilon<1$, suppose $\operatorname{polylog}(P)\gg \log(\frac{1}{\epsilon})$. Then \Cref{hypothesis-p1.4} holds for all  $T_{2}<t\leq T^{\epsilon}_{2}=T_{2}+O\Big(\frac{\log(P\epsilon^{-1})}{\eta\epsilon}\Big)$, and at iteration $t=T^{\epsilon}_{2}+1$, we have 
    \begin{enumerate}
        \item  $\tilde{\cL}_{k,\pb}(Q^{T^{\epsilon}_{2}+1})<\frac{\epsilon}{2K}$;
        \item  If $\mask\in\cE_{k,n}$ , we have  $(1-\Attn_{\pb\to\cP_{k,n}}^{(T^{\epsilon}_{2}+1)})^2\leq O(\epsilon)$.
    \end{enumerate}
       
\end{lemma}
\begin{proof}
    The  existence of $T_{2,k}^{\epsilon}=T_{2,k}+O(\frac{\log(P\epsilon^{-1})}{\eta\epsilon})$ directly follows from  Lemma~\ref{lemma-cg4.1}. We further derive
    \begin{align*}
      \tilde{\cL}_{k, \pb}(Q^{T^{\epsilon}_{2}+1})  
      & = \frac{1}{2}\mathbb{E}\left[\1\{ k_X=k,  \pb\in\mathcal{M}\cap \mask\in\cE_{k,n}\}\left(z_n^2\left(1-\Attn_{\pb\to\cP_{k,n}}\right)^2+ \sum_{m\not=n} z_m^2 \left(\Attn_{n, m}\right)^2   \right)\right]\\
      &\leq \frac{1}{2K}\cdot \gamma\cdot U^2 \cdot (1+o(1))\cdot O(\epsilon)\\
      &\leq \frac{\epsilon}{2K} ,
\end{align*} 
where the first inequality is due to direct calculations by the definition of $T_{2}^{\epsilon}$, and the second inequality can be obtained by setting $c_{n,2}$ in \eqref{eq-con} sufficiently large. 
\end{proof}

\subsection{Analysis for local areas with negative information gap}\label{sec:back:neg}
In this section, we focus on a specific patch $\pb\in\cP$ with the $k$-th cluster for $k\in[K]$, and present the analysis for the case that $X_{\pb}$ is located in the local area for the $k$-th cluster, i.e. \( a_{k,\bp}>1 \).  Throughout this section, we denote $a_{k,\bp}=n$ for simplicity. When $\Delta\leq -\Omega(1)$, we can show that the gap of attention correlation changing rate for the positive case does not exist anymore, and conversely $\alpha_{\pb\to v_{k,n}}^{(t)}\gg\alpha_{\pb\to v_{k,1}}^{(t)}$ from the beginning.  We can reuse most of the gradient calculations in the previous section and only sketch them in this section. 
\paragraph{Stage 1:} we define stage 1 as all iterations $0\leq t \leq T_{\mathrm{neg}, 1}$, where
$$
T_{\mathrm{neg}, 1} \triangleq \max \left\{t: \Phi_{\pb\to v_{k,n}}^{(t)}\leq \frac{(1-\kappa_s)}{L}\log(P)\right\}.
$$
We state the following induction hypothesis, which will hold throughout this stage: 
\begin{hypothesis}
  For each $0 \leq t \leq T_{\mathrm{neg}, 1}$, $\qb\in\cP\setminus\{\pb\}$, the following holds:
     \begin{enumerate}[label={\alph*.}]
     \item $\Phi^{(t)}_{\pb\to v_{k,n}}$ is monotonically increasing, and $\Phi^{(t)}_{\pb\to v_{k,n}}\in\Big[0, \frac{(1-\kappa_s)}{L}\log(P)\Big]$;
     \item $\Phi^{(t)}_{\pb\to v_{k,1}}$ is monotonically decreasing and $\Phi^{(t)}_{\pb\to v_{k,1}}\in \bigg[-O\Big(\frac{\Phi^{(t)}_{\pb\to v_{k,n}}}{P^{-\Delta}}\Big),0 \bigg];$ 
          \item $|\Phi^{(t)}_{\pb\to v_{k,m}}|=
        O\Big(\frac{\Phi^{(t)}_{\pb\to v_{k,n}}-\Phi^{(t)}_{\pb\to v_{k,1}}}{P^{1-\kappa_s}}\Big)$ for $m\not=1,n$;
         \item $\Upsilon^{(t)}_{k, \pb\to \qb}=
        O\Big(\frac{\Phi^{(t)}_{\pb\to v_{k,n}}}{C_n}\Big)$ for $a_{k,\qb}=n$, $|\Upsilon^{(t)}_{k, \pb\to \pb}|=
        O\Big(\frac{\Phi^{(t)}_{\pb\to v_{k,n}}-\Phi^{(t)}_{\pb\to v_{k,1}}}{P}\Big)$;
                \item $|\Upsilon^{(t)}_{k, \pb\to \qb}|=O\Big(\frac{|\Phi_{\pb\to v_{k,1}}^{(t)}|}{C_1}\Big)
       + O\Big(\frac{\Phi^{(t)}_{\pb\to v_{k,n}}-\Phi^{(t)}_{\pb\to v_{k,1}}}{P}\Big)$ for $ a_{k,\qb}=1$; 
        \item $|\Upsilon^{(t)}_{k, \pb\to \qb}|=
        O\Big(\frac{\Phi^{(t)}_{\pb\to v_{k,n}}-\Phi^{(t)}_{\pb\to v_{k,1}}}{P}\Big)$ for $a_{k,\qb}\not=1,n$.
     \end{enumerate}
\label{hypothesis-neg1} 
\end{hypothesis}
Through similar calculations for phase II, stage 1 in \Cref{sec:p2-s1}, we obtain the following lemmas to control the gradient updates for attention correlations.
\begin{lemma}\label{lemma-neg1}
    If \Cref{hypo-main-neg} and \Cref{hypothesis-neg1} hold for  $0\leq t \leq T_{\mathrm{neg}, 1}$, then we have
    \begin{subequations}\label{eq:neg1}
    \begin{align}
    &  \alpha^{(t)}_{\pb\to v_{k,n}}\geq \min\Bigg\{\Omega\Big(\frac{1}{P^{(1-\kappa_s)}}\Big), \Omega\left(\frac{1}{P^{2(\frac{U}{L}-1)(1-\kappa_s)}}\right)\Bigg\}, \\
 &     0\geq   \alpha_{\pb\to v_{k,1}}^{(t)}\geq -O\Big(\frac{\alpha_{\pb\to v_{k,n}}^{(t)}}{P^{-\Delta}}\Big),     \\
 &|\alpha^{(t)}_{\pb\to v_{k,m}}|\leq O\Big( \frac{ \alpha^{(t)}_{\pb\to v_{k,n}}-\alpha_{\pb\to v_{k,1}}^{(t)}}{P^{1-\kappa_s}}\Big) \text{  for all }m\not=n,1\\ 
 &\beta^{(t)}_{k, \pb\to\qb}= \Theta\Big(\frac{\alpha_{\pb\to v_{k,n}}^{(t)}}{C_n}\Big) \text{ for } a_{k,\qb}=n, \qb\not=\pb\\
  &    |\beta^{(t)}_{k, \pb\to\qb}|=  O\Big(\frac{\alpha_{\pb\to v_{k,n}}^{(t)}}{P}\Big)+O\Big(\frac{|\alpha_{\pb\to v_{k,1}}^{(t)}|}{C_1}\Big) \text{ for } a_{k,\qb}=1,\\
 &|\beta^{(t)}_{k,\pb\to\pb}|, |\beta^{(t)}_{k,\pb\to\qb}|=O\Big(\frac{\alpha_{\pb\to v_{k,n}}^{(t)}-\alpha_{\pb\to v_{k,1}}^{(t)}}{P}\Big) \quad\text{  for all } a_{k,\pb}\not=n,1.
    \end{align} 
    \end{subequations}
\end{lemma}
Here $\Delta<0$ implies  $ |\alpha_{\pb\to v_{k,1}}^{(t)}|\ll  \alpha^{(t)}_{\pb\to v_{k,n}}$. \Cref{hypothesis-neg1} can be directly proved by Lemma~\ref{lemma-neg1} and we have 
\begin{align}\label{eq:t1n}
T_{\mathrm{neg}, 1}=O\Big(\frac{P^{\max\{1,2(\frac{U}{L}-1)\}\cdot (1-\kappa_s)}\log(P)}{\eta}\Big).
\end{align}

\paragraph{Stage 2:} 
Given any  $0<\epsilon <1$, define 
\begin{align}\label{eq-con-neg}
   T_{\mathrm{neg}, 1}^{\epsilon}\triangleq \max\left\{t>T_{1}: \Phi_{\pb\to v_{k,n}}^{(t)}\leq \log\left(c_{6}\left(\left(\frac{3}{\epsilon}\right)^{\frac{1}{2}}-1\right)P^{1-\kappa_s}\right) \right\}.
  \end{align}
  where $c_{6}$ is some largely enough constant. We then state the following induction hypotheses, which will hold throughout this stage:
\begin{hypothesis}
    For $n>1$, suppose $\operatorname{polylog}(P)\gg \log(\frac{1}{\epsilon})$, for $\qb\in\cP\setminus\{\pb\}$, and each $T_{\mathrm{neg}, 1}< t \leq T_{\mathrm{neg}, 1}^{\epsilon}$, the following holds:
     \begin{enumerate}[label={\alph*.}]
     \item $\Phi^{(t)}_{\pb\to v_{k,n}}$ is monotonically increasing, and $\Phi^{(t)}_{\pb\to v_{k,n}}\in\Big[\frac{(1-\kappa_s)}{L}\log(P), O(\log(P/\epsilon))\Big]$;
     \item $\Phi^{(t)}_{\pb\to v_{k,1}}$ is monotonically decreasing and $\Phi^{(t)}_{\pb\to v_{k,1}}\in \Bigg[-O\Big(\frac{\Phi^{(t)}_{\pb\to v_{k,n}}}{P^{-\Delta}}\Big),0 \Bigg];$ 
          \item $|\Phi^{(t)}_{\pb\to v_{k,m}}|=
        O\Big(\frac{\Phi^{(t)}_{\pb\to v_{k,n}}-\Phi^{(t)}_{\pb\to v_{k,1}}}{P^{1-\kappa_s}}\Big)$ for $m\not=1,n$;
         \item $\Upsilon^{(t)}_{k, \pb\to \qb}=
        O\Big(\frac{\Phi^{(t)}_{\pb\to v_{k,n}}}{C_n}\Big)$ for $a_{k,\qb}=n$, $|\Upsilon^{(t)}_{k, \pb\to \pb}|=
        O\Big(\frac{\Phi^{(t)}_{\pb\to v_{k,n}}-\Phi^{(t)}_{\pb\to v_{k,1}}}{P}\Big)$;
                \item $|\Upsilon^{(t)}_{k, \pb\to \qb}|=O\Big(\frac{|\Phi_{\pb\to v_{k,1}}^{(t)}|}{C_1}\Big)
       + O\Big(\frac{\Phi^{(t)}_{\pb\to v_{k,n}}-\Phi^{(t)}_{\pb\to v_{k,1}}}{P}\Big)$ for $a_{k,\qb}=1$; 
        \item $|\Upsilon^{(t)}_{k, \pb\to \qb}|=
        O\Big(\frac{\Phi^{(t)}_{\pb\to v_{k,n}}-\Phi^{(t)}_{\pb\to v_{k,1}}}{P}\Big)$ for $a_{k,\qb}\not=1,n$.
     \end{enumerate}
\label{hypothesis-neg2} 
\end{hypothesis}
\begin{lemma}\label{lemma-neg2}
    If \Cref{hypo-main-neg} and \Cref{hypothesis-neg2} hold for  $T_{\mathrm{neg}, 1}< t \leq T_{\mathrm{neg}, 1}^{\epsilon}$, then we have
    \begin{subequations}\label{eq:neg2}
    \begin{align}
    &  \alpha^{(t)}_{\pb\to v_{k,n}}\geq \Omega({\epsilon}), \\
      &     0\geq   \alpha_{\pb\to v_{k,1}}^{(t)}\geq -O\Big(\frac{\alpha_{\pb\to v_{k,n}}^{(t)}}{P^{-\Delta}}\Big),     \\
 &|\alpha^{(t)}_{\pb\to v_{k,m}}|\leq O\Big( \frac{ \alpha^{(t)}_{\pb\to v_{k,n}}-\alpha_{\pb\to v_{k,1}}^{(t)}}{P^{1-\kappa_s}}\Big) \text{  for all }m\not=n,1, \\ 
 &\beta^{(t)}_{k, \pb\to\qb}= \Theta(\frac{\alpha_{\pb\to v_{k,n}}^{(t)}}{C_n}) \text{ for } a_{k,\qb}=n, \qb\not=\pb, \\
  &    |\beta^{(t)}_{k, \pb\to\qb}|=  O\Big(\frac{\alpha_{\pb\to v_{k,n}}^{(t)}}{P}\Big)+O\Big(\frac{|\alpha_{\pb\to v_{k,1}}^{(t)}|}{C_1}\Big) \text{ for } a_{k,\qb}=1,\\
 &|\beta^{(t)}_{k,\pb\to\pb}|, |\beta^{(t)}_{k,\pb\to\qb}|=O\Big(\frac{\alpha_{\pb\to v_{k,n}}^{(t)}-\alpha_{\pb\to v_{k,1}}^{(t)}}{P}\Big) \quad\text{  for all } a_{k,\pb}\not=n,1.
    \end{align} 
    \end{subequations}
\end{lemma}
\Cref{hypothesis-neg2} can be directly proved by Lemma~\ref{lemma-neg2}. Furthermore, at the end of this stage, we will have:
\begin{lemma}\label{lem:end-neg}
    Suppose $\operatorname{polylog}(P)\gg \log(\frac{1}{\epsilon})$, then \Cref{hypothesis-neg2} holds for all  $T_{\mathrm{neg},1}<t\leq T_{\mathrm{neg},1}^{\epsilon}=T_{\mathrm{neg},1}+O\Big(\frac{\log(P\epsilon^{-1})}{\eta\epsilon}\Big)$, and at iteration $t=T_{\mathrm{neg},1}^{\epsilon}+1$, we have 
    \begin{enumerate}
        \item  $\tilde{\cL}_{k,\pb}(Q^{T_{\mathrm{neg},1}^{\epsilon}+1})<\frac{\epsilon}{2K}$;
        \item  If $\mask\in\cE_{k,n}$ , we have  $\Big(1-\Attn_{\pb\to\cP_{k,n}}^{(T_{\mathrm{neg},1}^{\epsilon}+1)}\Big)^2\leq O(\epsilon)$.
    \end{enumerate}
       
\end{lemma}

\subsection{Analysis for the global area}\label{sec:glo}
When $a_{\pb,k}=1$, i.e. the patch lies in the global area, the analysis is much simpler and does not depend on the value of $\Delta$. We can reuse most of the gradient calculations in \Cref{sec:back:pos} and only sketch them in this section.  

For $X_{\pb}$ in the global region $\cP_{k,1}$, since the overall attention $\Attn_{\pb\to\cP_{k,1}}^{(0)}$ to the target feature already reaches $\Omega\Big(\frac{C_1}{P}\Big)=\Omega\Big(\frac{1}{P^{1-\kappa_c}}\Big)$ due to the large number of unmasked patches featuring $v_{k,1}$ when $\mask\in\cE_{k,1}$, which is significantly larger than $\Attn_{\pb\to\cP_{k,m}}^{(0)}=\Theta\Big(\frac{1}{P^{1-\kappa_s}}\Big)$ for all other $m>1$. This results in large $\alpha_{\pb\to v_{k,1}}^{(t)}$ initially, and thus the training directly enters phase II. 

\paragraph{Stage 1:} we define stage 1 as all iterations $0\leq t \leq {T_{c, 1}}$, where
$$
{T_{c, 1}} \triangleq \max \left\{t: \Phi_{\pb\to v_{k,1}}^{(t)}\leq \frac{(1-\kappa_c)}{L}\log(P)\right\}.
$$
We state the following induction hypotheses, which will hold throughout this stage: 
\begin{hypothesis}
    For each $0\leq t \leq {T_{c, 1}}$, $\qb\in\cP\setminus\{\pb\}$, the following holds:
     \begin{enumerate}[label={\alph*.}]
     \item $\Phi^{(t)}_{\pb\to v_{k,1}}$ is monotonically increasing, and $\Phi^{(t)}_{\pb\to v_{k,1}}\in\Big[0, \frac{(1-\kappa_c)}{L}\log(P)\Big]$;
     \item $\Phi_{\pb\to v_{k,m}}$ is monotonically decreasing for $m>1$ and $\Phi_{\pb\to v_{k,m}}\in \Big[-O\big(\frac{\log(P)}{N}\big), 0 \Big]$; 
         \item $\Upsilon^{(t)}_{k, \pb\to\qb}=
        O\Big(\frac{\Phi^{(t)}_{\pb\to v_{k,1}}}{C_1}\Big)$ for $a_{k,\qb}=1$,  $|\Upsilon^{(t)}_{k, \pb\to\pb}|=
        O\Big(\frac{\Phi^{(t)}_{\pb\to v_{k,1}}}{P}\Big)$;        
        \item $|\Upsilon^{(t)}_{k, \pb\to\qb}|=
        O\Big(\frac{\Phi^{(t)}_{\pb\to v_{k,1}}}{P}\Big)$ for $a_{k,\qb}\not=1$.
     \end{enumerate}
\label{hypothesis-glo.1} 
\end{hypothesis}
Through similar calculations for phase II, stage 1 in \Cref{sec:p2-s1}, we obtain the following lemmas to control the gradient updates for attention correlations.
\begin{lemma}\label{lemma-glo1}
    If \Cref{hypo-main} (or \Cref{hypo-main-neg}) and \Cref{hypothesis-glo.1} hold for  $0\leq t \leq {T_{c, 1}}$, then we have
    \begin{subequations}\label{eq:gd-glo-1}
               \begin{align}
          & \alpha^{(t)}_{\pb\to v_{k,1}}\geq     \min\Bigg\{\Omega\Big(\frac{1}{P^{(1-\kappa_c)}}\Big), \Omega\left(\frac{1}{P^{2(\frac{U}{L}-1)(1-\kappa_c)}}\right)\Bigg\},\\
      &  |\alpha^{(t)}_{\pb\to v_{k,m}}|\leq O\Big( \frac{ \alpha^{(t)}_{\pb\to v_{k,1}}}{P^{1-\kappa_s}}\Big)\quad\text{  for all }m\not=1,\\   &\beta^{(t)}_{k,\pb\to\qb}= \Theta\Big(\frac{\alpha^{(t)}_{\pb\to v_{k,1}}}{C_1}\Big), \text{ for } a_{k,\qb}=1, \qb\not=\pb,\\
      &
        |\beta^{(t)}_{k,\pb\to\pb}|, |\beta^{(t)}_{k,\pb\to\qb}|=O\Big(\frac{\alpha^{(t)}_{\pb\to v_{k,1}}}{P}\Big) \quad\text{  for all } a_{k,\qb}>1.
    \end{align} 
    \end{subequations}
\end{lemma}
\Cref{hypothesis-glo.1} can be directly proved by Lemma~\ref{lemma-glo1} and we have 
\begin{align}\label{eq:t11}
    {T_{c, 1}}=O\Bigg(\frac{P^{\max\{1,2(\frac{U}{L}-1)\}\cdot (1-\kappa_c)}\log(P)}{\eta}\Bigg).
\end{align}

\paragraph{Stage 2:} 
Given any  $0<\epsilon <1$, define 
\begin{align}\label{eq-con-global}
    T^{\epsilon}_{c,1}\triangleq \max\left\{t>{T_{c, 1}}: \Phi_{\pb\to v_{k,1}}^{(t)}\leq \log\left(c_{7}\left(\left(\frac{3}{\epsilon}\right)^{\frac{1}{2}}-1\right)P^{1-\kappa_c}\right) \right\}.
  \end{align}
  where $c_{7}$ is some largely enough constant. We then state the following induction hypotheses, which will hold throughout this stage:
\begin{hypothesis}
    For $n>1$, suppose $\operatorname{polylog}(P)\gg \log(\frac{1}{\epsilon})$, $\qb\in\cP\setminus\{\pb\}$,  for each ${T_{c, 1}}+1 \leq t \leq T^{\epsilon}_{c,1}$, the following holds:
     \begin{enumerate}[label={\alph*.}]
     \item $\Phi^{(t)}_{\pb\to v_{k,1}}$ is monotonically increasing, and $\Phi^{(t)}_{\pb\to v_{k,1}}\in\Big[\frac{(1-\kappa_c)}{L}\log(P), O(\log(P/\epsilon))\Big]$;
     \item $\Phi_{\pb\to v_{k,m}}$ is monotonically decreasing for $n>1$ and $\Phi_{\pb\to v_{k,m}}\in \Big[-O\big(\frac{\log(P)}{N}\big), 0 \Big]$; 
         \item $\Upsilon^{(t)}_{k, \pb\to\qb}=
        O\Big(\frac{\Phi^{(t)}_{\pb\to v_{k,1}}}{C_1}\Big)$ for $a_{k,\qb}=1$,  $|\Upsilon^{(t)}_{k, \pb\to\pb}|=
        O\Big(\frac{\Phi^{(t)}_{\pb\to v_{k,1}}}{P}\Big)$;        
        \item $|\Upsilon^{(t)}_{k, \pb\to\qb}|=
        O\Big(\frac{\Phi^{(t)}_{\pb\to v_{k,1}}}{P}\Big)$ for $a_{k,\qb}\not=1$.
     \end{enumerate}
\label{hypothesis-glo.2} 
\end{hypothesis}
We also have the following lemmas to control the gradient updates for attention correlations.
\begin{lemma}\label{lemma-glo2}
    If \Cref{hypo-main} (or \Cref{hypo-main-neg}) and \Cref{hypothesis-glo.1} hold for  ${T_{c, 1}}+1\leq t \leq {T^{\epsilon}_{c, 1}}$, then we have
    \begin{subequations}\label{eq:gd-glo-2}
               \begin{align}
          & \alpha^{(t)}_{\pb\to v_{k,1}}\geq    \Omega\left({\epsilon}\right),  |\alpha^{(t)}_{\pb\to v_{k,m}}|\leq O\Big( \frac{ \alpha^{(t)}_{\pb\to v_{k,1}}}{P^{1-\kappa_s}}\Big)\quad\text{  for all }m\not=1\\   &\beta^{(t)}_{k,\pb\to\qb}= \Theta\Big(\frac{\alpha^{(t)}_{\pb\to v_{k,1}}}{C_1}\Big), \text{ for } a_{k,\qb}=1, \qb\not=\pb\\
      &
        |\beta^{(t)}_{k,\pb\to\pb}|, |\beta^{(t)}_{k,\pb\to\qb}|=O\Big(\frac{\alpha^{(t)}_{\pb\to v_{k,1}}}{P}\Big) \quad\text{  for all } a_{k,\qb}>1.
    \end{align} 
    \end{subequations}
\end{lemma}
\Cref{hypothesis-glo.2} can be directly proved by Lemma~\ref{lemma-glo2}. Furthermore, at the end of this stage, we will have:
\begin{lemma}\label{lem:end-glo}
    Suppose $\operatorname{polylog}(P)\gg \log(\frac{1}{\epsilon})$, then \Cref{hypothesis-glo.2} holds for all  ${T_{c, 1}}<t\leq T^{\epsilon}_{c,1}={T_{c, 1}}+O\Big(\frac{\log(P\epsilon^{-1})}{\eta\epsilon}\Big)$, and at iteration $t=T^{\epsilon}_{c,1}+1$, we have 
    \begin{enumerate}
        \item  $\tilde{\cL}_{k,\pb}(Q^{T^{\epsilon}_{c,1}+1})<\frac{\epsilon}{2K}$;
        \item  If $\mask\in\cE_{k,1}$ , we have  $\Big(1-\Attn_{\pb\to\cP_{k,1}}^{(T^{\epsilon}_{c,1}+1)}\Big)^2\leq O(\epsilon)$.
    \end{enumerate}
       
\end{lemma}

\subsubsection{Proof of Induction Hypotheses}
We are now ready to show \Cref{hypo-main} (resp. \Cref{hypo-main-neg}) holds through the learning process.
\begin{theorem}[Positive Information Gap]\label{thm:hypo:pos}
    For sufficiently large $P>0$, $\eta\ll \log(P)$, $\Omega(1)\leq\Delta<1$,  \Cref{hypo-main} holds for all iterations $t=0,1,\cdots, T=O\Big(\frac{e^{\polylog(P)}}{\eta}\Big)$. 
    \end{theorem}
    \begin{theorem}[Negative Information Gap]\label{thm:hypo:neg}
    For sufficiently large $P>0$, $\eta\ll \log(P)$, $-0.5<\Delta\leq-\Omega(1) $,  \Cref{hypo-main-neg} holds for all iterations $t=0,1,\cdots, T=O\Big(\frac{e^{\polylog(P)}}{\eta}\Big)$. 
    \end{theorem}
    \paragraph{Proof of \Cref{thm:hypo:pos}.} It is easy to verify \Cref{hypo-main} holds at iteration $t=0$ due to the initialization $Q^{(0)}=\mathbf{0}_{d\times d}$. At iteration $t>0$:
    \begin{itemize}
        \item \Cref{hypo-main}\ref{hypo-main-a} can be proven by \Cref{hypothesis-p1.1}-\ref{hypothesis-p1.4} a and \Cref{hypothesis-glo.1}-\ref{hypothesis-glo.2} a, combining with the fact that $\log(1/\epsilon)\ll \polylog(P)$. 
        \item \Cref{hypo-main}\ref{hypo-main-b} can be obtained by invoking \Cref{hypothesis-p1.1}-\ref{hypothesis-p1.4} b. 
         \item \Cref{hypo-main}\ref{hypo-main-c} can be obtained by invoking \Cref{hypothesis-p1.1}-\ref{hypothesis-p1.4} c and \Cref{hypothesis-glo.1}-\ref{hypothesis-glo.2} b. 
         \item To prove \Cref{hypo-main}\ref{hypo-main-d}, for $\qb\not=\pb$, $\Upsilon_{\pb\to\qb}^{(t)}=\sum_{k=1}^{K}\Upsilon_{k,\pb\to\qb}^{(t)}$. By item d-f in \Cref{hypothesis-p1.1}-\ref{hypothesis-p1.4}  and item c-d in \Cref{hypothesis-glo.1}-\ref{hypothesis-glo.2}, we can conclude that no matter the relative areas $\qb$ and $\pb$ belong to for a specific cluster, for all $k\in[K]$,  throughout the entire learning process,  the following upper bound always holds:
         $$
         \Upsilon_{k,\pb\to\qb}^{(t)}\leq \max_{t\in [T]}(|\Phi^{(t)}_{\bp\to v_{k,n}}|+|\Phi^{(t)}_{\bp\to v_{k,1}}|)   \max\Bigg\{O\Big(\frac{1}{C_1}\Big), O\Big(\frac{1}{C_n}\Big), O\Big(\frac{1}{P}\Big)\Bigg\}\leq \tilde{O}\Big(\frac{1}{C_n}\Big). 
         $$
         Moreover, since $K=O(\polylog(P))$, we then have $\Upsilon_{\pb\to\qb}^{(t)}=\tilde{O}(\frac{1}{C_n})$, which completes the proof. 
         \item The proof for \Cref{hypo-main}\ref{hypo-main-d} is similar as before, by  noticing that $\Upsilon_{k,\pb\to\pb}^{(t)}=\tilde{O}(\frac{1}{P})$ for each $k\in[K]$, which is due to  \Cref{hypothesis-p1.1}-\ref{hypothesis-p1.4} d  and  \Cref{hypothesis-glo.1}-\ref{hypothesis-glo.2} c. 
    \end{itemize}
The proof of \Cref{thm:hypo:neg} mirrors that of \Cref{thm:hypo:pos}, with the only difference being the substitution of relevant sections with \Cref{hypo-main-neg}. For the sake of brevity, this part of the proof is not reiterated here.

\subsubsection{Proof of Theorem~\ref{thm:positive} and Theorem~\ref{thm:dynamics} with positive information gap}
\begin{theorem}
\label{thm:positive-re}
Suppose $\Omega(1)\leq \Delta\leq 1 $. For any $0<\epsilon<1$, suppose  $\polylog(P)\gg \log(\frac{1}{\epsilon})$. We apply GD to train the loss function given in \eqref{loss} with $\eta\ll \poly(P)$. 
Then for each $\pb\in\cP$, 
we have
\begin{enumerate}[label={\arabic*.}]
\item The loss converges: after $T^{\star}=O\Big(\frac{\log(P)P^{\max\{2(\frac{U}{L}-1),1\}(1-\kappa_s)}}{\eta}+\frac{\log(P\epsilon^{-1})}{\eta\epsilon}\Big)$ iterations,  ${\cL}_{\pb}(Q^{(T^{\star})})-{\cL}_{\pb}^* \leq \epsilon$, where ${\cL}_{\pb}^{\star}$  is the global minimum of patch-level construction loss in \eqref{eq-obj-n}.
\item Attention score concentrates: given cluster $k\in[K]$, if  $X_{\pb}$ is masked, then  
 the one-layer transformer nearly ``pays all attention" to all {unmasked} patches in the same area $\cP_{k, a_{k,\pb}}$, i.e., $\Big(1-\Attn_{{\pb}\to \cP_{k, a_{k,\pb}}}^{(T^{\star})}\Big)^2 \leq O(\epsilon)$. 
 \item {\bf Local} area learning  feature attention correlation through {\bf two-phase}: given $k\in[K]$, if $a_{k,\pb}>1$, then we have
 \begin{enumerate}[topsep=0pt, left=0pt] 
     \item $\Phi^{(t)}_{\pb\to v_{k,1}}$ first quickly decrease with all other $\Phi^{(t)}_{\pb\to v_{k,m}}$ , $\Upsilon^{(t)}_{\pb\to\qb}$ not changing much;
     \item after some point, the increase of $\Phi^{(t)}_{\pb\to v_{k,a_{k,\pb}}}$ takes dominance. Such $\Phi^{(t)}_{\pb\to v_{k,a_{k,\pb}}}$ will keep growing until convergence with all other feature and positional attention correlations nearly unchanged. 
 \end{enumerate}
 \item {\bf Global} area learning  feature attention correlation through {\bf one-phase}: given $k\in[K]$, if $a_{k,\pb}=1$, throughout the training,  the increase of $\Phi^{(t)}_{\pb\to v_{k,1}}$  dominates, whereas all $A^{(t)}_{1,m}$ with $m\not=1$ and  position attention correlations remain close to $0$. 
\end{enumerate}
\end{theorem} 
\begin{proof}
    The first statement is obtained by letting $T^{\star}=\max\{T_{2}^{\epsilon}, T_{c,1}^{\epsilon}\}+1$ in Lemma~\ref{lem:end} and Lemma~\ref{lem:end-glo}, combining wth  Lemma~\ref{app:lem:opt1} and Lemma~\ref{app:lem:opt2}, which lead to 
    \begin{align*}
        {\cL}_{\pb}(Q^{(T^{\star})})-{\cL}_{\pb}^*&\leq {\cL}_{\pb}(Q^{(T^{\star})})-\Loi_{\pb}\\
        &\leq \tilde{\cL}_{\pb}(Q^{T^{\star}})+O\Big(\exp\Big(-\big(c_{3}P^{\kappa_c}+\ind\big\{1\not\in \cup_{k\in[K]}\{a_{k,\pb}\}\big\} c_4 P^{\kappa_s}\big)\Big)\Big)\\
        &\leq K\cdot \frac{\epsilon}{2K}+ O\Big(\exp\Big(-\big(c_{3}P^{\kappa_c}+\ind\big\{1\not\in \cup_{k\in[K]}\{a_{k,\pb}\}\big\} c_4 P^{\kappa_s}\big)\Big)\Big)\\
        &<\epsilon.
    \end{align*}
      The second statement follows from  Lemma~\ref{lem:end} and Lemma~\ref{lem:end-glo}. The third and fourth statements directly follow from the learning process described in \Cref{sec:back:pos} and \Cref{sec:glo} when \Cref{hypo-main} holds.
\end{proof}

\subsubsection{Proof of Theorem~\ref{thm:positive} and Theorem~\ref{thm:dynamics} with negative information gap}
\begin{theorem}
\label{thm:neg-re}
Suppose $-0.5 \leq  \Delta\leq \Omega(1) $. For any $0<\epsilon<1$, suppose  $\polylog(P)\gg \log(\frac{1}{\epsilon})$. We apply GD to train the loss function given in \eqref{loss} with $\eta\ll \poly(P)$. 
Then for each $\pb\in\cP$, 
we have
\begin{enumerate}[label={\arabic*.}]
\item The loss converges: after $T^{\star}=O\Big(\frac{\log(P)P^{\max\{2(\frac{U}{L}-1),1\}(1-\kappa_s)}}{\eta}+\frac{\log(P\epsilon^{-1})}{\eta\epsilon}\Big)$ iterations,  ${\cL}_{\pb}(Q^{(T^{\star})})-{\cL}_{\pb}^* \leq \epsilon$, where ${\cL}_{\pb}^{\star}$  is the global minimum of patch-level construction loss in \eqref{eq-obj-n}.
\item Attention score concentrates: given cluster $k\in[K]$, if  $X_{\pb}$ is masked, then  
 the one-layer transformer nearly ``pays all attention" to all {unmasked} patches in the same area $\cP_{k, a_{k,\pb}}$, i.e., $\Big(1-\Attn_{{\pb}\to \cP_{k, a_{k,\pb}}}^{(T^{\star})}\Big)^2 \leq O(\epsilon)$. 
 \item {\bf All} areas learning  feature attention correlation through {\bf one-phase}: given $k\in[K]$,  throughout the training,  the increase of $\Phi^{(t)}_{\pb\to v_{k,a_{k,\bp}}}$  dominates, whereas all $\Phi^{(t)}_{\pb\to v_{k,m}}$ with $m\not=1$ and  position attention correlations $\Upsilon^{(t)}_{\pb\to\qb}$ remain close to $0$. 
\end{enumerate}
\end{theorem}

\begin{proof}
    The first statement is obtained by letting $T^{\star}=\max\{T_{\mathrm{neg},1}^{\epsilon}, T_{c,1}^{\epsilon}\}+1$ in Lemma~\ref{lem:end-neg} and Lemma~\ref{lem:end-glo}, combining wth  Lemma~\ref{app:lem:opt1} and Lemma~\ref{app:lem:opt2}, which lead to 
    \begin{align*}
        {\cL}_{\pb}(Q^{(T^{\star})})-{\cL}_{\pb}^*&\leq {\cL}_{\pb}(Q^{(T^{\star})})-\Loi_{\pb}\\
        &\leq \tilde{\cL}_{\pb}(Q^{T^{\star}})+O\Big(\exp\Big(-\big(c_{3}P^{\kappa_c}+\ind\big\{1\not\in \cup_{k\in[K]}\{a_{k,\pb}\}\big\} c_4 P^{\kappa_s}\big)\Big)\Big)\\
        &\leq K\cdot \frac{\epsilon}{2K}+ O\Big(\exp\Big(-\big(c_{3}P^{\kappa_c}+\ind\big\{1\not\in \cup_{k\in[K]}\{a_{k,\pb}\}\big\} c_4 P^{\kappa_s}\big)\Big)\Big)\\
        &<\epsilon.
    \end{align*}
      The second statement follows from  Lemma~\ref{lem:end-neg} and Lemma~\ref{lem:end-glo}. The third and fourth statements directly follow from the learning process described in \Cref{sec:back:neg} and \Cref{sec:glo} when \Cref{hypo-main-neg} holds.
\end{proof}
   \section{Proof of Main Theorems in Contrastive Learning}\label{sec:contrastive-app}
\paragraph{Notation.} Throughout this section,   we abbreviate  $\score_{\pb\to\qb}(X;Q^{(t)})$ as  $\score_{\pb\to\qb}^{(t)}(X)$.   We also write $F^{\texttt{cl}}$ as $F$ and $\cL_{\texttt{cl}}$ as $\cL$ for simplicity. We abbreviate $\Attn^{\texttt{c}}_{{\pb}\to  \cP_{k,m} }(X;Q^{(t)})$  (resp. $\score^{\texttt{c}}_{\pb\to\qb}(X;Q^{(t)})$) as $\Attn_{{\pb}\to  \cP_{k,m} }^{(t)}$(resp. $\score_{\pb\to\qb}^{(t)}$), when the context makes it clear.  Furthermore, we denote
   \begin{align*}
      \ell_{p}(X, \mathfrak{B} ) \coloneqq \frac{e^{\operatorname{\mathsf{Sim}}_F\left(X^{+}, X^{++}\right)/\tau }}{\sum_{X \in \mathfrak{B}} e^{\operatorname{\mathsf{Sim}}_F\left(X^{+}, X\right)/\tau }}, \quad \ell_{s}(X, \mathfrak{B} ) \coloneqq \frac{e^{\operatorname{\mathsf{Sim}}_F\left(X^{+}, X^{-,s}\right)/\tau }}{\sum_{X \in \mathfrak{B}} e^{\operatorname{\mathsf{Sim}}_F\left(X^{+}, X\right)/\tau }}.
   \end{align*} 
   

   \begin{theorem}[Learning with contrastive objective]
\label{thm:cl-convergence-app} 
Suppose the information gap $\Delta\in [-0.5,-\Omega(1)]\cup[\Omega(1),1]$.  We train the ViTs in Definition~\ref{def:model-arch-cl} by GD to minimize  (\ref{obj-cl}) with $\eta\ll \poly(P)$, $\sigma_0^2=\frac{1}{d}$, $\tau=O(\frac{1}{\log d})$, after $T^{\star}=O(\frac{\poly(P)\log P}{\eta})$ iterations, 
we have
\begin{enumerate}
\item {Objective converges:} $\cL_{\texttt{cl}}(Q^{(T^{\star})}) \leq \cL_{\texttt{cl}}^{\star} + \frac{1}{\poly(P)}$, where $\cL_{\texttt{cl}}^{\star}$  is the global minimum of the contrastive objective in (\ref{obj-cl}).  
\item  Attention concentration on {\bf global} area: given  $X\in\cD^{\texttt{cl}}_k$ with $k\in[K]$, for any $\pb\in\cP$, with high probability,  we have $1-\Attn_{{\pb}\to  \cP_{k,1}}(X'; Q^{(T^{\star})})=o(1)$ for $X'\in\{X^{+}, X^{++}\}$.
\item All patches learn global FP correlation: given $k\in [K]$, for any $\pb\in\cP$, $t\in [0, T^{\star}]$,  $\Phi^{(t)}_{\pb\to v_{k,1}}\gg \Phi^{(t)}_{\pb\to v_{k,m}}$ with $m>1$, and at the  convergence, $\Phi^{(T^{\star})}_{\pb\to v_{k,1}}=\Theta(\log P),  \Phi^{(T^{\star})}_{\pb\to v_{k,m}}=o(1)$. 
\end{enumerate}
\end{theorem} 
In the following, we will sketch the proof of the above theorem. Indeed, the roadmap of the analysis is similar to the masked reconstruction loss by using the induction argument, where the key difference is the properties for the gradient of the contrastive objective.  
\subsection{Preliminaries}
In the following, we denote the contrastive loss without regularization as 
\begin{align*}
    \overline{\cL}(Q)\triangleq {\mathbb{E}}_{X^+, X^{++}, \mathfrak{N}}\left[-\tau\log \left(\frac{e^{\operatorname{\mathsf{Sim}}_{F^{\texttt{cl}}}\left(X^{+}, X^{++}\right)/\tau }}{\sum_{X' \in \mathfrak{B}} e^{\operatorname{\mathsf{Sim}}{F^{\texttt{cl}}}\left(X^{+}, X'\right)/\tau }}\right)\right]. 
\end{align*}
\begin{lemma}[feature gradient of contrastive loss]\label{lem:grad-cl}
   Given $k\in[K]$, for ${\pb}\in\cP$, 
   let $ \tilde{\alpha}^{(t)}_{{\pb}\to {v_{k,m}}}\coloneqq\frac{1}{\eta}\big(\Phi^{(t+1)}_{{\pb}\to v_{k,m}}-\Phi^{(t)}_{{\pb}\to v_{k,m}}\big)$ for $m\in[N_k]$, then
\begin{align*}
    \tilde{\alpha}^{(t)}_{{\pb}\to {v_{k,m}}}&= e_{{\pb}}^{\top}\big(-\frac{\partial \cL}{\partial Q}(Q^{(t)})\big){v_{k,m}}=\alpha^{(t)}_{{\pb}\to {v_{k,m}}}-\lambda \Phi^{(t)}_{{\pb}\to v_{k,m}},
\end{align*}
where 
   \begin{align*}
   \alpha^{(t)}_{{\pb}\to {v_{k,m}}}  &= e_{{\pb}}^{\top}\big(-\frac{\partial \overline{\cL}}{\partial Q}(Q^{(t)})\big){v_{k,m}}\\
    &=\frac{1}{   P}\EE\bigg[\sum_{{\qb}\in\cP}\score^{(t)}_{{\pb}\to{\qb}}(X^+){X_{{\qb}}^+}^{\top}\Big(F(X^{++};Q^{(t)})-\sum_{X' \in \mathfrak{B}} \frac{e^{\operatorname{\mathsf{Sim}}_F(X^{+}, X')/\tau  }}{\sum_{X' \in \mathfrak{B}} e^{\operatorname{\mathsf{Sim}}_F(X^{+}, X')/\tau  }}F(X';Q^{(t)}) \Big)\\
    &~~~~~~~~~\cdot \Big[X_{{\qb}}^+-\sum_{{\rb}}\score^{(t)}_{{\pb}\to{\rb}}(X^+)X_{{\rb}}^+\Big]^{\top} v_{k,m} \bigg].
   \end{align*}
\end{lemma}

\begin{proof}
    Notice that
$$
-\frac{\partial \cL}{\partial Q}=-\frac{\partial \overline{\cL}}{\partial Q}+\lambda Q.
$$
Then for $-\frac{\partial \overline{\cL}}{\partial Q}$,  we begin with the chain rule and obtain
   \begin{align}
  & -\frac{\partial \overline{\cL}}{\partial Q} \notag\\
&=  \EE\bigg[\frac{\partial }{\partial Q}\Big(\operatorname{\mathsf{Sim}}_F(X^{+}, X^{++})-\tau\log \big(\sum_{X' \in \mathfrak{B}} e^{\operatorname{\mathsf{Sim}}_F(X^{+}, X')/\tau }\big)\Big)\bigg] \nonumber\\&=\EE\bigg[\frac{\partial F(X^+; Q)}{\partial Q} \Big(F(X^{++};Q)-\sum_{X' \in \mathfrak{B}} \frac{e^{\operatorname{\mathsf{Sim}}_F(X^{+}, X') /\tau }}{\sum_{X' \in \mathfrak{B}} e^{\operatorname{\mathsf{Sim}}_F(X^{+}, X')/\tau  }}F(X';Q) \Big)\bigg]\notag\\
   &=\frac{1}{   P}\EE\bigg[\sum_{{\pb,\qb}\in\cP}\frac{\partial \score_{{\pb}\to{\qb}}(X^+)}{\partial Q} {X_{{\qb}}^{+}}^{\top}\Big(F(X^{++};Q)-\sum_{X' \in \mathfrak{B}} \frac{e^{\operatorname{\mathsf{Sim}}_F(X^{+}, X')/\tau  }}{\sum_{X' \in \mathfrak{B}} e^{\operatorname{\mathsf{Sim}}_F(X^{+}, X')/\tau  }}F(X' ;Q) \Big)\bigg] .\label{eq:chain-cl}
\end{align}
We focus on the gradient for each attention score: 
\begin{align*}
   \frac{\partial \score_{{\pb}\to{\qb}}(X^+)}{\partial Q}
&=\sum_{{\rb}}\frac{\exp\left(e^{\top}_{{\pb}}Q({X}_{{\bb}}^{+}+{X}_{{\qb}}^{+})\right)}{\left(\sum_{{\rb}}\exp(e^{\top}_{{\pb}}QX_{{\rb}})\right)^2} e_{{\pb}}(X_{{\qb}}^{+}-{X}_{{\rb}}^{+})^{\top}\\
   &=\score_{{\pb}\to{\qb}}\sum_{{\rb}}\score_{{\pb}\to{\rb}}e_{{\pb}}(X_{{\qb}}^{+}-X_{{\rb}}^{+})^{\top}\\
       &=\score_{{\pb}\to{\qb}}(X^+)e_{{\pb}}\cdot \left[X_{{\qb}}^{+}-\sum_{{\rb}}\score_{{\pb}\to{\rb}}(X^+)X_{{\rb}}^{+}\right]^{\top}.
\end{align*}
Substituting the above equation into \eqref{eq:chain-cl}, we have 
\begin{align*}
  -\frac{\partial \overline{\cL}}{\partial Q} 
&=\frac{1}{   P}\EE\bigg[\sum_{{\pb,\qb}\in\cP}\score_{{\pb}\to{\qb}}(X^+){X_{{\qb}}^{+}}^{\top}\Big(F(X^{++};Q)-\sum_{X' \in \mathfrak{B}} \frac{e^{\operatorname{\mathsf{Sim}}_F(X^{+}, X')/\tau  }}{\sum_{X' \in \mathfrak{B}} e^{\operatorname{\mathsf{Sim}}_F(X^{+}, X')/\tau  }}F(X';Q) \Big)\\
&~~~~~~~~~~~~\cdot e_{{\pb}}\Big[X_{{\qb}}^+-\sum_{{\rb}}\score_{{\pb}\to{\rb}}(X^+)X_{{\rb}}^+\Big]^{\top} \bigg] .
\end{align*}
Therefore, 
\begin{align*}
   \alpha^{(t)}_{{\pb}\to {v_{k,m}}}&= e_{{\pb}}^{\top}(-\frac{\partial \overline{\cL}}{\partial Q}){v_{k,m}}\\
   &=\frac{1}{   P}\EE\bigg[\sum_{{\qb}\in\cP}\score_{{\pb}\to{\qb}}(X^+){X_{{\qb}}^+}^{\top}\Big(F(X^{++};Q)-\sum_{X' \in \mathfrak{B}} \frac{e^{\operatorname{\mathsf{Sim}}_F(X^{+}, X')/\tau  }}{\sum_{X' \in \mathfrak{B}} e^{\operatorname{\mathsf{Sim}}_F(X^{+}, X')/\tau  }}F(X';Q) \Big)\\
&~~~~~~~~~~~~\cdot \Big[X_{{\qb}}^+-\sum_{{\rb}}\score_{{\pb}\to{\rb}}(X^+)X_{{\rb}}^+\Big]^{\top} v_{k,m} \bigg]. 
\end{align*}
\end{proof}

We then present a high-probability event ensuring that the number of common unmasked patches in each area between positive augmented data pairs is proportional to the total number of patches in that area. 
\begin{lemma}[masking overlap]
   Given a sample $X\sim \cD^{\texttt{cl}}$,  with propbability $1-e^{-\Theta(P^{\kappa_s})}$ over the randomness of masking augmentation to obtain $X^{+}, X^{++}$, supposing $X$ belongs to the $k$-th cluster,  it holds that
   \begin{align*}
    \sum_{\pb\in\cP_{k,m}}\1\left\{X_{\pb}^{+}\not=\mathbf{0}\right\}\1\left\{X_{\pb}^{++}\not=\mathbf{0}\right\}&= \Theta(C_{k,m}),\quad \forall m\in[N_k].
   \end{align*}
    We denote the event that the above relationships hold as $\cA_{1,com}$. Similarly, we have the following event  for $X^{+}$ and $X^{++}$ holds with high probability:
    \begin{align*}
\cA_{1,+} & \coloneqq \bigg\{\sum_{\pb\in\cP_{k,m}}\1\left\{X_{\pb}^{+}\not=\mathbf{0}\right\}= \Theta(C_{k,m}), \forall m\in[N_k] \bigg\}, \\
\cA_{1,++} &\coloneqq \bigg\{\sum_{\pb\in\cP_{k,m}}\1\left\{X_{\pb}^{++}\not=\mathbf{0}\right\}= \Theta(C_{k,m}), \forall m\in[N_k] \bigg\}. 
       \end{align*}
 \end{lemma}
 
 \begin{proof}
    The proof is similar to the analysis of \Cref{app:lem:prob1} by using the concentration property of hypergeometric distribution.
 \end{proof}

 \subsection{Initial stage: emergence of global correlations}
 For the training process at the initial stage,  we define the stage transition
time $T_1$ to be the iteration when $\Phi_{\pb\to v_{k,1}}^{(t)}\geq (1-\kappa_c)\log(P)$ for all $\pb\in\cP$ and $k\in[K]$.

We state the following induction hypothesis, which will hold throughout this stage: 
\begin{hypothesis}
    For each $0 \leq t \leq T_{ 1}=O(\frac{\log(P)P^{3-2\kappa_c}}{\eta })$, $k\in [K]$, letting $\lambda= \frac{2}{P^{3-s\kappa_c}\log (P) }$, the following holds:
       \begin{enumerate}[label={\alph*.}]
       \item $\Phi^{(t)}_{\pb\to v_{k,1}}$ is monotonically increasing, and $\Phi^{(t)}_{\pb\to v_{k,1}}\in\Big[0, (1-\kappa_c)\log(P)\Big]$;
       \item  $|\Phi^{(t)}_{\pb\to v_{k,m}}|\leq O\big(\max\{P^{\kappa_s-1}, P^{2(\kappa_s-\kappa_c)}\cdot\Phi^{(t)}_{\pb\to v_{k,1}} \}\big)$ for $m>1$. 
       \end{enumerate}
  \label{hypothesis-cl1} 
  \end{hypothesis}
\begin{lemma}[bounding the noise correlation]\label{lem:noiseless}
Let us define a noiseless version of the attention score and the network output as
\begin{align}
   \textstyle \hat{\score}_{{\pb}\to {\qb}}(X)& \coloneqq  \frac{e^{e_{{\pb}}^{\top}Q ({X}_{{\qb}}-\xi_{\qb})} }{\sum_{{\rb}\in \cP}e^{e_{{\pb}}^{\top}Q({X}_{{\rb}}-\xi_{\rb})} }, \quad \textrm{for } \pb,\qb \in \cP.\\
   \hat{F}(X; Q) &\coloneqq \frac{1}{P}\sum_{{\pb}, {\qb} \in \cP} \hat{\score}_{{\pb}\to {\qb}}(X)\cdot X_{{\qb}} \quad \in \mathbb{R}^d.
\end{align}
Denote
\begin{align*}
   \hat{\ell}_{p}(X, \mathfrak{B} ) \coloneqq \frac{e^{\operatorname{\mathsf{Sim}}_{\hat{F}}\left(X^{+}, X^{++}\right)/\tau }}{\sum_{X \in \mathfrak{B}} e^{\operatorname{\mathsf{Sim}}_{\hat{F}}\left(X^{+}, X\right)/\tau }}, \quad   \hat{\ell}_{s}(X, \mathfrak{B} ) \coloneqq \frac{e^{\operatorname{\mathsf{Sim}}_{\hat{F}}\left(X^{+}, X^{-,s}\right)/\tau }}{\sum_{X \in \mathfrak{B}} e^{\operatorname{\mathsf{Sim}}_{\hat{F}}\left(X^{+}, X\right)/\tau }}.
\end{align*}
Then supposing \Cref{hypothesis-cl1} holds for $t\leq T_1$, with high probability over the randomness of $X^{+}, X^{++}, \mathfrak{N}$, then for $X\in\mathfrak{B}$, any $\pb,\qb\in \cP$, $s\in [N_c]$,  it holds that
\begin{align*}
   \left|\hat{\score}^{(t)}_{\pb\to\qb}(X)-\score^{(t)}_{\pb\to\qb}(X)\right|&\leq \frac{1}{\poly(d)};\\
   \left\|{\hat{F}^{(t)}}(X; Q)-{{F^{(t)}}}(X; Q)\right\|&\leq \frac{1}{\poly(d)};\\
   \left|\ell_{p}^{(t)}(X, \mathfrak{B})-\hat{\ell}_{p}^{(t)}(X, \mathfrak{B})\right|, \left|\ell_{s}^{(t)}(X, \mathfrak{B})-\hat{\ell}_{s}^{(t)}(X, \mathfrak{B})\right|&\leq \frac{1}{\poly(d)}. 
\end{align*}
We denote the event that the above inequalities hold as $\cA_{2}$.
\end{lemma}
\begin{proof}
   The result follows directly from the concentration of Gaussian random variables,  the boundedness of the feature vectors and the boundedness of $\|e_{\pb}Q\|_{2}\leq \Phi_{k\to v_{k,m}}$ due to the \Cref{hypothesis-cl1} .

\end{proof}
\begin{lemma}[attention score]\label{lem:attn-cl-int}
   Suppose the \Cref{hypothesis-cl1}  holds for $t\leq T_1$, given $\{X^+, X^{++}, \mathfrak{N}\}$, assuming $X\in\cD^{\texttt{cl}}_{k}$ with $k\in[K]$, then for  $m\in[N_k]$, ${\pb}\in\cP$, we have  
   \begin{enumerate}
    \item for $a\in\{+, {++}\}$, if $X^a\in\cA_{2,a}$, then 
    \begin{enumerate}
        \item     $1-\Attn^{(t)}_{\pb\to \cP_{k,1}}(X^{a})\geq \Omega(1)$ and 
        $\Attn^{(t)}_{\pb\to \cP_{k,1}}(X^a)\geq \Omega(\frac{1}{P^{1-\kappa_c}})$ ;
        \item for $m>1$, $\Attn^{(t)}_{\pb\to \cP_{k,m}}(X^a)= \Theta(\frac{1-\Attn^{(t)}_{\pb\to \cP_{k,1}}(X^{a})}{P^{1-\kappa_s}})$;
    \end{enumerate}
    \item for $X'\in\mathfrak{N}$, we have
  \begin{enumerate}
        \item      $1-\tilde{\Attn}^{(t)}_{\pb\to \cP_{k,1}}(X')\geq \Omega(1)$ and $\Score^{(t)}_{\pb\to \cP_{k,1}}(X')\geq \Omega(\frac{1}{P^{1-\kappa_c}})$;
        \item for $m>1$, $\Score^{(t)}_{\pb\to \cP_{k,m}}(X')= \Theta(\frac{1-\Score^{(t)}_{\pb\to \cP_{k,1}}(X')}{P^{1-\kappa_s}})$.
    \end{enumerate}
\end{enumerate}
\end{lemma}
The intuition behind this lemma is that, due to the zero initialization of \(Q\), the attention scores are nearly uniform. As a result, the area attention score \(\Attn_{\pb\to \cP_{k,m}}(X+)\) is proportional to the number of unmasked patches in this area. If \Cref{hypothesis-cl1} holds, we can easily conclude that only the area attention score for the global area will increase, while the relative relationships among the local area attention scores will be preserved.
\begin{lemma}[logit score]\label{lem:logit-cl-int}
    Suppose the \Cref{hypothesis-cl1}  holds for $t\leq T_1$, given $\{X^+, X^{++}, \mathfrak{N}\}$, suppose $X\in\cD^{\texttt{cl}}_k$,  we have 
    \begin{align*}
        1-\ell^{(t)}_q(X,\mathfrak{B})&\geq \Omega(1),\quad \ell^{(t)}_q(X,\mathfrak{B})\geq \Omega(\frac{1}{N_s}),          \quad  q\in \mathfrak{B}\cap \cD_{k}^{cl}\\
&\ell^{(t)}_q(X,\mathfrak{B})\leq O(\frac{1}{N_s}), \quad \text{else.}
    \end{align*}
\end{lemma}

\begin{lemma}[feature gradient near initialization]\label{lem:grad-cl-int}
  Suppose the \Cref{hypothesis-cl1}  holds for $t\leq T_0$, then for $t\leq T_1$, given $k\in[K]$, $m\in[N_k]$, for ${\pb}\in\cP$, 
  \begin{itemize}
   \item For the global feature $m=1$, 
\begin{align*}
   \alpha^{(t)}_{{\pb}\to {v_{k,1}}}&= \Theta\Bigg(\frac{1}{   P}\EE\bigg[z_1(1-{\ell}_p){\Attn}_{{\pb}\to{\cP_{k,1}}}(X^+){\Attn}_{{\pb}\to{\cP_{k,1}}}(X^{++}) \bigg]\Bigg).
\end{align*}
\item For the local feature $m>1$
\begin{align*}
   \alpha^{(t)}_{{\pb}\to {v_{k,m}}}&= \Theta\Bigg(\frac{1}{   P}\EE\bigg[z_m(1-{\ell}_p){\Attn}_{{\pb}\to{\cP_{k,m}}}(X^+){\Attn}_{{\pb}\to{\cP_{k,m}}}(X^{++}) \bigg]\Bigg)\\
  &~~+O\Bigg(\frac{1}{   P}\EE\bigg[z_m(1-{\ell}_p){\Attn}_{{\pb}\to{\cP_{k,m}}}(X^+){\Attn}_{{\pb}\to{\cP_{k,1}}}(X^+){\Attn}_{{\pb}\to{\cP_{k,1}}}(X^{++}) \bigg]\Bigg).
  \end{align*}
  \end{itemize}

\end{lemma}

\begin{proof}
   \begin{align*}
      \alpha^{(t)}_{{\pb}\to {v_{k,m}}}
      &=\frac{1}{   P}\EE\bigg[\sum_{{\qb}\in\cP}\score_{{\pb}\to{\qb}}(X^+){X_{{\qb}}^+}^{\top}\Big((1-\ell_{p})F(X^{++};Q)-\sum_{s=1}^{N_c} \ell_s F(X^{-,s};Q) \Big)\\
      &~~~~~~~~~~~~\cdot \Big[X_{{\qb}}^+-\sum_{{\rb}}\score_{{\pb}\to{\rb}}(X^+)X_{{\rb}}^+\Big]^{\top} v_{k,m}(\1_{\cA_1}+\1_{\cA^c_1} )\bigg]\\
      &\stackrel{(a)}{=}\frac{1}{   P}\EE\bigg[\sum_{{\qb}\in\cP}\hat{\score}_{{\pb}\to{\qb}}(X^+){X_{{\qb}}^+}^{\top}\Big((1-\hat{\ell}_{p})\hat{F}(X^{++};Q)-\sum_{s=1}^{N_c} \hat{\ell}_s \hat{F}(X^{-,s};Q) \Big)\\
      &~~~~~~~~~~~~\cdot \Big[X_{{\qb}}^+-\sum_{{\rb}}\hat{\score}_{{\pb}\to{\rb}}(X^+)X_{{\rb}}^+\Big]^{\top} v_{k,m} (\1_{\cA_1}+\1_{\cA^c_1} ) \bigg]+ \Xi_{\pb,k,m,1}\\
      & =\frac{1}{   P}\EE\bigg[\sum_{i=1}^{N_k}\sum_{{\qb}\in\cP_{k,i}\cap {\cU^+}}\hat{\score}_{{\pb}\to{\qb}}(X^+)(z_iv_{k,i})^{\top}\Big((1-\hat{\ell}_p)\hat{F}(X^{++};Q)-\sum_{s=1}^{N_c} \hat{\ell}_s \hat{F}(X^{-,s};Q) \Big)\\
      &~~~~~~~~~~~~\cdot \Big[z_iv_{k,i}-\sum_{j=1}^{N_k}\sum_{{\rb}\in\cP_{k,j}\cap\cU^+}\hat{\score}_{{\pb}\to{\rb}}(X^+)z_{j}v_{k,j}\Big]^{\top} v_{k,m} \bigg] \tag{$J_1$}\\
      &+ \frac{1}{   P}\EE\bigg[\sum_{i=1}^{N_k}\sum_{{\qb}\in\cP_{k,i}\cap {\cU^+}}\hat{\score}_{{\pb}\to{\qb}}(X^+)(z_iv_{k,i})^{\top}\Big((1-\hat{\ell}_p)\hat{F}(X^{++};Q)-\sum_{s=1}^{N_c} \hat{\ell}_s \hat{F}(X^{-,s};Q) \Big)\\
      &~~~~~~~~~~~~\cdot \Big[\xi_{\qb}-\sum_{{\rb}\in\cP\cap\cU^+}\hat{\score}_{{\pb}\to{\rb}}(X^+)\xi_{\rb}\Big]^{\top} v_{k,m} \bigg]\tag{$J_2$}\\
      &+ \frac{1}{   P}\EE\bigg[\sum_{{\qb}\in\cP\cap {\cU^+}}\hat{\score}_{{\pb}\to{\qb}}(X^+)\xi_{\qb}^{\top}\Big((1-\hat{\ell}_p)\hat{F}(X^{++};Q)-\sum_{s=1}^{N_c} \hat{\ell}_s \hat{F}(X^{-,s};Q) \Big) \\
      &~~~~~~~~~~~~\cdot \Big[z_iv_{k,i}-\sum_{j=1}^{N_k}\sum_{{\rb}\in\cP_{k,j}\cap\cU^+}\hat{\score}_{{\pb}\to{\rb}}(X^+)z_{j}v_{k,j}\Big]^{\top} v_{k,m} \bigg]\tag{$J_3$}\\
      &+ \frac{1}{   P}\EE\bigg[\sum_{{\qb}\in\cP\cap {\cU^+}}\hat{\score}_{{\pb}\to{\qb}}(X^+)\xi_{\qb}^{\top}\Big((1-\hat{\ell}_p)\hat{F}(X^{++};Q)-\sum_{s=1}^{N_c} \hat{\ell}_s \hat{F}(X^{-,s};Q) \Big)\\
      &~~~~~~~~~~~~\cdot \Big[\xi_{\qb}-\sum_{{\rb}\in\cP\cap\cU^+}\hat{\score}_{{\pb}\to{\rb}}(X^+)\xi_{\rb}\Big]^{\top} v_{k,m} \bigg]\tag{$J_4$}\\&+\Xi_{\pb,k,m, 1},
     \end{align*}
     where $(a)$ is bounded by  \Cref{lem:noiseless} with error up to $\Xi_{\pb,k,m,1}\leq \frac{1}{\poly(d)}$, $\cU^+$ is the set of masked patches for $X^+$.    We first look at the term $J_1$, 
     notice that $\xi_{\qb}$ is the random Gaussian noise with zero mean, and is independent of $\hat{\score}$ and $\hat{\ell}$, we then have 
     \begin{align*}
      J_4= &\frac{1}{   P^2}\EE\bigg[\sum_{{\qb}\in\cP_{k,m}\cap {\cU^+}}\hat{\score}_{{\pb}\to{\qb}}(X^+)\xi_{\qb}^{\top}\Big((1-\hat{\ell}_p)\sum_{\pb'\in\cP}\sum_{\rb\in\cP_{k,m}\cap \cU^{++}}\hat{\score}_{{\pb'}\to{\rb}}(X^{++})z_mv_{k,m} \Big)\\
       &~~~~~~~~~~~~\cdot \Big[\xi_{\qb}-\sum_{{\rb}\in\cP\cap\cU^+}\hat{\score}_{{\pb}\to{\rb}}(X^+)\xi_{\rb}\Big]^{\top} v_{k,m} \bigg]\\
       &+\frac{1}{   P^2}\EE\bigg[\sum_{{\qb}\in\cP_{k,m}\cap {\cU^+}}\hat{\score}_{{\pb}\to{\qb}}(X^+)\xi_{\qb}^{\top}\Big((1-\hat{\ell}_p)\sum_{\pb'\in\cP}\sum_{\rb\in \cU^{++}}\hat{\score}_{{\pb'}\to{\rb}}(X^{++})\xi_{\rb} \Big)\\
       &~~~~~~~~~~~~\cdot \Big[\xi_{\qb}-\sum_{{\rb}\in\cP\cap\cU^+}\hat{\score}_{{\pb}\to{\rb}}(X^+)\xi_{\rb}\Big]^{\top} v_{k,m} \bigg]\\
       &=\frac{1}{   P^2}\EE\bigg[z_m\sum_{{\qb}\in\cP_{k,m}\cap {\cU^+}}\hat{\score}_{{\pb}\to{\qb}}(X^+)\Big((1-\hat{\ell}_p)\sum_{\pb'\in\cP}\sum_{\rb\in\cP_{k,m}\cap \cU^{++}}\hat{\score}_{{\pb'}\to{\rb}}(X^{++}) \Big)\Big[1-\hat{\score}_{{\pb}\to{\qb}}(X^+)\Big]\bigg]\\
       &=\frac{1}{   P^2}\EE\bigg[z_m\Big[{\Attn}_{{\pb}\to{\cP_{k,m}}}(X^+)-\sum_{\qb\in\cP_{k,m}\cap\cU^{+}}{\score}^2_{{\pb}\to{\qb}}(X^+)\Big]\Big((1-{\ell}_p)\sum_{\pb'\in\cP}{\Attn}_{{\pb'}\to{\cP_{k,m}}}(X^{++}) \Big)\bigg]+\Xi_{\pb,k,m,2} \\
       &= \Theta\Bigg(\frac{1}{   P}\EE\bigg[z_m(1-{\ell}_p){\Attn}_{{\pb}\to{\cP_{k,m}}}(X^+){\Attn}_{{\pb}\to{\cP_{k,m}}}(X^{++}) \bigg]\Bigg)
      \end{align*}
    where $(a)$ is bounded by invoking  \Cref{lem:noiseless} with error up to $\Xi_{\pb,k,m,2}\leq \frac{1}{\poly(d)}$, and the last equality is due to \Cref{lem:attn-cl-int}. Turning to $J_2$,
     \begin{align*}
      J_2&=\frac{1}{   P}\EE\bigg[\EE\Big[\sum_{i=1}^{N_k}\sum_{{\qb}\in\cP_{k,i}\cap {\cU^+}}\hat{\score}_{{\pb}\to{\qb}}(X^+)(z_iv_{k,i})^{\top}(1-\hat{\ell}_p)\hat{F}(X^{++};Q) \\
      &~~~~~~~~~~~~\cdot \Big[\xi_{\qb}-\sum_{{\rb}\in\cP\cap\cU^+}\hat{\score}_{{\pb}\to{\rb}}(X^+)\xi_{\rb}\Big]^{\top} v_{k,m} \big| \xi\Big]\bigg]\\
      &=\frac{1}{   P^2}\EE\bigg[\EE\Big[\sum_{i=1}^{N_k}\sum_{{\qb}\in\cP_{k,i}\cap {\cU^+}}\hat{\score}_{{\pb}\to{\qb}}(X^+)(z_iv_{k,i})^{\top}(1-\hat{\ell}_p)\sum_{\pb'\in\cP, {\rb}\in\cP\cap\cU^{++}}\hat{\score}_{{\pb'}\to{\rb}}(X^{++})\xi_{\rb} \\
      &~~~~~~~~~~~~\cdot \Big[\xi_{\qb}-\sum_{{\rb}\in\cP\cap\cU^+}\hat{\score}_{{\pb}\to{\rb}}(X^+)\xi_{\rb}\Big]^{\top} v_{k,m} \big| \xi\Big]\bigg]\\
      &=\frac{1}{   P^2}\EE\bigg[\EE\Big[\sum_{{\qb}\in\cP_{k,m}\cap {\cU^+}}\hat{\score}_{{\pb}\to{\qb}}(X^+)(z_mv_{k,m})^{\top}(1-\hat{\ell}_p)\sum_{\pb'\in\cP, {\rb}\in\cP\cap\cU^{++}}\hat{\score}_{{\pb'}\to{\rb}}(X^{++})\xi_{\rb} \\
      &~~~~~~~~~~~~\cdot \Big[\xi_{\qb}-\sum_{{\rb}\in\cP\cap\cU^+}\hat{\score}_{{\pb}\to{\rb}}(X^+)\xi_{\rb}\Big]^{\top} v_{k,m} \big| \xi\Big]\bigg]\\
      &=\frac{1}{   P^2}\EE\bigg[z_m\sum_{{\qb}\in\cP_{k,m}\cap {\cU^+}}\hat{\score}_{{\pb}\to{\qb}}(X^+)(1-\hat{\ell}_p)\\
      &~~~\cdot\Big(\hat{\score}_{{\pb}\to{\qb}}(X^{+})\1_{\qb\in \cU^{++}}-\sum_{\pb'\in\cP,{\rb}\in\cP\cap\cU^{++}\cap \cU^+}\hat{\score}_{{\pb'}\to{\rb}}(X^{++})\hat{\score}_{{\pb}\to{\rb}}(X^{+}) \big)\bigg]\\
      &=\frac{1}{   P^2}\EE\bigg[z_m\sum_{{\qb}\in\cP_{k,m}\cap {\cU^+}}{\score}_{{\pb}\to{\qb}}(X^+)(1-{\ell}_p)\\
      &~~~\cdot\Big({\score}_{{\pb}\to{\qb}}(X^{+})\1_{\qb\in \cU^{++}}-\sum_{\pb'\in\cP,{\rb}\in\cP\cap\cU^{++}\cap \cU^+}{\score}_{{\pb'}\to{\rb}}(X^{++}){\score}_{{\pb}\to{\rb}}(X^{+}) \big)\bigg]+\Xi_{\pb,k,m,3}.
     \end{align*}
     Thus, by invoking \Cref{lem:attn-cl-int}, we have
   \begin{align*}
   |J_2|   &\leq O\Big(\frac{1}{   P}\EE\bigg[z_m\sum_{{\qb}\in\cP_{k,m}\cap {\cU^+}}{\score}_{{\pb}\to{\qb}}(X^+)(1-{\ell}_p)\cdot\Big(\max_{{\rb}\in\cP\cap\cU^{++}\cap \cU^+}{\score}_{{\pb}\to{\rb}}(X^{+})\big)\bigg]\Big)+\Xi_{\pb,k,m,3}\\
   &\leq O\Big(\frac{1}{   P \cdot C_{k,1}}\EE\bigg[z_m(1-{\ell}_p){\Attn}_{{\pb}\to{\cP_{k,m}}}(X^+)\cdot{\Attn}_{{\pb}\to{\cP_{k,1}}}(X^{++})\bigg]\Big).
   \end{align*}
For $J_3$,
     \begin{align*}
      J_3&=\frac{1}{   P}\EE\bigg[\sum_{{\qb}\in\cP\cap {\cU^+}}\hat{\score}_{{\pb}\to{\qb}}(X^+)\xi_{\qb}^{\top}\Big((1-\hat{\ell}_p)\hat{F}(X^{++};Q)-\sum_{s=1}^{n} \hat{\ell}_s \hat{F}(X^{-,s};Q) \Big) \\
      &~~~~~~~~~~~~\cdot \Big[z_iv_{k,i}-\sum_{j=1}^{N_k}\sum_{{\rb}\in\cP_{k,j}\cap\cU^+}\hat{\score}_{{\pb}\to{\rb}}(X^+)z_{j}v_{k,j}\Big]^{\top} v_{k,m} \bigg]\\
      &=\frac{1}{   P}\EE\bigg[\EE\Big[\sum_{{\qb}\in\cP_{k,m}\cap {\cU^+}}\hat{\score}_{{\pb}\to{\qb}}(X^+)\xi_{\qb}^{\top}\Big((1-\hat{\ell}_p)\hat{F}(X^{++};Q)\Big) \\
      &~~~~~~~~~~~~\cdot z_m\Big[1-\sum_{{\rb}\in\cP_{k,m}\cap\cU^+}\hat{\score}_{{\pb}\to{\rb}}(X^+)\Big] \big| \xi\Big]\bigg]\\
      &=\frac{1}{   P^2}\EE\bigg[z_m\sum_{{\qb}\in\cP_{k,m}\cap {\cU^+}\cap {\cU^{++}}}\hat{\score}_{{\pb}\to{\qb}}(X^+)(1-\hat{\ell}_p)\\
      &~~~~~~~~~~\cdot \sum_{\pb'\in\cP}\hat{\score}_{{\pb'}\to{\qb}}(X^{++})\Big[1-\sum_{{\rb}\in\cP_{k,m}\cap\cU^+}\hat{\score}_{{\pb}\to{\rb}}(X^+)\Big] \bigg]\\
      &=\frac{1}{   P^2}\EE\bigg[z_m\sum_{{\qb}\in\cP_{k,m}\cap {\cU^+}\cap {\cU^{++}}}{\score}_{{\pb}\to{\qb}}(X^+)(1-{\ell}_p)\\
      &~~~~~~~~~~\cdot \sum_{\pb'\in\cP}{\score}_{{\pb'}\to{\qb}}(X^{++})\Big[1-{\Attn}_{{\pb}\to{\cP_{k,m}}}(X^+)\Big] \bigg]+\Xi_{\pb,k,m,4}\\
      &{\leq} O\big(\frac{1}{   P^2}\EE\bigg[z_m{\Attn}_{{\pb}\to{\cP_{k,m}}}(X^+)\Big(1-{\Attn}_{{\pb}\to{\cP_{k,m}}}(X^+)\Big)(1-{\ell}_p)\cdot \sum_{\pb'\in\cP} O(\frac{1}{C_{k,m}})\cdot {\Attn}_{{\pb'}\to{\cP_{k,m}}}(X^{++}) \bigg]\big)\\
      &\leq O(\frac{J_4}{ C_{k,m} }) , 
     \end{align*}
where the last inequality is due to \Cref{lem:attn-cl-int}. Moreover, for $J_1$,
     \begin{align*}
      J_1   &=\frac{1}{   P^2}\EE\bigg[\sum_{i=1}^{N_k}\sum_{{\qb}\in\cP_{k,i}\cap\cU^+}\hat{\score}_{{\pb}\to{\qb}}(X^+)(z_iv_{k,i})^{\top}\\
      &~~~~~~~~~~~~\cdot\Big((1-\hat{\ell}_p)\sum_{p'\in\cP}\sum_{j=1}^{N_k}\sum_{{\qb'}\in\cP_{k,j}\cap\cU^{++}}\hat{\score}_{{\pb'}\to{\qb'}}(X^{+,+})z_{j}v_{k,j}\\
      &~~~~~~~~~~~~  -\sum_{X^{-,s}\in \mathfrak{N}\cap\cD_{k}^{cl} } \hat{\ell}_s \sum_{p'\in\cP}\sum_{j=1}^{N_k}\sum_{{\qb'}\in\cP_{k,j}
}\hat{\score}_{{\pb'}\to{\qb'}}(X^{-,s})z_{s,j}v_{k,j} \Big)\\
      &~~~~~~~~~~~~\cdot \Big[z_iv_{k,i}-\sum_{j=1}^{N_k}\sum_{{\rb}\in\cP_{k,j}\cap\cU^+}\hat{\score}_{{\pb}\to{\rb}}z_{j}v_{k,j}\Big]^{\top} v_{k,m} \bigg]\\
   &{=}\frac{1}{   P^2}\EE\bigg[\sum_{i=1}^{N_k}\sum_{{\qb}\in\cP_{k,i}\cap\cU^+}\score_{{\pb}\to{\qb}}(X^+)(z_iv_{k,i})^{\top}\\      &~~~~~~~~~~~~\cdot\Big((1-{\ell}_p)\sum_{p'\in\cP}\sum_{j=1}^{N_k}\sum_{{\qb'}\in\cP_{k,j}\cap\cU^{++}}{\score}_{{\pb'}\to{\qb'}}(X^{+,+})z_{j}v_{k,j}\\
      &~~~~~~~~~~~~  -\sum_{X^{-,s}\in \mathfrak{N}\cap\cD_{k}^{cl} } {\ell}_s \sum_{p'\in\cP}\sum_{j=1}^{N_k}\sum_{{\qb'}\in\cP_{k,j}}{\score}_{{\pb'}\to{\qb'}}(X^{-,s})z_{s,j}v_{k,j} \Big)\\
   &~~~~~~~~~~~~\cdot \Big[z_iv_{k,i}-\sum_{j=1}^{N_k}\sum_{{\rb}\in\cP_{k,j}\cap\cU^+}\score_{{\pb}\to{\rb}}z_{j}v_{k,j}\Big]^{\top} v_{k,m} \bigg]+\Xi_{\pb,k,m, 5}\\
   &=\frac{1}{   P^2}\EE\bigg[\Attn_{{\pb}\to{\cP_{k,m}}}(X^+)\bigg(z_m^2\Big(1-\Attn_{{\pb}\to{\cP_{k,m}}}(X^+)\Big)v_{k,m}-\sum_{i\not=m}z_mz_i\Attn_{{\pb}\to{\cP_{k,i}}}(X^+)v_{k,i}\bigg)^{\top}\\
   &~~~~~~~~~~~~\cdot\Big((1-{\ell}_p)\sum_{\pb'\in\cP}\sum_{j=1}^{N_k}\sum_{{\qb'}\in\cP_{k,j}\cap\cU^{++}}{\score}_{{\pb'}\to{\qb'}}(X^{+,+})z_{j}v_{k,j}\\
   &~~~~~~~~~~~~  -\sum_{X^{-,s}\in \mathfrak{N}\cap\cD_{k}^{cl} } {\ell}_s \sum_{p'\in\cP}\sum_{j=1}^{N_k}\sum_{{\qb'}\in\cP_{k,j}}{\score}_{{\pb'}\to{\qb'}}(X^{-,s})z_{s,j}v_{k,j} \Big)\bigg]+\Xi_{\pb,k,m, 5}\\
   &=\frac{1}{   P^2}\EE\bigg[\Attn_{{\pb}\to{\cP_{k,m}}}(X^+)\bigg(z_m^2\Big(1-\Attn_{{\pb}\to{\cP_{k,m}}}(X^+)\Big)v_{k,m}-\sum_{i\not=m}z_mz_i\Attn_{{\pb}\to{\cP_{k,i}}}(X^+)v_{k,i}\bigg)^{\top}\\
   &~~~~~~~~~~~~\cdot\Big((1-{\ell}_p)\sum_{\pb'\in\cP}\sum_{j=1}^{N_k}\Attn_{\pb'\to \cP_{k,j}}(X^{++})z_{j}v_{k,j}-\sum_{X^{-,s}\in \mathfrak{N}\cap\cD_{k}^{cl} } {\ell}_s \sum_{\pb'\in\cP}\sum_{j=1}^{N_k}{\Score}_{{\pb'}\to \cP_{k,j}}(X^{-,s})z_{s,j}v_{k,j} \Big)\bigg]\\
   &~~~~~~~~~~~~+\Xi_{\pb,k,m, 5}\\
   &=\frac{1}{   P^2}\EE\bigg[\Attn_{{\pb}\to{\cP_{k,m}}}(X^+)\bigg(z_m^2\Big(1-\Attn_{{\pb}\to{\cP_{k,m}}}(X^+)\Big)\\
  &~~~~~~~~~~~~~ \cdot \Big(\sum_{\pb'\in\cP}\big((1-{\ell}_p)z_{m}\Attn_{\pb'\to \cP_{k,m}}(X^{++})-\sum_{X^{-,s}\in \mathfrak{N}\cap\cD_{k}^{cl} }z_{s,m}\ell_s{\Score}_{{\pb'}\to \cP_{k,m}}(X^{-,s}) \big) \Big)\bigg)\bigg] \tag{$J_{1,1}$}\\
   &~~~~-\frac{1}{   P^2}\EE\bigg[\Attn_{{\pb}\to{\cP_{k,m}}}(X^+)\bigg(\sum_{i\not=m}z_mz_i\Attn_{{\pb}\to{\cP_{k,i}}}(X^+)\\
   &~~~~~~~~~~~~~~ \cdot \Big(\sum_{\pb'\in\cP}\big((1-{\ell}_p)z_{i}\Attn_{\pb'\to \cP_{k,i}}(X^{++})-\sum_{X^{-,s}\in \mathfrak{N}\cap\cD_{k}^{cl} }z_{s,i}\ell_s{\Score}_{{\pb'}\to\cP_{k,i}}(X^{-,s}) \big) \Big)\bigg)\bigg] \tag{$J_{1,2}$}\\
   &~~~~~~~~~~~~+\Xi_{\pb,k,m, 5} .
    \end{align*}
    Notice that $J_{1,1}=\Theta(J_4)$. Furthermore, when $m=1$, $J_{1,2}$ is negligible compared to $J_1$, else 
\begin{align*}
   |J_{1,2}|&\leq O\Big(\frac{1}{   P^2}\EE\bigg[\Attn_{{\pb}\to{\cP_{k,m}}}(X^+)\bigg(z_mz_1\Attn_{{\pb}\to{\cP_{k,1}}}(X^+)\\
   &~~~~~~~~~~~~~~ \cdot \Big(\sum_{\pb'\in\cP}\big((1-{\ell}_p)z_{1}\Attn_{\pb'\to \cP_{k,1}}(X^{++})-\sum_{X^{-,s}\in \mathfrak{N}\cap\cD_{k}^{cl} }z_{s,1}\ell_s{\Score}_{{\pb'}\to\cP_{k,1}}(X^{-,s}) \big) \Big)\bigg)\bigg] \Big)\\
   &\leq O\Bigg(\frac{1}{   P}\EE\bigg[z_m(1-{\ell}_p){\Attn}_{{\pb}\to{\cP_{k,m}}}(X^+){\Attn}_{{\pb}\to{\cP_{k,1}}}(X^+){\Attn}_{{\pb}\to{\cP_{k,1}}}(X^{++}) \bigg]\Bigg).
\end{align*}
Putting all the terms together, and noticed that 
\begin{align*}
   & O\Bigg(\frac{1}{   P}\EE\bigg[z_m(1-{\ell}_p){\Attn}_{{\pb}\to{\cP_{k,m}}}(X^+){\Attn}_{{\pb}\to{\cP_{k,1}}}(X^+){\Attn}_{{\pb}\to{\cP_{k,1}}}(X^{++}) \bigg]\Bigg)\\
    &~~\geq O\Big(\frac{1}{   P \cdot C_{k,1}}\EE\bigg[z_m(1-{\ell}_p){\Attn}_{{\pb}\to{\cP_{k,m}}}(X^+)\cdot{\Attn}_{{\pb}\to{\cP_{k,1}}}(X^{++})\bigg]\Big), 
\end{align*}
then we complete the proof.
\end{proof}

\begin{proof}[Proof of \Cref{hypothesis-cl1}]

By \Cref{lem:grad-cl-int} and \Cref{lem:logit-cl-int}, at the initial stage of the learning process, we have
\begin{align*}
  \alpha_{\pb\to v_{k,1}}^{(0)}&\propto\frac{1}{   P}\EE\bigg[\Attn_{{\pb}\to{\cP_{k,1}}}(X^+)\Attn_{{\pb}\to{\cP_{k,1}}}(X^{++})\bigg],\\
  \alpha_{\pb\to v_{k,1}}^{(0)}&\lesssim \frac{1}{   P}\max\bigg\{\EE\bigg[\Attn_{{\pb}\to{\cP_{k,m}}}(X^+)\Attn_{{\pb}\to{\cP_{k,m}}}(X^{++})\bigg], \\
&~~~~~~~~~~~\EE\bigg[\Attn_{{\pb}\to{\cP_{k,m}}}(X^+)\Attn_{{\pb}\to{\cP_{k,1}}}(X^+)\Attn_{{\pb}\to{\cP_{k,1}}}(X^{++})\bigg]\bigg\}.
\end{align*}

Then by the relations of attention score in \Cref{lem:attn-cl-int}, focusing on the high-propbability event $A_{1,+}$ and $A_{1,++}$ we have
\begin{align}
    |\alpha_{\pb\to v_{k,m}}^{(t)}|&\leq  O\big(\max\{P^{\kappa_s-1}, P^{2(\kappa_s-\kappa_c)}\cdot\Phi^{(t)}_{\pb\to v_{k,1}} \}\big)|\alpha_{\pb\to v_{k,1}}^{(t)}|\quad \text{for } m>1 . \label{eq: grad-relation-s1}
  \end{align}
 
   By \Cref{lem:grad-cl-int} and \Cref{lem:attn-cl-int}, we have $\alpha^{(t)}_{\pb\to\cP_{k,1}}\geq \Omega(\frac{1}{P^{3-2\kappa_c}})\geq \lambda \Phi^{(t)}_{\pb\to\cP_{k,1}}$, which implies the regularization in this stage is not violated for the dominated FP correlation $\Phi^{(t)}_{\pb\to\cP_{k,1}}$. Hence, we could focus on the relation between $\alpha^{(t)}_{\pb\to\cP_{k,m}}$ and $\alpha^{(t)}_{\pb\to\cP_{k,1}}$ for $m>1$.

Therefore, the existence of $T_1$ can be directly obtained by the gradient estimation in \Cref{lem:grad-cl-int} and the lower bound for the area attention of the global area in  
    \Cref{lem:attn-cl-int}. The induction argument follows directly from \ref{eq: grad-relation-s1}.
\end{proof}

The key takeaway from the first stage is that the growth of feature-position attention correlation for the global area is dominant, specifically, \( \alpha_{\pb\to v_{k,1}}^{(t)} \gg |\alpha_{\pb\to v_{k,m}}^{(t)}| \). After this initial stage, \(\Phi_{\pb\to v_{k,1}}\) reaches $\Omega(\log (P))$,  \(\Attn_{\pb\to\cP_{k,1}}\) has reached $\Omega(1)$ and \(1 - \ell_p\) still keeps at a constant level. The dominance of global FP correlation will be preserved in the following and the learning process will enter the convergence stage.
\subsection{Convergence}

At this stage, we are going to prove that as long as the ViTs have already learned the global FP correlations, they
will indeed converge to these global solutions, which leads to the collapsed global representation. We present the
statement of our convergence theorem below.
\begin{theorem}[Convergence guarantees]\label{thm:cl-int-3}
          Letting $T_2=\Omega(\frac{ P^4\log P}{\eta})$,  for any $T\in [T_2,  O((\frac{\poly(P) \log P}{\eta}))]$, letting $\lambda=\Theta(\frac{1}{P\log P})$ we have 
        \begin{align*}
 \frac{1}{T} \sum^{T}_{t =T_2}\cL(Q^{(t)})\leq \cL_{\texttt{cl}}^{\star}+ \frac{1}{\poly P}, 
        \end{align*}
    where $\cL_{\texttt{cl}}^{\star}$ is the 
    global minimum of the regularized contrastive objective. 
\end{theorem}
We have the following hypothesis for the end of the learning process.
\begin{hypothesis}\label{hypothesis-cl-end}
    For $t\in [\Omega(\frac{P^4 \log P}{\eta}), O((\frac{\poly(P) \log P}{\eta}))]$, we have the following resutls:
    \begin{itemize}
        \item For any \(k\in[K]\), \(\pb\in\cP\), and \(m\in[N_k]\)
     \begin{align*}
            \Phi_{\pb\to v_{k,1}}^{(t)}\in [C_1^{*}, C_{2}^*]\log P\quad |\Phi_{\pb\to v_{k,m}}^{(t)}|\leq \tilde{O}(\frac{1}{P^{\delta_*}}).
        \end{align*}
        where $C_1^{*}, C_{2}^*>0$ are some constants and $\delta_*\in (0,1)$ is some small constant.
        \item Attention score from the global area: given $X\in\cD_{k}$, $1-\Attn_{\pb\to\cP_{k,1}}^{(t)}(X^a)\leq \frac{1}{\poly(P)}$ for $a\in \{+,++\}$ and $1-\Score_{\pb\to\cP_{k,1}}^{(t)}(X^{n,s})\leq \frac{1}{\poly(P)}$ for $s\in [N_c]$ with high probability. 
        \item Bounded gradient for the  loss: 
        \begin{align*}
       \|\nabla_{Q}\cL(Q^{(t)})\|_F^2 \leq \tilde{O}(\frac{1}{\poly P}). 
        \end{align*}
    \end{itemize}
\end{hypothesis}

We can reuse most of the calculations in the proof of \Cref{hypothesis-cl1} to prove the hypothesis and here we only discuss how to bound the gradient of the objective. If the regularization is not violated, i.e., $\alpha_{\pb\to v_{k,m}}^{(t)}\geq \frac{\Phi_{\pb\to v_{k,m}}^{(t)}}{\lambda}$, we have $\Phi_{\pb\to v_{k,m}}^{(t)}\leq O(\log(P))$. For $t\geq T_{1}$, denote the first time when $\alpha_{\pb\to v_{k,m}}^{(t)}-\frac{\Phi_{\pb\to v_{k,m}}^{(t)}}{\lambda}\leq O(\frac{1}{P^4})$ as $\tilde{T}_{1}$, by \Cref{lem:grad-cl-int},
we have $\tilde{\alpha}_{\pb\to v_{k,m}}^{(t)}\geq \Omega(\frac{1}{P^4})$ for $t\in [T_1, \tilde{T}_1]$, and $\Phi_{\pb\to v_{k,m}}^{(\tilde{T}_1)}= \tilde{C}\log P$ for some constant $\tilde{C}>0$. Then we have $\tilde{T}_{1}\leq O(\frac{P^4\log P}{\eta})$. Thus, for $t\geq \tilde{T}_{1}$, 
\begin{align*}
    \|\nabla_Q \cL(Q^{(t)})\|_F^2 
    \leq O\bigg(\sum_{k=1}^{K}\sum_{\pb\in\cP} (\alpha_{\pb\to v_{k,1}}^{(t)}-\frac{\Phi_{\pb\to v_{k,m}}^{(t)}}{\lambda})^2\bigg)\leq O(\frac{1}{\poly(P)}). 
\end{align*}


\begin{proof}[Proof of convergence]
We first define a learning network that we deem as the “optimal” network with the global feature-position attention pattern. 
Specifically, we define \({Q}^{\star}\) as a matrix satisfied  \(e_{\pb}^{\top}Q^{\star}v_{k,1}=\sigma_{\star}\) with $\sigma^2_{\star}=\frac{\|\bar{Q}\|_{F}}{P(\sum_{k=1}^{K}N_k)}$ and \(e_{\pb}^{\top}Q^{\star}v_{k,m}=0\) for ${\pb}\in\cP$ and $k\in [K]$, $m\in[N_k]$. Furthermore, \(w_1^{\top}Q^{\star}w_{2}=0\), where $w_1, w_2\in \operatorname{Span}\big(\{e_{\pb}\}_{\pb\in\cP}\cap\{v_{k,m}\}_{k\in [K, m\in[N_{k}]]}\big)^{\perp}$. Here we suppose  $\cL_{\texttt{cl}}^{\star}$ is achieved at the matrix $Q=\bar{Q}$.   

Moreover, we consider the following {\bf pseudo} losses and objective: define the linearized learner \(\widetilde{F}^{(t)}(Q,X) = F(Q^{(t)},X) + \nabla_Q F(Q^{(t)},X)( Q - Q^{(t)})\),
\begin{align*}
    \widetilde{\mathcal{L}}_t(Q) &\coloneqq \E\left[ - \tau \log\left(\frac{e^{\langle \widetilde{F}^{(t)}(Q,X^{+}), F(Q^{(t)}; X^{++})\rangle / \tau} }{\sum_{X'\in \mathfrak{B}}e^{\langle \widetilde{F}^{(t)}(Q,X^{+}), F(Q^{(t)}; X')\rangle/\tau} }\right)\right],\\
    \widetilde{\operatorname{\bf Obj}}_t(Q) &\coloneqq \widetilde{\mathcal{L}}_t(Q)+\frac{\lambda}{2} \|Q\|_2^2,
\end{align*}
and 
\begin{align*}
    \widehat{\mathcal{L}}_t(Q) &\coloneqq\E\left[ - \tau \log\left(\frac{e^{\langle F(Q,X^{+}), F(Q^{(t)}; X^{++})\rangle / \tau} }{\sum_{X'\in \mathfrak{B}}e^{\langle F(Q,X^{+}), F(Q^{(t)}; X')\rangle/\tau} }\right)\right].\\
\end{align*}

Then we discuss the values of different losses at $Q=Q^{\star}$. We have the following properties:
\begin{align}
    &
    \cL(Q^{\star})\leq \cL_{\texttt{cl}}^{\star}+ O(\frac{1}{\poly(d)}),  \label{prop-cl-1} \\
    |&\widehat{\mathcal{L}}(Q^{\star})-\overline{\mathcal{L}}_t(Q^{\star})|\leq O(\frac{1}{\poly(d)}),\label{prop-cl-2}\\
    |&\widetilde{\mathcal{L}}_t(Q^{\star}) - \widehat{\mathcal{L}}_t(Q^{\star})| \leq \frac{1}{\poly d}.\label{prop-cl-3}
\end{align}
For the first property, we only need to consider the contrastive loss at the global minimum. Notice that for our data distribution,  the global minimum of the contrastive loss is achieved when the network can perfectly distinguish the samples from different clusters. Thus, we have $\cL_{\texttt{cl}}^{\star}=\Theta(\log \frac{N_{c}}{K})$.
Notice that  on the event $\cA_{1,com}$, supposing $X\in\cD_{k}^{cl}$, which happens with prob $\geq 1-e^{-P^{\kappa_s}}$  we have 
\begin{align*}
    \langle F(Q^{\star},X^{+}), F(Q^{\star}; X^{++})\rangle &= \langle v_{k,1}, v_{k,1}\rangle + \frac{1}{|\Theta(C_{k,1})|^2}\sum_{ \pb\in\cP_{k,1}\cap\cU^+\cap\cU^{++}}\|\xi_{\pb}\|_2^2\pm o(1), \\
    \langle F(Q^{\star},X^{+}), F(Q^{\star}; X')\rangle &= \langle v_{k,1}, v_{k,1}\rangle \pm o(1) \text{ for } X'\in\mathfrak{N}\cap\cD_{k}^{cl} . 
\end{align*}
Furthermore, by Bernstein's inequality, we have with probability $\geq 1-\frac{1}{\poly(d)}$, we have $\|\xi_{\pb}\|_2^2=\sigma_0^2d\pm \tilde{O}(\frac{1}{\poly(d)})=1\pm \tilde{O}(\frac{1}{\poly(d)})$, we denote such an event as $\cA_{3}$. Suppose we consider the temperature \(\tau = O(\frac{1}{\log d})\), then conditioned on $\cA_{1, com}\cap \cA_{3}$, we have \(\langle F(Q^{\star},X^{+}), F(Q^{(t)};X')\rangle = \omega(\log d)\pm o(1)\) for $X'\in\mathfrak{B}\cap\cD_{k}^{cl}$, which could minimize the loss to the level of  $\Theta(\log \frac{N_{c}}{K})$ up to the error of $O(\frac{1}{\poly(d)})$.
 Then we have 
\begin{align*}
    \mathcal{L}(Q^{\star})\leq  (1-\frac{1}{\poly(d)})\Theta(\log \frac{N_{c}}{K})+\tilde{O}(\frac{1}{\poly(d)})\leq     \mathcal{L}(\overline{Q}
    )+O(\frac{1}{\poly(d)}). 
\end{align*} 

The second property follows from the observation that
\begin{align*}
    |\widehat{\mathcal{L}}_t(Q^{\star})-\overline{\mathcal{L}}(Q^{\star})|&\leq  O(\|\nabla_Q \overline{\mathcal{L}}(Q^{\star})\|_2)\|F(Q^{\star},X^{++})-F(Q^{(t)},X^{++})\|_2\\
&\leq   O\Big(\|\nabla_Q \overline{\mathcal{L}}(Q^{\star})\|_2\cdot \Big(1-\Attn_{\pb\to\cP_{k,1}}^{(t)}(X^{++})\Big)
\leq \tilde{O}(\frac{1}{\poly(P)}).
\end{align*}
Similarly, the third property follows from the fact that
\begin{align*}
    |\tilde{\mathcal{L}}_t(Q^{\star})-\hat{\cL}^{cl}_t(Q^{\star})|&\leq  O(\|\nabla_Q \tilde{\mathcal{L}}_t(Q^{\star})\|_2)\|\tilde{F}^{(t)}(Q^{\star},X^{+})-F(Q^{\star},X^{+})\|_2
\leq \tilde{O}(\frac{1}{\poly(P)}).
\end{align*}

Now we will use the tools from online learning to obtain a loss guarantee:
\begin{align*}
    &\quad \eta\langle \nabla_Q \cL(Q^{(t)}), Q^{(t)} - Q^{\star} \rangle \\
    &= \frac{1}{2}\eta^2\|\nabla_Q\cL(Q^{(t)})\|_F^2 - \frac{1}{2}\|Q^{(t)} - Q^{\star}\|_F^2 + \frac{1}{2}\|Q^{(t+1)} - Q^{\star}\|_F^2 \\
    & = \frac{\eta^2}{2}\cdot \frac{1}{\poly (P)}- \frac{1}{2}\|Q^{(t)} - Q^{\star}\|_F^2 + \frac{1}{2}\|Q^{(t+1)} - Q^{\star}\|_F^2.
\end{align*}
Notice that $\tilde{\operatorname{\bf Obj}}_t(Q)$ is a convex function over $Q$ and $\tilde{\operatorname{\bf Obj}}_t(Q^{(t)})=\cL(Q^{(t)})$, thus 
\begin{align*}
    \langle \nabla_Q \cL(Q^{(t)}), Q^{(t)} - Q^{\star}\rangle &= \langle \nabla_Q \tilde{\operatorname{\bf Obj}}_t(Q^{(t)}), Q^{(t)} - Q^{\star}\rangle \\
    & \geq\tilde{\operatorname{\bf Obj}}_t(Q^{(t)}) - \tilde{\operatorname{\bf Obj}}_t(Q^{\star}) \tag{by convexity}\\
    & \geq \tilde{\operatorname{\bf Obj}}_t(Q^{(t)}) - {\mathcal{L}}(Q^\star)
    - \tilde{O}(\frac{1}{\poly(P)}) \tag{by \ref{prop-cl-2} and \ref{prop-cl-3} }\\
    & \geq \tilde{\operatorname{\bf Obj}}_t(Q^{(t)}) - \cL_{\texttt{cl}}^{\star}- \tilde{O}(\frac{1}{\poly(P)}) \tag{by \ref{prop-cl-1}}\\
    & = \cL(Q^{(t)}) - \cL_{\texttt{cl}}^{\star}- \tilde{O}(\frac{1}{\poly(P)})  \tag{by definition of \(\widetilde{\operatorname{\bf Obj}}\)} .
\end{align*}
Thus by a telescoping summation, we have
\begin{align*}
 \frac{1}{T - T_2} \sum^{T}_{t =T_2}\cL(Q^{(t)}) - \cL_{\texttt{cl}}^{\star} 
    &\leq  \sum^{T}_{t =T_2}   \langle \nabla_Q \cL(Q^{(t)}), Q^{(t)} - Q^{\star}\rangle  + {O}(\frac{1}{\poly(P)}) \\
    &\leq O(\frac{\|Q^{(T)} - Q^{(\star)}\|_2^2}{T\eta}) \leq O(\frac{1}{\poly(P)}), 
\end{align*}
which completes the proof.
\end{proof}
\end{document}